\documentclass[11pt]{article}
\usepackage{Packages}
\usepackage{NewCommands}

\usepackage{amsmath,amssymb,amsfonts,bm}

\def\eqref#1{equation~\ref{#1}}

\def\1{\bm{1}}

\def\vzero{{\bm{0}}}

\def\vb{{\bm{b}}}

\def\ve{{\bm{e}}}

\def\vk{{\bm{k}}}

\def\vo{{\bm{o}}}

\def\vq{{\bm{q}}}

\def\vu{{\bm{u}}}
\def\vv{{\bm{v}}}
\def\vw{{\bm{w}}}
\def\vx{{\bm{x}}}
\def\vy{{\bm{y}}}
\def\vz{{\bm{z}}}

\def\mA{{\bm{A}}}
\def\mB{{\bm{B}}}
\def\mC{{\bm{C}}}
\def\mD{{\bm{D}}}
\def\mE{{\bm{E}}}

\def\mI{{\bm{I}}}

\def\mK{{\bm{K}}}

\def\mM{{\bm{M}}}

\def\mO{{\bm{O}}}

\def\mQ{{\bm{Q}}}

\def\mU{{\bm{U}}}
\def\mV{{\bm{V}}}
\def\mW{{\bm{W}}}
\def\mX{{\bm{X}}}
\def\mY{{\bm{Y}}}

\DeclareMathAlphabet{\mathsfit}{\encodingdefault}{\sfdefault}{m}{sl}
\SetMathAlphabet{\mathsfit}{bold}{\encodingdefault}{\sfdefault}{bx}{n}

\newcommand{\E}{\mathop{\mathbb{E}}}

\newcommand{\R}{\mathbb{R}}

\DeclareMathOperator*{\argmax}{arg\,max}
\DeclareMathOperator*{\argmin}{arg\,min}

\DeclareMathOperator{\sign}{sign}

\title{On the Benefits of Rank in Attention Layers}

\usepackage{authblk}
\author[1]{Noah Amsel}
\author[2]{Gilad Yehudai}
\author[1,2,3]{Joan Bruna}
\affil[1]{Courant Institute of Mathematical Sciences, New York University}
\affil[2]{Center for Data Science, New York University}
\affil[3]{Flatiron Institute}

\begin{document}

\maketitle

\begin{abstract}
Attention-based mechanisms are widely used in machine learning, most prominently in transformers.
However, hyperparameters such as the rank of the attention matrices and the number of heads are scaled nearly the same way in all realizations of this architecture, without theoretical justification.
In this work we show that there are dramatic trade-offs between the rank and number of heads of the attention mechanism.
Specifically, we present a simple and natural target function that can be represented using a single full-rank attention head for any context length, but that cannot be approximated by low-rank attention unless the number of heads is exponential in the embedding dimension, even for short context lengths.
Moreover, we prove that, for short context lengths, adding depth allows the target to be approximated by low-rank attention. For long contexts, we conjecture that full-rank attention is necessary.
Finally, we present experiments with off-the-shelf transformers that validate our theoretical findings.

\end{abstract}

{\small 
\setcounter{tocdepth}{1}
\tableofcontents
}
\section{Introduction}
Attention-based architectures are ubiquitous in contemporary machine learning.
The most prominent examples are transformers, which are constructed by stacking several layers of attention with MLPs, residual connections, and normalization layers to represent functions on sequences or sets.
This basic skeleton leaves the user free to set several hyperparameters, although few of these have been carefully studied.
In fact, in the thousands of papers that use this architecture, many hyperparameters are kept the same or nearly the same as in the original paper \cite{NIPS2017_3f5ee243} (see \cref{sec:hyperparams} for a comparison). In this paper, we study the importance of the rank of the attention mechanism.

An attention layer is a map between sequences of vectors in $\R^d$.
The size of an attention layer is determined by the number of heads ($H$) and the rank of the query and key weight matrices ($r$), so that the total number of parameters is of order $d H r$. 
Notably, nearly every transformer architecture sets the number of heads to be $H = d / r$, and the few exceptions of which we are aware differ by a factor of 2 at most (see \cref{sec:hyperparams}).
In fact, this scaling is so standard that it is hard-coded into libraries like PyTorch \cite{paszke2019pytorch} and xFormers \cite{xFormers2022}, a fact which has probably discouraged experimentation with other scalings.
The original motivation for this scaling is to match the parameter count of a single full rank head, i.e. the case  $H=1, r=d$.
We know of no \emph{a priori} reason or experimental evidence that favors this scaling over any other, as the trade-offs between the rank and the number of heads are still not well-understood.
For example, most transformers in the literature use a small rank of between $64$ and $128$, despite the embedding dimension $d$ varying dramatically (e.g. $d=512$ in the original transformers paper \cite{NIPS2017_3f5ee243} and $d=8192$ in LLaMA \cite{touvron2023llama}).
It is not clear whether the expressive power of transformers is weakened by maintaining a fixed rank as the dimension is increased.

A long line of work in the theory of deep learning has studied the relative importance of width and depth in determining the expressive power of feedforward neural networks, as a first necessary step towards understanding the practical tradeoffs (that also include optimization aspects). 
This paper is analogous in that we study parameter trade-offs in transformers through the lens of expressive power, although transformers have more hyperparameters than just width and depth (see \cref{sec:hyperparams}).
For feedforward networks, depth $2$ suffices for universal approximation \cite{cybenko1989approximation}, but greater depth may be required for \emph{efficient} approximation.
That is, some functions can be efficiently represented by a three layer network but cannot be represented by a two layer network unless it is exponentially wide in the input dimension \cite{eldan2016power,daniely2017depth,safran2017depth}).
It is natural to ask a similar question about attention architectures. How should we set the hyperparameters to make our transformers efficient? In particular, is low-rank attention fundamentally weaker than high-rank attention, or is the expressive power driven solely by the parameter product $Hr$, acting as the analog of the width of an MLP layer?

In this paper, we study precisely these fine-grained trade-offs in the expressive capacity of attention layers. We present a simple target function arising naturally in semantic search that can be approximated up to any accuracy by a single full rank attention head regardless of the context length. On the other hand, approximating this target with low-rank attention requires the number of heads to be super-polynomial in the input dimension, even for short context lengths. Specifically, using full-rank heads the required total number of parameters is $dHr \simeq d^2$, while it becomes $\simeq d^{1+\epsilon^{-1}}$ if one uses low-rank heads instead, to reach relative accuracy $\epsilon$. 
Increasing the depth allows for better approximation using only polynomially many heads, at least for short context lengths. We complement these theoretical results with experiments on off-the-shelf transformer architectures. 
Our results demonstrate a very stark trade-off between the rank and number of heads in attention mechanisms and shed a new light on the standard scaling $H=d/r$ used in transformers. 

\subsection{Our Contributions}
\begin{itemize}
\item In \cref{sec:invariant}, we prove a rank separation for representing the nearest neighbor function using multi-head attention.
This function can be approximated to any accuracy using only a single full-rank head.
Yet in the high-dimensional regime, at least $\Omega\left((d/r)^{1/\epsilon}\right)$ heads of rank $r$ are required to achieve relative squared error $\epsilon$. 
Moreover, in the high-accuracy regime ($\epsilon$ going to zero with $d$ fixed), the required number of heads is exponential: $\Omega(\exp( d - r \log(d/r)))$. 
\item In \cref{sec:biased}, we use different techniques to establish exponential separation in the high-accuracy regime for the \emph{biased} nearest neighbor function.
This target function can be approximated up to any accuracy using single full-rank head with the addition of a bias, but $\Omega(\exp(d-r))$ rank-$r$ heads are required to approximate it with better than $O(1/d^4)$ relative squared error. 

\item In \cref{sec:majority}, we explore ways to circumvent the weakness of low-rank attention. We show that augmenting the attention architecture and adding a second, non-linear layer can achieve this using polynomially many heads, but unlike full-rank attention, such constructions may not scale to long sequence lengths.

\item In \cref{sec:experiments}, we support our theoretical results with experiments on standard transformer architectures with multiple layers of attention and MLPs.
We show that the full rank models easily learn the target to high accuracy --- even recovering our main construction --- but the low rank models struggle to do so.
Users of standard transformers may not think that setting $H=2$ could be much worse than $H=1$, but in this case, it is.
\end{itemize}

\section{Related Work}

\paragraph{Theory of transformers}
A growing line of work has sought to provide theoretical analysis of transformers and the attention mechanism.
Training dynamics, inductive biases, generalization, and in-context learning have all received significant attention.
However, papers in these areas nearly always assume that full-rank attention is used \cite{NEURIPS2023_0561738a,
cabannes2024scaling,
fu2023what,
sanford2024transformers,
edelman2022inductive,
NEURIPS2023_b2e63e36,
JMLR:v25:23-1042,
jelassi2024repeat,
chen2024training,
deora2023optimization,
tian2023scan}, even though many also assume there are multiple heads.
Our work provides important context for these results, showing that full-rank models may not be good proxies for the low-rank transformers used in practice.

\paragraph{Expressive power of transformers}
Our work belongs to a body of research studying the representational capacity of transformers.
Unlike other topics in transformer theory, results in this area often do apply to low-rank attention.
\cite{yun2019transformers} proves that (exponentially deep) transformers are universal approximators even with rank one.
\cite{wei2022statistically,merrill2023expresssive} show that transformers can simulate Turing machines if their size is allowed to grow with the sequence length.
\cite{kim2022provable,kajitsuka2023transformers} show that transformers are capable of memorizing data.
\cite{bhattamishra2024separationsrepresentationalcapabilitiestransformers} shows that transformers can efficiently implement a version of the nearest neighbor algorithm for in-context classification of points on the sphere, but their construction uses attention that is full-rank with respect to the input dimension.
Our formulation of the nearest neighbor task is slightly different and can be solved with full-rank attention almost trivially (see \cref{fact:invariant_upper}).
Finally, an important line of work analyzes the representational capacity of transformers using classes of formal languages, finite automata, and circuits
\cite{hahn2020theoretical,liu2022transformers,hao2022formal,merrill2022saturated,10.1162/tacl_a_00663},
but it does not capture separations in capacity due to rank.

\paragraph{Limitations of low-rank attention}
Several other studies have investigated the role of the rank of the attention mechanism.
\cite{bhojanapalli2020low} presents experiments that challenge the canonical $H = d/r$ scaling.
They argue that fixing $d$ and $r$ based on the context length $N$ and setting $H$ independently leads to more powerful and efficient models.
They also prove that a full-rank attention head can produce any attention pattern from any input (for \emph{some} setting of the weights), but a low-rank attention head cannot;
however, \cite{Likhosherstov_Choromanski_Weller_2023}
shows that even rank $r = \log(N)$ suffices to represent any \emph{sparse} attention pattern.
\cite{mahdavi2024memorization} asks how many input-output pairs a low-rank multi-head attention layer can exactly memorize.
For their problem, it is not worth setting $r > N$; furthermore the memorization capacity depends on $rH$ rather than on $r$ or $H$, supporting the standard scaling.
We study the more realistic and practically motivated setting of approximating a natural function over data drawn from a natural distribution.
Unlike \cite{Likhosherstov_Choromanski_Weller_2023,mahdavi2024memorization}, we show that high rank is sometimes essential, irrespective of $H$.

The paper closest to our own is \cite{sanford2024representational}, which proves two separations related to rank.
First, they present a function that can be well-approximated by a single attention head if and only if its rank is sufficiently large.
This result prompts the following question: can using multiple heads compensate for the weakness of low-rank attention? We answer this question in the negative.
Second, they present a one-dimensional function on $N$ inputs that is impossible to represent exactly unless $rHp > N$, where $p$ is the bits of precision.
We extend this result in that our lower bounds apply (1) even for $N=2$, (2) for infinite or finite precision (3) to function approximation over a natural distribution, not just exact representation.
Additionally, our target function engenders a stronger separation:
while $H \geq \Omega(1/r)$ suffices in their setting, ours requires $H$ to grow polynomially or even exponentially in $d/r$ to overcome the weakness of low-rank attention.
However, their target functions are more closely akin to the kinds of structured reasoning tasks to which transformers are often applied.
In particular, they highlight how attention is naturally suited to capturing pairwise interactions; recurrent architectures struggle to do this efficiently, while transformers struggle to capture third-order interactions.

\paragraph{Low rank compression and fine-tuning}
Much recent work in model compression \cite{lv-etal-2023-lightformer,hajimolahoseini2021compressing,ben-noach-goldberg-2020-compressing} and fine-tuning \cite{hu2022lora} is based on the empirical observation that the weight matrices of pretrained transformers (like those of other neural networks) can be replaced or fine-tuned by lower-dimensional proxies without sacrificing performance, and in some cases even helping it \cite{sharma2024the}.
Such results contextualize our work by showing that full-rank is not \emph{always} better than low-rank.

\paragraph{Depth-width trade-offs in neural networks}
Many previous works studied separation between neural networks of different depths, and between neural networks and kernel methods. \cite{eldan2016power,daniely2017depth,safran2017depth,venturi2022depth} constructed functions that can be approximated efficiently with a $3$-layer neural network, but for which $2$-layer networks require the width to be exponential in the input dimension. \cite{pmlr-v49-telgarsky16, chatziafratis2019depth} show depth separation for networks with constant input dimension and varying depths.
Our lower bounds are also closely related technically to separation results between neural networks and kernel methods. \cite{yehudai2019power} prove that random features (or any other kernel method) cannot learn even a single neuron unless the number of features or magnitude of the weights is exponential in the input dimension.
\cite{kamath2020approximate} improved on their result by removing the dependence on the magnitude of the weights.
\cite{ghorbani2021linearized,misiakiewicz2023six} study upper and lower bounds in approximating polynomials with kernel methods.
They show that essentially, it is necessary and sufficient for the number of features to be exponential in the degree of the approximated polynomial.
Our lower bounds are inspired by this work. 

\section{Setting and Notations}
\paragraph{Attention layers.}
A rank-$r$ attention head is parameterized by the weight matrices $\mQ, \mK, \mV, \mO \in \R^{d \times r}$.
(Some authors call these $\mW_Q, \mW_K, \mW_V$, and $\mW_O$.)
A multi-head attention layer is simply the sum of $H$ such attention heads.
The input to a multi-head attention layer is a sequence of vectors $\vx_1, \ldots \vx_N \in \R^d$ called the target points and a sequence $\vy_1 \ldots \vy_M$ called the source points.
(Note that the name ``target points'' is unrelated to that of the ``target function'' we wish to approximate.)
If the columns of $\mX \in \R^{N \times d}$ and $\mY \in \R^{M \times d}$ are the target and source points, respectively, then a softmax multi-head attention layer is a function of the form
\begin{equation} \sum_{h = 1}^H \mO_h\mV_h^\top \mX \sm\left(\mX^\top \mK_h \mQ_h^\top \mY \right) \in \R^{M \times d}~, \end{equation}
where $\sm(\cdot)$ computes the softmax of each column of its input; that is, for each $\vy$, it outputs a probability distribution over $[N]$ based on the scores $\mX^\top \mK_h \mQ_h^\top \vy \in \R^{N}$.
A hardmax attention layer is the same, except that the hardmax function $\hardmax(\cdot)$ outputs $\ve_{i^*}$, where $i^*$ is the index of the maximum score.
Note that hardmax heads are often considered to be a special case of softmax heads, since $\lim_{c \to \infty} \sm(\mX^\top c \mK_h \mQ_h^\top \mY) = \hardmax(\mX^\top \mK_h \mQ_h^\top \mY)$ in pointwise convergence.

Above, we have described so-called cross-attention, which takes both source points and target points as input.
The familiar self-attention layers are a special case in which the source points and target points are identical: $\mX = \mY$.
A given multi-head attention function can be applied to any number of source or target points, since no part of this definition depends on $N$ or $M$.
In addition, it is invariant to permutations of the target points and equivariant to permutations of the source points.

\paragraph{Generalized attention}
We prove our lower bounds against a class of functions that generalizes multi-head attention.
Rather than computing the attention distribution as $\sm(\mX^\top \mK_h \mQ_h \mY)$, we allow any function depending on $\vy$ and a rank-$r$ projection of $\mX$ that outputs a probability distribution over $[N]$.
In addition, we replace $\mO_h \mV_h$ with a single matrix $\mV_h \in \R^{d \times d}$.
Thus, our model is
\begin{equation} \sum_{h = 1}^H \mV_h \mX \phi_h\left(\mK_h^\top \mX, \mY \right) ~,\end{equation}
where $\mK_h \in \R^{d \times r}$, the function $\phi_h : \R^{r \times N} \times \R^d \to \Delta^N$ is applied column-wise to $\mY$ and $\Delta^N$ is the simplex.
Note that the function $\phi_h$ may vary between heads.
Moreover, we allow $\mV_h \in \R^{d \times d}$ to be full-rank.
Note that this class captures, beyond standard transformer architectures, the use of biases, additive positional encodings, and other encoding schemes like RoPE \cite{su2024roformer} and ALiBi \cite{press2022train} in the attention layer.
We also capture architectures from early works on attention \cite{Bahdanau, pmlr-v37-xuc15}, which used feedforward networks to compute the attention scores instead of the ``multiplicative'' or ``dot product'' attention scores $\mX^\top \mK \mQ \mY$ used in transformers.

\paragraph{Nearest neighbor function}
The input to the nearest neighbor function consists of a sequence of $N$ target points $\vx_1, \ldots, \vx_N \in \sphere$ (also denoted by $\mX \in \R^{d \times N}$) and a source point $\vy \in \sphere$. 

The nearest neighbor function outputs the target point that is closest to the source:
\begin{equation}
\label{eq:target}
    f(\vx_1, \ldots, \vx_N; \vy) := \argmin_{\vx \in \{\vx_1, \ldots \vx_N\}} \|\vx - \vy\|_2~.
\end{equation}
This function is analogous to performing a semantic search, in which the goal is to retrieve the entry or word in a database or context window that most closely matches a query.
This function is highly symmetric.
Like multi-head attention itself, it is defined for any $N$ and is invariant to permutations of the target points.
It is also invariant to simultaneous orthogonal transformations of $\mX$ and $\vy$, so it has no principal directions, subspaces, or scales.

\paragraph{Data distribution} We draw target and source points uniformly from the sphere.
For our lower bounds, it is convenient to assume that the target points are orthogonal.
For $N \leq d$, let $\mathcal D_N(\sphere)$ denote the uniform distribution over the set of sequences $\vx_1, \ldots \vx_N \in \sphere$ for which $i \neq j \implies \vx_i \perp \vx_j$.
Such samples can be generated by taking the first $N$ columns of a random orthonormal matrix.
Note that this is similar in essence to drawing the data points independently from the unit sphere, as isotropic random vectors in high dimension are nearly orthogonal.
This distribution is invariant to orthogonal transformations of $\mX$ and of $\vy$.

\section{Low-Rank Separation for Nearest Neighbors}
\label{sec:invariant}

In this section, we study the capacity of multi-head attention to represent the nearest-neighbor function.
We show a separation in representational power based on rank.
The target can be represented efficiently using full-rank attention, but under the assumptions below, approximating it using low-rank attention requires a much larger model. We begin with the upper bound using a single full-rank attention head:
\begin{fact}[Full-rank Efficient Approximation, Equivariant Case]
\label{fact:invariant_upper}
For the target function from \cref{eq:target}, any $\epsilon > 0$, $N,d\in\mathbb{N}$ there exist $\mK,\mQ,\mV\in\reals^{d\times d}$ such that:
\begin{equation}
\E_{\vy,\vx_1,\dots,\vx_N\sim\unif(\sd)}\left[\left\|f(\mX,\vy) - \mV \mX\sm(\mX^\top\mK\mQ^\top \vy)\right\|^2\right] \leq \epsilon~.
\end{equation}
\end{fact}
The construction is straightforward. Consider for simplicity the hardmax case. Set $\mV = \mK\mQ^\top = \mI$ so that $\|\vx_i - \vy\|_2 = 2 - \vx_i^\top\mK\mQ^\top\vy$.
Then $\hardmax(\mX^\top \mK\mQ^\top \vy) = \ve_{i^*}$ where $i^* = \argmin_{i \in [N]} \|\vx_i - \vy\|_2$ and $\ve_i$ is the $i$th standard basis vector.
Note that this construction using hardmax works for any input distribution on $\sphere$ and any number of points $N$, as it represents the target function exactly.
The softmax case is similar; for the formal statement see appendix \cref{appen:proof of upper invariant}.
This construction (or one very similar to it) is easily learned by gradient descent; see \cref{fig:perfect}.

We now turn to the lower bound.
We show that approximating the target function with rank-$r$ heads requires the number of heads to be large unless $r \sim d$.
For technical convenience, we set the number of target points to two and draw them from the distribution $\mathcal D_2(\sphere)$ in which they are always orthogonal.
Our main result establishes a strong quantitative separation 
between full-rank and low-rank self-attention layer, even when the total number of parameters is of the same order:
\begin{restatable}[Low-Rank Approximation Lower Bounds, Equivariant Case]{theorem}{invariantLower}
\label{thm:invariant_lower}
There exist universal constants $c, c', C$ and $C'$ such that if either of the following sets of assumptions hold:
\begin{enumerate}
\item \emph{High-accuracy regime}: $r \leq d-3$, $\epsilon \leq \frac{c}{d+1}$, and
\begin{equation}H \leq C \cdot 2^{d - (r+1)\log_2\left(2d/r\right)}~.\end{equation}
\item \emph{High-dimensional regime:} $d \geq 5$, $\epsilon \geq  \frac{c'}{d - 2e^2 \cdot r}$ and
\begin{equation} 
H \leq \frac12 \left(\frac1{2e} \cdot \frac{d}{r+C'/\epsilon}\right)^{C'/\epsilon}~. \end{equation}
\end{enumerate}
Then, for any choice of $H$ rank-$r$ generalized attention heads $\phi_h : \R^{r \times 2} \to \Delta^1, \mV_h \in \R^{d \times d}, \mK_h \in \R^{d \times r}$ the error of approximating the nearest neighbor function is bounded as follows
\begin{equation} 
\E_{\substack{\vx_1, \vx_2 \sim \mathcal{D}_2(\sphere)\\ \vy \sim \unif(\sphere)}} \left\|f(\mX; \vy) - \sum_{h=1}^H \mV_h \mX \phi_h\left(\mK_h^\top \mX, \vy\right)\right\|_2^2 \geq \epsilon~, \end{equation}
where $f$ is defined as in \cref{eq:target}.
\end{restatable}
For the proof of \cref{thm:invariant_lower}, see \cref{sec:proof_invariant_lower}.
Intuitively, the approximation problem becomes harder as $d \to \infty$ and as $\epsilon \to 0$.
\cref{thm:invariant_lower} combines guarantees in two different regimes. In the first regime, the desired accuracy $\epsilon$ is small.
In this case, the necessary number of heads grows exponentially with $d-r$.
In the second regime, the dimension $d$ is be large.
In this case, the necessary number of heads grows polynomially with $d/r$.
Informally, both regimes show that the error is at least $\epsilon$ whenever $H \lesssim (d/r)^{1/\epsilon}$.

We emphasize that the data distribution is $\frac1{\sqrt{d}}$-close to the uniform product measure in Wasserstein distance, and we expect our main proof techniques 
to generalise to this uniform measure, as well as other rotationally invariant distributions. Additionally, while $N=2$ is sufficient for our purposes to establish the separation, we also believe the framework should extend to the general setting of $N>2$, although this is out of the present scope.   

Our proof uses tools from harmonic analysis on the sphere.
It is reminiscent of the original depth separation work of Eldan and Shamir and Daniely \cite{eldan2016power,daniely2017depth}, which also exploited the inability of ridge functions to approximate radially-symmetric targets with substantial high-frequency energy.
Due to the rotational symmetry of the target function, attention function, and data distribution, we can transform our problem to depend on a pair of points $\vx = \vx_1 - \vx_2$ and $\vy$ drawn uniformly from the sphere, rather than $\vx_1, \vx_2$ and $\vy$.
Our target is essentially given by a step function of the form $(\vx, \vy) \mapsto \sgn(\vx^\top \vy)$, which has a slowly decaying spectrum with respect to the appropriate basis.
We construct this basis using spherical harmonics, and like them, our basis functions are organized into orthogonal subspaces based on degree $\ell$ polynomials.
Due to rotational symmetry, the energy of the target function is uniformly spread within each harmonic subspace.
In contrast, each attention head is tied to a few principal directions given by the span of $\mK_h$.
As a result, each head is spanned by only a fraction of the basis functions in each subspace.
Thus, with a limited number of heads, it is impossible to capture a substantial fraction of the energy of the target function.

We now comment on the tightness of this lower bound, focusing on the canonical setting of $r=1$.
In this case, our lower bound simplifies and strengthens slightly.
For fixed $\epsilon$ and large $d$, the error of approximation is at least $\epsilon$ whenever $H = O\left(d^{1/(4\epsilon)}\right)$.
We can construct an upper bound for our problem by considering rank-1 heads to be random features.
In \Cref{sec:random_features_proof}, we argue that we can approximate our target function in the RKHS associated with the feature map $(\vx_1-\vx_2, \vy) \mapsto \sgn\left((\vx_1-\vx_2)^\top \vk \vq^\top \vy\right)$, where $\vk$ and $\vq$ are drawn uniformly from the unit sphere.
The associated kernel integral operator diagonalizes in the same basis of tensorized spherical harmonics used to decompose the target function above, and thus the kernel ridge regression approximation can be explicitly analysed by bounding the spectral decay of the kernel. Then, via standard arguments from random feature expansions \cite{bach2017equivalence},
one can transfer the approximation guarantees from the RKHS to the random feature model, provided that $H = \widetilde{\Omega}(d^{2/\epsilon^2})$.
Thus, for $r=1$ and fixed $\epsilon$, the approximation lower bound of \Cref{thm:invariant_lower} captures the qualitatively correct behavior, though its precise dependence on $d$ may not be tight.

\section{Exponential Separation for Biased Nearest Neighbors}
\label{sec:biased}
In this section, we show another way to get exponential separation in the high-accuracy regime using different techniques and a modified target function.
Given $\vb = [b_1, \ldots b_N]^\top$, the biased nearest neighbor function is defined as follows:
\begin{equation}
\label{eq:biased_target}
    f_{\vb}(\vx_1, \ldots, \vx_N; \vy) = \argmin_{\vx_i \in \{\vx_1,\dots,\vx_N\}} \Big[\|\vx_i - \vy\|_2^2 + b_i\Big] ~.
\end{equation}

Like the unbiased nearest neighbor function of \cref{eq:target}, it is invariant to simultaneous orthogonal transformations of $\mX$ and $\vy$; however, it is not invariant to permutations of the target point $\mX$. 
We first show that a single full-rank attention head can approximate this target exactly, provided that biases are added to the architecture:
\begin{fact}[Full Rank Efficient Approximation, Biased Case]
\label{fact:biased_upper bound}
For any dimension $d$, number of points $N$, and bias $\vb \in \R^{N}$, a single biased full-rank hardmax attention head can exactly represent the biased nearest neighbor function defined in \cref{eq:biased_target}.
\end{fact}
The construction is the same as that of \cref{fact:invariant_upper} with the addition of biases $\vb$ inside the hardmax. That is, the head implements $\mX\hardmax\left(\mX^\top\vy + \vb\right)$ in the hardmax case.
In \cref{appen:proofs from biased} we prove the softmax case.
Note that this architecture is a special case of standard attention with concatenated positional encodings.
Let the positional encoding for $\vx_i$ be the scalar $b_i$, let the positional encoding for $\vy_i$ be $1$, and let $\mK\mQ^\top = \begin{bmatrix}\mI_{d\times d} & \cdot \\ \cdot & 1\end{bmatrix}$.
Then $\begin{bmatrix}\mX^\top & \vb\end{bmatrix} \mK\mQ^\top \begin{bmatrix}\vy \\ 1\end{bmatrix} = \mX^\top \vy + \vb$.

We now present our main result for this section which shows that even for $N=2$, there exists a biased nearest neighbor function that is hard to approximate using low rank attention heads:

\begin{theorem}[Low-rank Approximation Lower Bounds, biased case]
\label{thm:biased_lower_bound}
    There exists $\vb = [b_1,b_2]^\top \in \R^2$ such that for the function $f_{\vb}$ defined in \cref{eq:biased_target} the following holds: For any choice of rank-$r$ heads $g_1,\dots,g_H$ where $g_h = \mV_h \mX \phi_h(\mK_h \mX,\vy)$, $\mK_h$ is rank-$r$ and $\phi_h$ are arbitrary functions that output a vector in the simplex $\Delta^{1}$, if $H\cdot\max_{h}\norm{\mV_{h}} \leq \frac{\exp(c_1 (d-r))}{d^2c_2}$ then:
    \begin{equation}\label{eq:main lower bound}
        \E_{\substack{\vx_1,\vx_2 \sim \mathcal D_2(d^2\sphere) \\ \vy \sim \mathcal N\left(0, \mI\right)}}\left[\norm{f_{\vb}(\vx_1,\vx_2,\vy) - \sum_{h=1}^H g_h(\vx_1,\vx_2,\vy)}_2^2\right] > \frac{1}{20}~,
    \end{equation}
    for some universal constants $c_1,c_2 > 0$. 
\end{theorem}

The full proof is deferred to \cref{appen:proofs from biased}. The theorem states that unless the number of attention heads or the magnitude of the output weights (or both) are exponential in $d-r$, then rank-$r$ attention heads cannot approximate the target, even up to a constant accuracy. This is in contrast to the fact that a single full-rank head (with positional encoding) can approximate the target up to any given accuracy.
Note that the exponential separation is very strong in terms of the rank of the attention heads. Namely, having rank $O(d)$ is not enough to break this separation, for example even if $r=\frac{99}{100}\cdot d$ there is still an exponential separation between full rank and rank-$r$ attentions heads for a large enough input dimension $d$.

\begin{remark}[Bound on the weights]
    Note that in contrast to \cref{thm:invariant_lower}, here we have an exponential upper bound on the weights of the linear combination $\mV_h$, namely either the number of heads or the norm of the weights needs to be exponential to break the separation. This bound is also found in \cite{yehudai2019power} which inspires our proof. In \cite{kamath2020approximate} the authors were able to remove this bound by applying a more intricate analysis using SQ-dimension arguments, however in our case it is not clear how to extend their technique because of the dependence on $r$. We conjecture that it is still possible to remove this bound, and leave it for future work.
\end{remark}

\paragraph{Proof intuition.}
The crux of the proof of \cref{thm:biased_lower_bound} is to create a linear combination of many threshold functions which behaves like a periodic function with high frequency.
Our proof is inspired by and extends the proof method of \cite{yehudai2019power} for separation between kernel methods and $2$-layer neural networks.
In more details, note that the target can be re-written as a sum of two threshold functions:
 \begin{equation}
    f_{\vb}(\vx_1,\vx_2,\by) = \arg\max_{\vx_i}\inner{\vx_i,\by} + b_i = \mathds{1}(\inner{\vx_1 - \vx_2,\by} + b^* > 0)\vx_1 + \mathds{1}(\inner{\vx_1 - \vx_2,\by} + b^* < 0)\vx_2~,
\end{equation}
where $b^* = b_1-b_2$ will be determined later. Denote by $\vx:=\vx_1-\vx_2$; we will focus on showing hardness of approximation for the first threshold function $\mathds{1}(\inner{\vx,\vy} + b^* > 0)$, from which hardness of approximation for $f_{\vb}$ follows by standard arguments. We define a periodic function $\psi_a(z):\reals\rightarrow\reals$ in the interval $[-a,a]$ that is a linear combination of $a$ threshold functions (at different break points), where $a=\Omega(d^2)$ , and show that for any function $g$ which depends only on a projection $\mK$ of $\vx$ onto and $r$-dimensional subspace
we have:

 \begin{equation}
        \E_{\vx,\vy}\left[\left|\psi_a(\inner{\bx,\by})\cdot g(\mK \vx,\by)\right|\right]\leq \norm{g}\cdot\exp(-\Omega(d-r))~.
\end{equation}

In particular, if any single threshold function that is used to construct $\psi_a$ can be approximated by a rank-$r$ attention layer with $H/a$ heads, then also $\psi_a$ can be approximated by a rank-$r$ attention layer with $H$ heads. However, this is not possible if $r$ is small since $a$ is only polynomial in $d$, and the correlation between each head and $\psi_a$ is exponentially small. Hence, there exists some threshold function with a break point at  $b^*$ which is hard to approximate, unless the number of heads is of the order $O\left(\frac{\exp(d-r)}{a}\right)$.
In \cref{thm:biased_lower_bound}, the inputs $\vx_1$ and $\vx_2$ are drawn from the unit sphere scaled by a factor of $d^2$. We note that re-scaling the inputs is similar to decreasing the required accuracy by the same factor. Hence, this exponential separation result is akin to the high-accuracy regime of \cref{thm:invariant_lower}, although the techniques used in the proof are very different.

\section{Efficient Approximation Using Depth}
\label{sec:majority}
In the previous sections, 
we showed that
a single layer of low-rank attention fails to represent the target unless the number of heads is very large. In this section, we take up the question of whether 
additional layers 
of depth can overcome this weakness.
Depth can mean either adding an MLP after the attention layer or just another attention layer; in this section we consider both options.
We present a construction that approximates the target function (with slightly modified inputs) using two layers and only polynomially many rank-1 heads.
However, we present constructions only for the case where the context length $N = 2$, which is also the setting of our lower bounds.
We conjuncture that any construction using low-rank heads introduces an unfavorable dependence on $N$, a significant weakness compared to full-rank attention.

Our constructions are based on the strategy we call ``majority voting'', which we briefly describe here.
Consider the case of $N=2$ target points and hardmax attention.
The output of each head, like the target function itself, is either $\vx_1$ or $\vx_2$.
A random rank-$1$ head is weakly correlated with the target; the probability that it outputs the correct answer is $1/2 + \Omega(1/\sqrt{d})$.
Thus, combining many such random heads together, their mode (the output with the most ``votes'') matches the target function with high probability.
We use a second layer to calculate the ``majority vote'' of the heads in the attention layer.

Standard attention mechanisms make it difficult to count the number of votes each target point received---or even to remember what the target points $\vx_1$ and $\vx_2$ were---since the next layer gets only a linear combination of them with unknown coefficients.
Therefore, we slightly modify the attention layer to facilitate the majority voting strategy.
We concatenate labels to the vectors that allow us to count how many times $\vx_1$ and $\vx_2$ appear in the sum.
We then use a second layer of attention to look up the full vector corresponding to the majority label.
This labeling can be implemented by concatenating positional encodings to the input points.
That is, instead of inputting $\vx_1, \ldots \vx_N \in \sphere$ to the transformer, we now input $\begin{bmatrix}\vx_1 \\ \vb_1\end{bmatrix}, \ldots, \begin{bmatrix}\vx_N \\ \vb_N\end{bmatrix}$ for $\vb_i \in \R^{e}$.
A linear transformation can be used to map the output of this $(d+e)$-dimensional transformer back to $\R^d$.
Note that our target function is permutation-invariant, so the order of the points is irrelevant to the task at hand.
Thus, these concatenated ``positional encodings'' function more like a modification to the architecture.
They provide extra input dimensions that serve as scratch space in which the model can perform discrete operations like counting and indexing without corrupting the input data.
Also note that, because they change the dimension of the inputs and of the transformer, these concatenated positional encodings are different from the positional encodings used in practice (including RoPE \cite{su2024roformer} and ALiBi \cite{press2022train}), which are included in our framework of generalized attention.

Below, we give the formal definition of the multi-layer transformer architecture used in our construction.
It uses self-attention, meaning that the source and target points are the same.
We modify the attention mechanism by adding a self-excluding mask so that each input point cannot attend to itself (see below, where we form $\tilde{\mX_i}$ by deleting the $i$th column of $\mX$).
Following standard practice, we also use a skip connection.
We do not need a MLP or normalization layer, though 
our construction can easily be extended to include them.

\begin{definition}\label{def:two_layer_posn_transformer}
    A rank-$r$ \textbf{self-masked transformer layer} with $H$ heads is a function  $T: \R^{d \times N} \to \R^{d \times N}$ parameterized by rank-$k$ attention heads $\{(\mM_h, \mV_h)\}_{h=1}^H$ and defined as follows:
\begin{align}
\tilde{\mX_i} &:= \begin{bmatrix}\vline && \vline & \vline && \vline \\ \vx_1 & \cdots & \vx_{i-1} & \vx_{i+1} & \cdots & \vx_N \\ \vline && \vline & \vline && \vline\end{bmatrix}\\
T_i(\mX) &:= \vx_i + \sum_{h=1}^H \mV_h \tilde{\mX_i} \sm\left(\tilde{\mX_i}^\top \mM_h \vx_i\right) \\ 
\end{align}
Here, $T_i$ denotes the $i$th output (or $i$th column of the output) $\begin{bmatrix}T_1(\mX) & \cdots & T_N(\mX) \end{bmatrix}$.

A \textbf{two layer, rank-$r$ transformer with concatenated positional encodings} is a function $T : \R^{d \times N} \to \R^{d \times N}$ parameterized by a positional encoding matrix $\mE = \R^{d_e \times N}$ and two $(d+d_e)$-dimensional self-masked transformer layers, $T^{(1)}$ and $T^{(2)}$, and an output-layer matrix $\mA \in \R^{d \times (d+d_e)}$ and defined as follows:
\begin{equation} T(\mX) = \mA \cdot T_N^{(2)}\left(T^{(1)}\left(\begin{bmatrix}\mX \\ \mE \end{bmatrix}\right)\right) ~.\end{equation}
\end{definition}

The following theorem describes our majority voting construction using random rank-1 heads and concatenated positional encodings. For the proof, see \cref{sec:majority_positional_proof}.
\begin{restatable}{theorem}{majoritypositional}
\label{thm:majority_positional}
There exist universal constants $c_1, c_2$ such that for all $d > c_1$, and $\epsilon \in \left(0, \frac{1}{2}\right)$, and $H \geq c_2\cdot\frac{d^3}{\epsilon^2}$, there exists a two layer, rank-1 transformer $T$ with $H$ heads and $d_e = 2$ (as defined in \cref{def:two_layer_posn_transformer}) for which
\begin{equation}
\E_{\vx_1,\vx_2,\vy\sim \unif(\sphere)}\left\|f(\vx_1,\vx_2;\vy) - T\left(\begin{bmatrix}\vx_1 & \vx_2 & \vy\end{bmatrix}\right)\right\|_2^2 \leq \epsilon~.
\end{equation}
\end{restatable}

One might wonder whether the concatenated positional encodings are necessary to make this construction work, especially since they break permutation invariance in order to represent a permutation invariant target.
In \cref{sec:majority proofs}, we present an alternative construction (\cref{thm:majority_wide_hidden}) that is permutation invariant.
However, it modifies the architecture by applying the MLP to the concatenation of the outputs of the attention heads rather than to their sum.

Although our constructions assume for $N=2$ source points, it seems feasible to generalize them to larger $N$. However, the major drawback of such a generalization is that the size of the transformer will depend on $N$. Even the simple step of calculating the majority between $N$ possible terms does not seem to be possible without at least a linear dependence on $N$. On the other hand, \cref{fact:invariant_upper} shows that the target function can be approximated for any $N$ using a single full rank attention. We conjecture that such a dependence on $N$ is necessary when using low-rank attention:

\begin{conjecture}
\label{conjecture:N}
    There is no multi-layer transformer (with fixed size and weight matrices) of rank $r < d$ that approximates the target of \cref{eq:target} for all $N$.
\end{conjecture}

That is, while it may be possible to construct a transformer that approximates the target for a given fixed $N$ (as we do above), we conjecture that there is no such construction that is independent of $N$.
Proving or refuting the above conjecture would have very different implications.
A counterexample would mean that the the weakness of low-rank can be compensated by depth, and thus the rank does not play a decisive role in the expressive power of multi-layer transformers.
A proof would show that, even in the multi-layer case, low-rank attention is fundamentally weaker than high-rank attention.

\section{Experiments}
\label{sec:experiments}
In this section, we complement our theoretical results with experiments on a broader class of architectures.
We train off-the-shelf transformers---which include multiple layers of self-attention, MLP layers, skip connections, and normalization---on a slight modification of the nearest neighbor function.
Our experiments confirm the weakness of low-rank attention in this setting.
They also show that the full-rank construction of \cref{fact:invariant_upper} is easily learned by gradient descent.
All code is available at \url{https://github.com/NoahAmsel/attention-formers}.

\paragraph{Model and training details} We use the Pytorch implementation of transformer encoders \cite{paszke2019pytorch} with two modifications.
First, we generalize the standard scaling $H = d/r$, allowing $H$ to be any multiple of $d/r$.
(In particular, we try $H = d^{1.5}/r$ and $H = d^2/r$.)
Second, we replace the layer normalization with RMSNorm \cite{NEURIPS2019_1e8a1942}, a standard choice in modern transformers \cite{touvron2023llama,10.5555/3648699.3648939} that is also better suited to our target function.
We train with biases, but preliminary experiments showed that these make little difference.\footnote{Note that biases in the key, query, and value transformations have a different role from additive positional encodings. These biases differ between heads but are constant across tokens; in contrast, the positional encodings differ between tokens but not heads.
The biases implemented by Pytorch are also slightly different from those studied in \cref{sec:biased}.}
We run each experiment on a single Nvidia GPU (usually a V100) for no more than a few hours.

Since we are using self-attention, there is no distinction between the source and target points.
The $N$ input points are drawn uniformly and i.i.d. from $\sphere$, and they are not constrained to be orthogonal.
We change our target function accordingly.
For each input point, the target now outputs whichever of the other points is \emph{farthest} from it.
We output the farthest instead of the nearest point because otherwise, each point would map to itself.
The loss function is the average mean squared error over the $N$ points.
We do not use any attention mask.
In particular, we allow points to attend to themselves.
Our dataset is synthetic, so we train and test on a stream of freshly generated samples that never repeat.
We train on $10^5$ batches of size 256 each.
For all experiments, we use AdamW with the same learning rate of 0.01 and a learning rate schedule of cosine annealing with a linear warm-up.

\begin{figure}
\centering
\includegraphics[width=\textwidth]{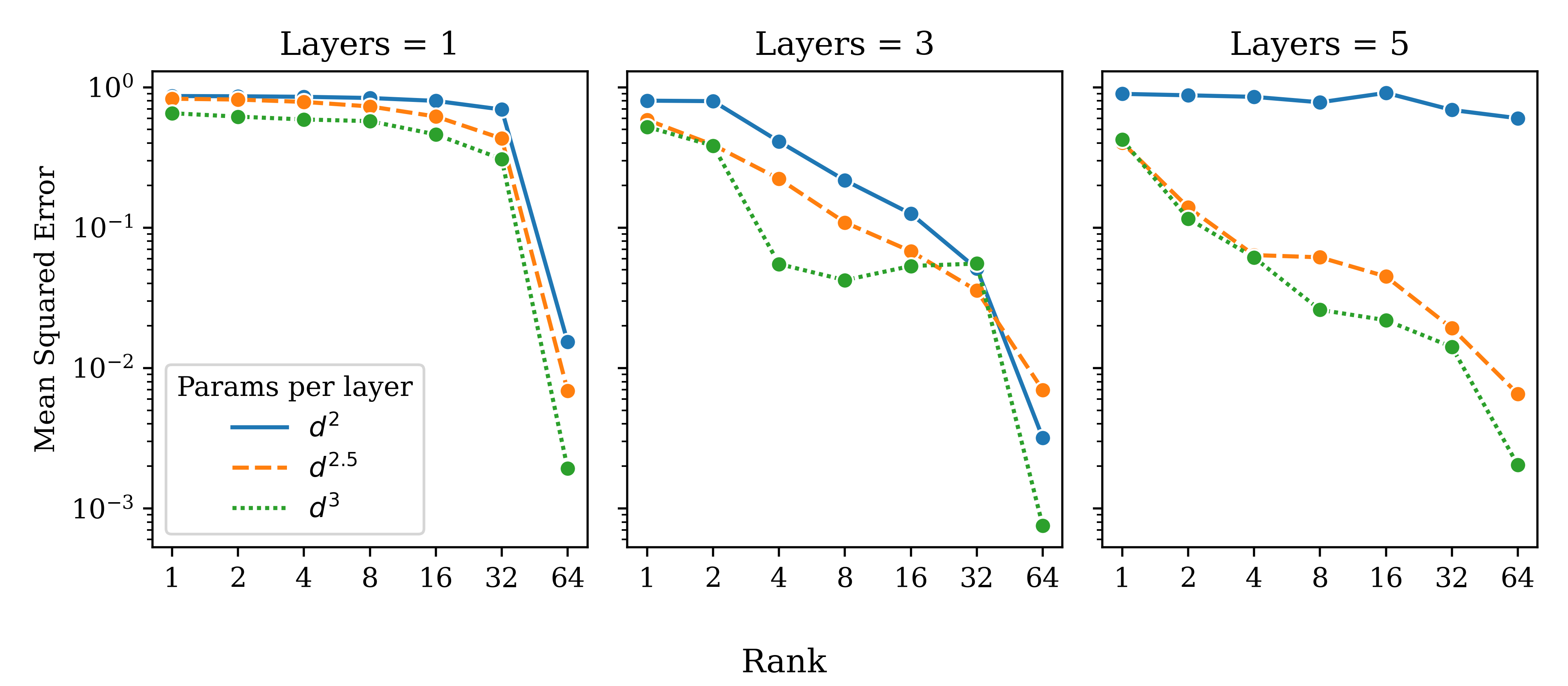}
\caption{Standard transformers trained on the farthest neighbor function. 
The dimension is $d = 64$ and the number of input points is $N = 16$.
Line shows best of five runs (except for $L=3, \text{params} = d^3, r \in \{16, 32\}$, which are best of eight).
Across different numbers of layers and heads, high-rank models significantly outperform low-rank models with the same number of parameters.}
\label{fig:main}
\end{figure}

\paragraph{Rank separation} Our first experiment studies the importance of rank across various numbers of heads ($H$) and layers ($L$).
We fix the dimension $d = 64$ and the number of points $N = 16$.
In this experiment, we use no positional encodings.
\cref{fig:main} plots the results, showing the best of five runs for each setting.
Each line uses a different number of heads, but the number of parameters per attention layer, $rdH = d^{c+1}$, is kept constant within each.
The standard scaling is $d^2$ parameters per layer.
When $L=1$, the results suggest that using full-rank ($r=64$) is necessary and sufficient to learn the target function accurately; even $2d$ heads of rank $d/2$ fails.
For $L>1$, the trade-off between rank and accuracy is more favorable, but low-rank attention still significantly underperforms full-rank attention, even when it gets to use more parameters.
The standard five layer transformers (that is, $L=5$, parameters per layer $= d^2$) seem to suffer from optimization difficulties on this problem.
Excluding that case, the best-performing model that is not full-rank ($L=5$, $d^3$ parameters per layer, $r = 32$) performs no better than the worst full-rank model ($L=1$, $d^2$ parameters per layer, $r=64$) despite having 80x more parameters in its attention layers.
In short, a standard transformer with $H=1$ performs much better on this task than one with $H$ even moderately larger.

\begin{figure}
\centering
\includegraphics[width=\textwidth]{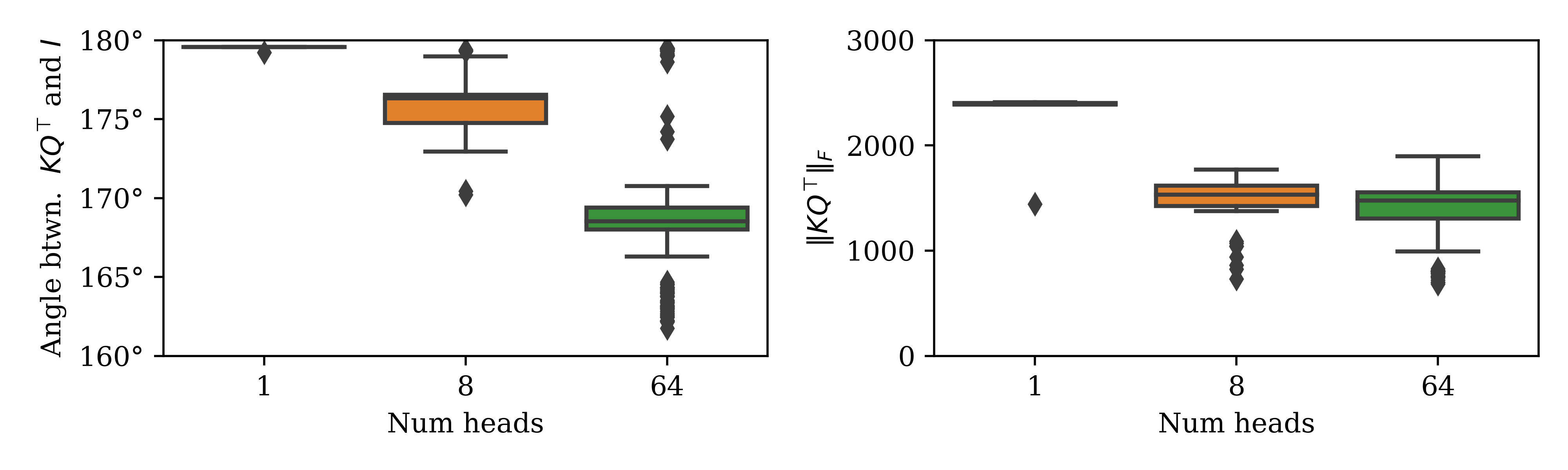}
\caption{Properties of learned $\mK \mQ^\top$ matrices for full-rank models with one layer.
Boxplots show distribution over heads from five runs, each on a model which has between 1 and 64 full-rank heads. Left panel plots Frobenius angle with the identity: $\arccos\left(\inner{\mK \mQ^\top, \mI}_{\mathsf F} / (\|\mK \mQ^\top\|_{\mathsf F} \|\mI\|_{\mathsf F})\right)$. Results show that $\mK\mQ^\top$ nearly equals $-c\mI$ for $c > 1000$ in all cases.}
\label{fig:perfect}
\end{figure}

\paragraph{Full-rank solution}
In the full-rank case the transformer learns the target, but what representation has it learned?
\cref{fig:perfect} suggests that, in some cases, it is very nearly the construction of \cref{fact:invariant_upper}.
Recall that in \cref{fact:invariant_upper}, we use a hardmax attention head with $\mK_h \mQ_h^\top = \mI$.
In our experiments however, we use the farthest neighbor target function and softmax heads, so the corresponding construction is $\mK_h \mQ_h^\top = -c\mI$ for $c \gg 1$.
The first panel shows the median Frobenius angle between the matrices $\mK_h \mQ_h^\top$ and $\mI$ learned by the full-rank, single layer models in the previous experiment.
This shows that $\mK_h \mQ_h^\top$ very nearly equals $-\mI$ up to a constant factor.
Moreover, as the second panel shows, the norm of this matrix is large, which causes the softmax to act like a hardmax.
Results are similar for three layer networks with a single full-rank head, but when $L > 1$ and $H > 1$, it seems the network learns some other, less interpretable strategy to represent the target.

\begin{figure}
\centering
\includegraphics[width=.8\textwidth]{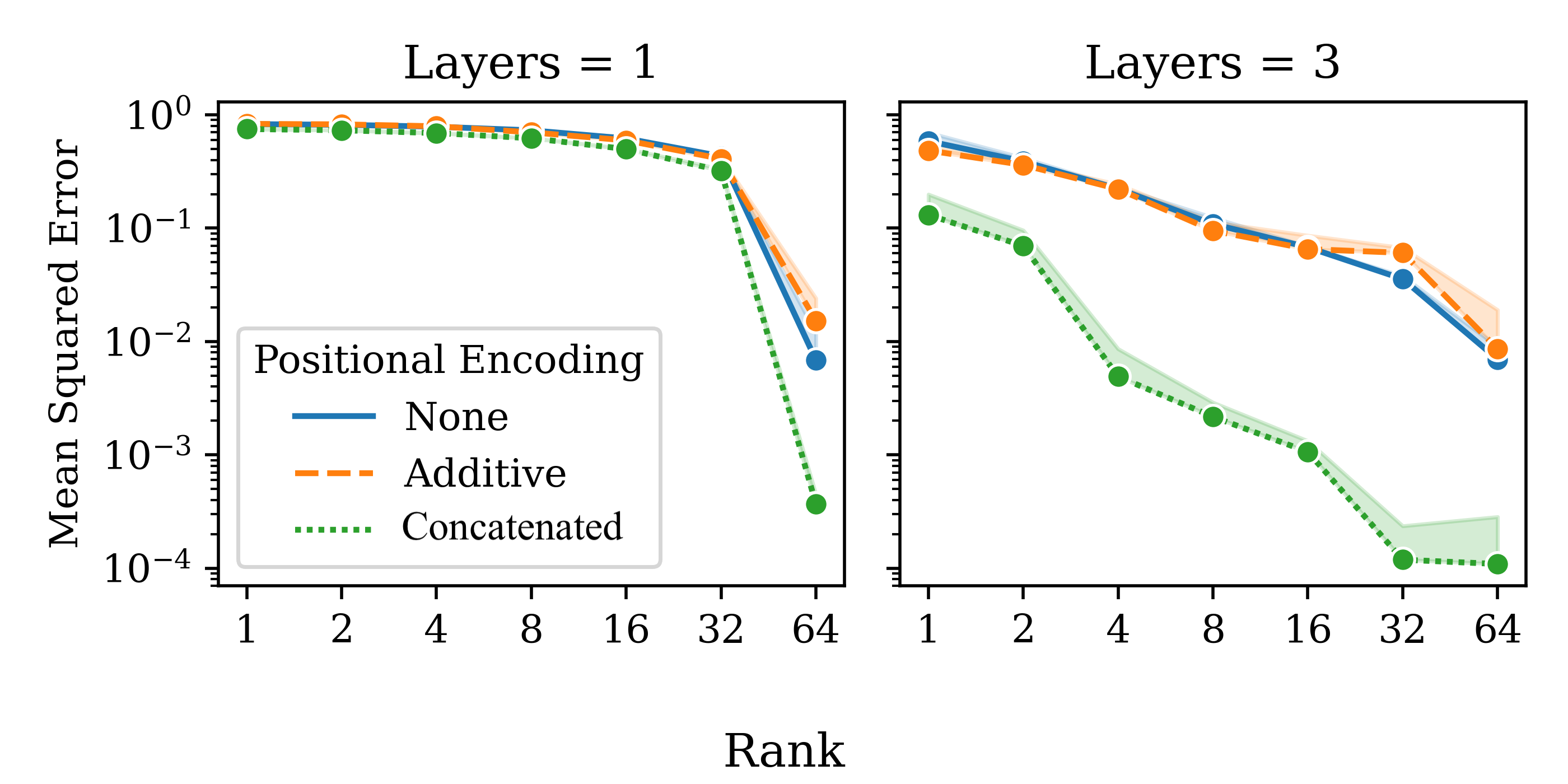}
\caption{Standard transformers with positional encodings ($d=64$, $N=16$).
Line shows best of five runs; shaded region shows range over five runs.
Positional encodings help when the encodings are concatenated to the inputs and there are multiple layers (cf. \cref{thm:majority_positional}). Otherwise, they do not help.}
\label{fig:posn}
\end{figure}

\paragraph{Positional encodings}
Since our target function is permutation-invariant, no positional information exists in the data.
However, in \cref{sec:majority}, we showed that concatenated positional encodings can help low-rank attention succeed when $L>1$ by giving the model extra dimensions of scratch space.
The positional encoding schemes used in practice, like additive encodings \cite{NIPS2017_3f5ee243}, RoPE \cite{su2024roformer} and ALiBi \cite{press2022train}, cannot be used in this way, being versions of the generalized attention heads studied in this paper.
In \cref{fig:posn}, we experiment with positional encodings.
As expected, additive attention fails to help low-rank attention at all.
The left panel shows that when $L=1$, concatenated positional encodings fail too.
However, when $L=3$, concatenated positional encodings yield dramatic improvements, a finding that accords with \cref{thm:majority_positional}.

\begin{figure}
\centering
\includegraphics[width=.8\textwidth]{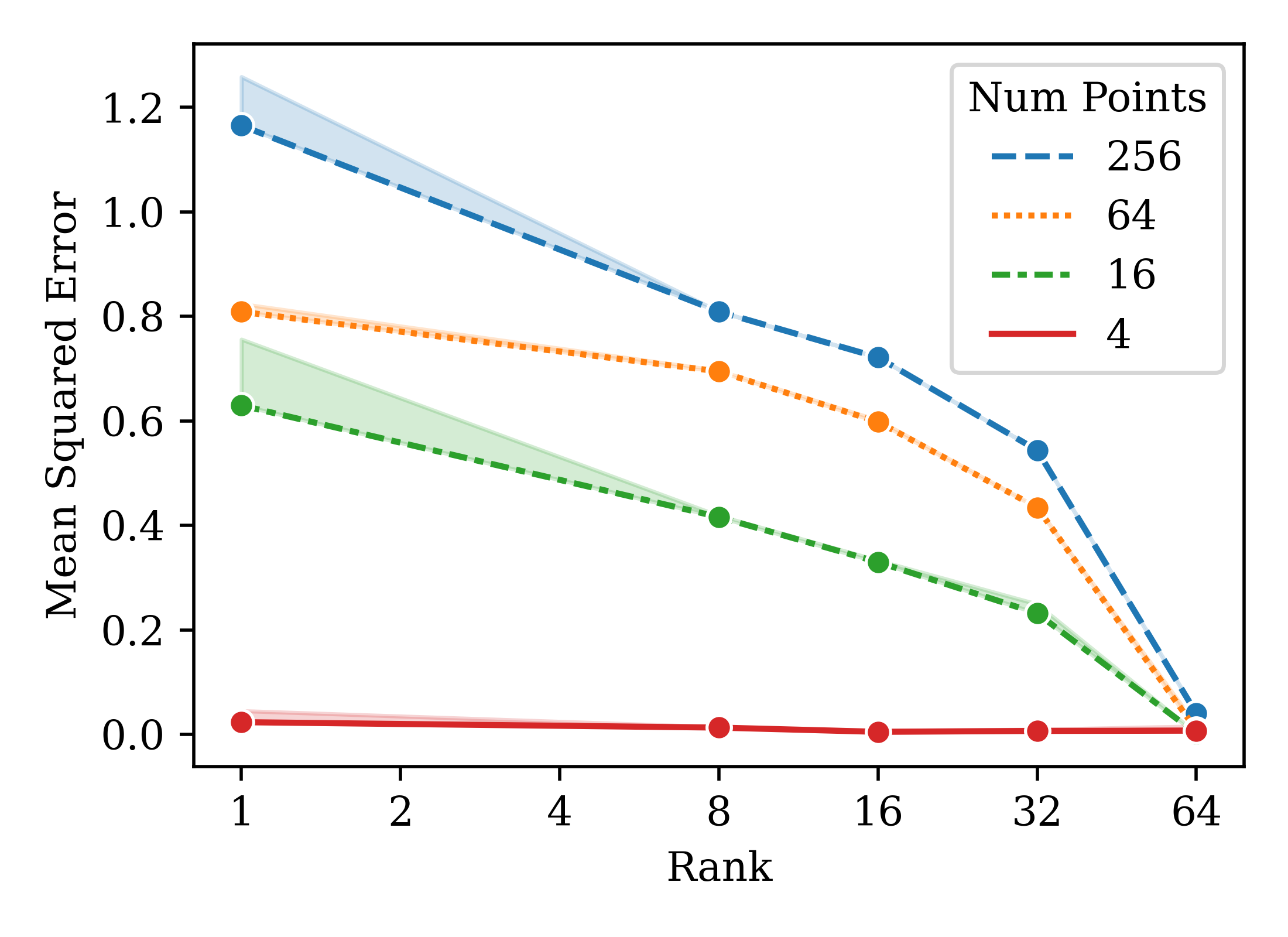}
\caption{
Effect of the number of points ($N$) on the difficulty of learning the farthest neighbor function.
Full-rank attention learns an accurate representation across many $N$s, but the performance of low-rank attention degrades as $N$ grows. 
Dimension is $64$. All models have two layers with $H = d^2 / r$ heads each. Line shows best of five runs; shaded region shows range over five runs.
}
\label{fig:role_of_N}
\end{figure}

\paragraph{Role of $N$} In \cref{fig:role_of_N}, we explore how the number of input points $N$ affects the difficulty of learning the target function.
We fix $d = 64$, $H = d^2/r$, and the number of layer $L = 2$.
The results show that, as predicted by \cref{fact:invariant_upper}, the full-rank heads learn the target accurately across a range of $N$.
However, the low-rank heads suffer declining accuracy as $N$ grows.
This accords with \cref{conjecture:N}, which predicts that low-rank transformers of a fixed size fail to accurately represent the target for sufficiently large $N$.

\section{Conclusions and Limitations}
\label{sec:conclusion}

In this paper, we have investigated the role of rank in attention mechanisms.
We question the nearly universal practice of trading off the rank and the number of heads according to $H = d/r$.
We show that for a simple and natural target function inspired by semantic search, low-rank attention is fundamentally weaker than full-rank attention, even when $H \gg d/r$.
We demonstrate this strict separation between the low-rank and high-rank regimes both theoretically, by proving hardness of approximation in the shallow setting, and empirically, through experiments with off-the-shelf transformers.
Our results thus hint at a potentially beneficial tradeoff between number of heads and rank that remains largely unexplored in applications. 

That said, our theoretical analysis is inherently limited to the study of shallow transformers, and our results of \Cref{sec:majority} illustrate how adding depth may overcome the limitations of low-rank self-attention in some cases.
However, we hope that our results will motivate theoreticians and practitioners to more carefully consider the settings and scalings of transformer hyperparameters.
In particular, they suggest that theoretical models that use full-rank attention may not accurately describe transformers used in practice, and that much remains to be understood about the successes and failure modes of attention-based architectures.

Several open questions remain for future work.
The basic transformer architecture of \cite{NIPS2017_3f5ee243} allows the user to set a number of hyperparameters.
Despite the ubiquity of this architecture, hyperparameter settings other than the embedding dimension and number of layers are almost never significantly changed; see \cref{sec:hyperparams}.
While considerable prior work has studied scaling laws for the dimension and number of layers, we believe that future research should also consider the other hyperparameters and seek to understand the trade-offs, dependencies, and scaling laws between them.
Here, we focus on the query/key rank and its relationship to the number of heads, but the depth and width of the MLPs and value/output rank are also of interest.

Additionally, the rotational invariance of the input data distribution is instrumental in establishing our lower bounds. Given the inherently discrete nature of text-based transformers, a natural question is to understand how to generalize our techniques beyond the rotationally-invariant setting. 
Another direction for future work is to understand the relationship between the rank and the context length.
Focusing on the $N=2$ case suffices for us to prove rank separation, but we believe a similar result should hold at least for all $N \leq d$;
\cref{fig:role_of_N} provides preliminary experimental evidence.
Understanding the $N>2$ case may also help address a final open question:
What is the relationship between rank and depth?
In particular, does \cref{conjecture:N} hold?

\paragraph{Acknowledgements:} 
This work was partially supported by the Alfred P. Sloan Foundation, and awards NSF RI-1816753, NSF CAREER CIF 1845360, NSF CHS-1901091 and NSF DMS-MoDL
2134216. We thank Ohad Shamir for useful discussions while this work was being completed. 

\bibliography{final_bib}
\bibliographystyle{alpha}

\appendix

\section{Hyperparameters of Transformer}
\label{sec:hyperparams}
The transformer architecture \cite{NIPS2017_3f5ee243} leaves the user free to set the following hyperparameters:
\begin{itemize}
\item The embedding dimension ($d$) 
\item The number of layers ($L$)
\item The width of the MLPs ($w$)
\item The depth of the MLPs ($D$)
\item The rank of the $\mW_Q$ and $\mW_K$ matrices for each head ($r$)
\item The rank of the $\mW_V$ and $\mW_O$ matrices for each head ($r_2$)
\item The number of attention heads in each layer ($H$)
\end{itemize}
In this paper, we consider the dimension $d$ to be given by the domain of the target function, rather than being a hyperparameter as in language modeling.
As \cref{table:hyperparams} shows, only $d$ and $L$ have been significantly changed relative to the original model.
For all models of which we are aware, $w$ lies within a factor of two from \cite{NIPS2017_3f5ee243}, $r$ lies within a factor of four, and $D$ and $r_2$ are not changed at all.
$H$ has been scaled, but always according to the standard scaling (up to a factor of 2).

\begin{table}[ht]
    \caption{
    Hyperparameter settings of popular transformer models (largest versions reported).
    Except for $d$ and $L$, they are strikingly consistent.
    See text of \cref{sec:hyperparams} for notation.}
  \label{table:hyperparams}
  \centering
  \begin{tabular}{clccccccc}
    \toprule
    Year & Model & $d$ & $L$ & $w$ & $D$ & $r$ & $r_2$ & $H$ \\
    \midrule
    2017 & Attention is all you need \cite{NIPS2017_3f5ee243} & 512 & 6 & $4d$ & 2 & 64 & $r$ & $d/r$ \\
    2018 & GPT, GPT-2 \cite{radford2018improving,radford2019language} & 768 & 12 & $4d$ & 2 & 64 & $r$ & $d/r$ \\
    2019 & Bert-Large \cite{devlin-etal-2019-bert} & 1,024 & 24 & $4d$ & $2$ & 64 & $r$ & $d/r$ \\
    2021 & ViT-Huge \cite{dosovitskiy2021an} & 1280 & 32 & $4d$ & 2 & 80 & $r$ & $d/r$\\
     & CLIP (text encoder) \cite{pmlr-v139-radford21a} & 1,024 & 12 & $4d$ & 2 & 64 & $r$ & $d/r$ \\
     &Jurassic-1 & 13,824 & 76 & $4d$ & 2 & 144 & $r$ & $d/r$\\
     & Gopher 280B \cite{rae2021scaling} & 16,384 & 80 & $4d$ & 2 & 128 & $r$ & $d/r$ \\
     &LaMDA \cite{51115} & 8192 & 64 & $8d$ & 2 & 128 & $r$ & $2d/r$ \\
    2022 & Chinchilla 70B \cite{NEURIPS2022_c1e2faff} & 8,192 & 80 & $4d$ & 2 & 128 & $r$ & $d/r$\\
     &GPT-3 \cite{brown2020language} & 12,288 & 96 & $4d$ & 2 & 128 & $r$ & $d/r$\\
    2023 & PaLM \cite{10.5555/3648699.3648939} & 18,432 & 118 & $4d$ & 2 & 256 & $r$ & $2d/3r$ \\
     & LLaMA, Llama-2 \cite{touvron2023llama,touvron2023bllama} & 8,192 & 80 & $8d/3$ & $2$ & 128 & $r$ & $d/r$ \\
    2024 & OLMo \cite{groeneveld2024olmo} & 8,192 & 80 & $8d/3$ & 2 & 128 & $r$ & $d/r$\\
    \bottomrule
  \end{tabular}
\end{table}

\section{Proofs from \texorpdfstring{\cref{sec:invariant}}{Section 4}}
\label{sec:proof_invariant_lower}
In this section, we prove the upper bound \cref{fact:invariant_upper}, the lower bound \cref{thm:invariant_lower} and some important properties relating to the approximation of the target by random heads.

We begin with the proof of \cref{fact:invariant_upper} in \cref{appen:proof of upper invariant}.
In \cref{sec:spherical_harm}, we review the basics of spherical harmonics and describe the corresponding family of ultraspherical orthogonal polynomials on the interval.
In \cref{sec:ortho_basis}, we construct a basis for functions of pairs of points on the sphere that we will use to analyze the target and the attention mechanism.
In \cref{sec:target_expansion}, we show how to expand the target function in this basis, proving the critical properties of slow spectral decay and rotational invariance between basis elements of the same degree.
In \cref{sec:head_expansion}, we expand a single attention head in this basis, showing that the number of basis elements with which it is correlated is limited by the rank of the attention head.
In \cref{sec:main_proof_main_step}, we use these results to obtain a lower bound on the error of approximation that depends only on certain universal constants related to the spherical harmonics, particularly the number of spherical harmonics of a given degree and the coefficients of the ultraspherical expansion of the sign function.
In \cref{sec:asymptotics}, we analyze this expression to derive a bound on the necessary number of heads that depends only on the dimension $d$, the rank $r$, and the error level $\epsilon$.
Finally, in \cref{sec:random_features_proof}, we analyze a construction that approximates the target function using random rank-1 heads.

\subsection{Proof of \texorpdfstring{\cref{fact:invariant_upper}}{Fact 1}}\label{appen:proof of upper invariant}

Let $\epsilon > 0$. We set $\mV = \mI$, $\mK \mQ^\top = \alpha \mI$ for $\alpha > 0$ to be chosen later. Since $\vx_i,\vy\sim\unif(\sd)$, for every $i\in\{1,\dots,N\}$, there exists $\delta > 0$ (which depends on $\epsilon$) such that for the set:
\begin{equation}
A_\delta:=\{(\vx_1,\dots,\vx_N,\vy)\in(\sd)^{N+1}: \forall i\neq j,~ |(\vx_i-\vx_j)^\top \vy| > \delta\}~,
\end{equation}
we have that $\Pr((\vx_1,,\dots,\vx_N,\vy) \notin A_\delta) \leq \frac{\epsilon}{2}$. 
Note that:
\begin{equation}
\mX\sm(\alpha\mX^\top  \vy)\underset{\alpha\rightarrow\infty}{\longrightarrow}\arg\max_{\vx_i} (\vx_i^\top \vy) = \arg\max_{\vx_i}\norm{\vx_i - \vy}^2~,
\end{equation}
where the convergence is uniform 
on $A_\delta$, and the equality follows since all the vectors are from the unit sphere. In particular, there exists $\alpha > 0$ such that:
\begin{equation}
\sup_{(\vx_1,\dots,\vx_N,\vy)\in A_\delta}\norm{\mX\sm(\alpha\mX^\top  \vy) - \arg\max_{\vx_i}\norm{\vx_i - \vy}^2}^2 \leq \frac{\epsilon}{2}~.
\end{equation}
Combining both bounds and taking expectation over the vectors finishes the proof.

\subsection{Spherical Harmonics}
\label{sec:spherical_harm}
We begin by reviewing some basic results from the theory of spherical harmonics.
Let $\tau(\cdot)$ denote the uniform distribution over $\sphere$ and define the inner product $\inner{\cdot, \cdot}_\tau$ over $L^2(\sphere)$ as follows
\begin{equation} \inner{f, g}_{\tau} := \int_{\sphere} f(\vx)g(\vx) d\tau(\vx) \end{equation} 
A polynomial $H: \R^d \to \R$ is called harmonic and degree-$\ell$ homogeneous if
\begin{equation} \nabla^2 H = 0, \qquad H(a\vx) = a^{\ell}H(\vx) \end{equation}
A spherical harmonic of degree $\ell$ is the restriction of a harmonic homogeneous polynomial to the sphere $\sphere$.
That is, a function $Y : \sphere \to \R$ is a spherical harmonic of degree $\ell$ if and only if the $\R^d \to \R$ function defined by
\begin{equation} \vx \mapsto \|\vx\|^{\ell} Y\left(\frac{\vx}{\|\vx\|}\right) \end{equation}
is a harmonic homogeneous polynomial of degree $\ell$.
The set of spherical harmonics of degree $\ell$ on $\sphere$ form a function space $\Fcal_\ell \subset L^2(\sphere)$.
These subspaces have the following dimensions (Theorem 4.4 of \cite{frye2012spherical}):
\begin{equation}
\label{eq:ndl}
N(d, \ell) := \dim \Fcal_\ell = \frac{2\ell + d - 2}{\ell} \binom{\ell + d - 3}{\ell - 1}~.
\end{equation}
The reason spherical harmonics are so useful is that the $\Fcal_\ell$ are linearly independent, and their direct sum is $L^2(\sphere)$.
That is, if $\{Y_\ell^j\}_{j=1}^{N(d,\ell)}$ is an orthonormal basis of $\Fcal_\ell$, then $\cup_{\ell=0}^{\infty} \{Y_\ell^j\}_{j=1}^{N(d,\ell)}$ is an orthonormal basis of $L^2(\sphere)$ with respect to $\inner{\cdot, \cdot}_\tau$.

For a unit vector $\ve$, let $u_d$ denote the distribution of $\vx^\top \ve$ when $\vx \sim \tau$.
Then for $t \in [-1, 1]$,
\begin{equation} u_d(t) := \frac{A_{d-2}}{A_{d-1}} \cdot (1-t^2)^{\frac{d-3}2} \end{equation}
where $A_{d-1}$ is the surface area of $\sphere$ (see Lemma 4.17 of \cite{frye2012spherical}).
Define the following inner product over functions mapping $[-1,1] \to \R$:
\begin{equation} \inner{f,g}_{u_d} := \int_{-1}^1 f(t)g(t)u_d(t)dt \end{equation}
The ultraspherical polynomials $P_\ell : [-1,1] \to \R$ for $\ell \in \mathbb N_{\geq 0}$ are defined by the following properties:
\begin{enumerate}
\item $P_\ell$ has degree $\ell$
\item $\ell \neq \ell' \Longleftrightarrow \inner{P_\ell, P_{\ell'}}_{u_d} = 0$
\item $P_\ell(1) = 1$
\end{enumerate}
These polynomials form an orthogonal basis for $L^2([-1, 1], u_d)$, which includes all bounded functions on $[-1,1]$.
Moreover, they are intimately connected to the spherical harmonics.
We exploit three such connections.
First (Equation 4.30 of \cite{frye2012spherical})
\begin{equation}\label{eq:geg_norm}
\|P_\ell\|_{u_d}^2 = \frac1{N(d, \ell)}
\end{equation}
Second, the addition formula states that each ultraspherical polynomial can be expressed in terms of the spherical harmonics of the same degree and vice versa (Theorem 4.11\footnote{Note that \cite{frye2012spherical} has an extra factor of $A_{d-1}$ in the theorem statement. This is because they use a different normalization for the spherical harmonics.} of \cite{frye2012spherical})
\begin{equation}\label{eq:addition}
P_\ell(\vx^\top \vy) = \frac1{N(d, \ell)}\sum_{j=1}^{N(d,\ell)} Y_\ell^j(\vx) Y_\ell^j(\vy)
\end{equation}
Finally, the Hecke-Funk formula (Theorem 4.24 of \cite{frye2012spherical}) gives the relationship between the ultraspherical expansion of $t \mapsto f(t)$ and the spherical harmonic expansion of $\vy \mapsto f(\vx^\top \vy)$.
For any degree-$\ell$ spherical harmonic $Y_\ell$,
\begin{equation}\label{eq:hecke_funk}
\Big\langle f\big(\inner{\vx, \cdot}\big), Y_\ell\Big\rangle_\tau := \int_{\sphere} f(\vx^\top \vy) Y_{\ell}(\vy) d\tau(\vy) = Y_{\ell}(\vx) \inner{f, P_\ell}_{u_d}
\end{equation}

We will make use of the ultraspherical expansion of two particular functions:
\begin{definition}\label{def:geg_coeffs}
Let $\{\alpha_\ell\}$ be the ultraspherical series for $\arcsin$ and let $\{\eta_\ell\}$ be the ultraspherical series for $\sign$.
That is,
\begin{align}
\arcsin(t) &= \sum_{\ell = 0}^\infty \alpha_\ell \frac{P_\ell(t)}{\|P_\ell\|_{u_d}}\\
\sign(t) &= \sum_{\ell = 0}^\infty \eta_\ell \frac{P_\ell(t)}{\|P_\ell\|_{u_d}}
\qquad \forall t \in [-1,1]
\end{align}
\end{definition}

\subsection{Orthonormal Basis for Target and Attention Heads}
\label{sec:ortho_basis}
The goal of this section is to define the orthonormal basis that we will use to analyze the (surrogate) target and attention functions.
We define the input space for these functions as follows: $\Xcal = \sphere \times \sphere$.
We denote elements of this set by $(\vx, \vy)$ or $z$ for short.
For any two functions, define their tensorization by
\begin{equation} (f \otimes g)(z) = f(\vx)f(\vy) \end{equation}
We let $\bar \tau = \tau \otimes \tau$ be the uniform measure on $\Xcal$.
We also define a feature space $\Omega = \sphere \times \sphere$ and denote elements of this space by $(\vq, \vk)$ or $\omega$. Of course, $\Omega = \Xcal$, but since they are used in different contexts, we use separate notation for readability.

We define the feature mapping that we will use to analyze the surrogate target and attention functions:
\begin{definition}
\label{def:Tcal}
Define the ``rank-1 head'' function $\rho : \Xcal \times \Omega \to \{\pm 1\}$ by
\begin{equation}\label{eq:rank1_hm}\rho(z, \omega) := \sign\left(\vx^\top \vk \vq^\top \vy\right) \end{equation}
and the feature map linear operator $\Tcal : L^1(\Omega) \to L^2(\Xcal)$ by
\begin{equation} (\Tcal u)(z) := \int_{\Omega}\rho(z,\omega)u(\omega)d\bar\tau(\omega) \end{equation}
\end{definition}
The intuition is as follows.
For a fixed value of $\omega = (\vk, \vq)$, the function $\rho(\cdot, \omega)$ acts like a hardmax attention head with rank 1.
More precisely, if $\vx = \vx_1 - \vx_2$ and $\mV = \mI$, then $\rho(z, \omega)$ is the output of the head applied to the source $\vy$ and targets $\vx_1$ and $\vx_2$, projected onto $\vx$.
Furthermore, $\Tcal u$ is a weighted linear combination of all possible rank-1 hardmax heads.

We will construct a basis using functions of the form $\Tcal(Y \otimes Y')$ for spherical harmonics $Y$ and $Y'$.
The rationale for choosing this basis is as follows.
$\Tcal$ defines a positive semidefinite operator $\Tcal^* \Tcal : L^1(\Omega) \to L^2(\Omega)$, which is described by the following formula:
\begin{equation} (\Tcal^*\Tcal u)(\omega) = \int_{\Omega} \E_{z \sim \bar\tau} [\rho(z, \omega)\rho(z, \omega')] \cdot  u(\omega')d\bar{\tau}(\omega') \end{equation}
Functions of the form $Y \otimes Y'$ will turn out to be eigenfunctions of this operator.
To see why, we must first analyze the kernel $\E_{z \sim \bar\tau} [\rho(z, \omega)\rho(z, \omega')]$, which we do in the following lemma.

\begin{lemma}
\label{lem:dot_repr}
\begin{equation} \E_{z \sim \bar\tau} [\rho(z, \omega)\rho(z, \omega')] = \frac4{\pi^2} \arcsin(\vq^\top \vq')\arcsin(\vk^\top \vk') \end{equation}
\end{lemma}
\begin{proof}
To begin, we compute a closely related property -- the probability that the signs are equal:
\begin{align}
\Pr_{\vz \sim \bar\tau}\left[\rho(\vz, \omega) = \rho(\vz, \omega') \right]
&= \Pr_{\vz \sim \bar\tau}\left[\inner{\vx, \vk}\inner{\vq, \vy}\inner{\vx, \vk'}\inner{\vq', \vy} > 0\right]\\
\end{align}
Let $\theta$ be the angle between $\vq$ and $\vq'$ and let $\phi$ be the angle between $\vk$ and $\vk'$. We have
\begin{align}
    \Pr_{\vy}[\inner{\vy, \vq}\inner{\vy, \vq'} \geq 0] &= 1 - \frac{\theta}{\pi} \\
    \Pr_{\vx}[\inner{\vx, \vk}\inner{\vx, \vk'} \geq 0] &= 1 - \frac{\phi}{\pi} \\
    \Pr_{\vx,\vy}[\inner{\vy, \vq}\inner{\vy, \vq'} \geq 0 \wedge \inner{\vx, \vk}\inner{\vx, \vk'} \geq 0] &= \left(1 - \frac{\theta}{\pi}\right)\left(1 - \frac{\phi}{\pi}\right) \\
    \Pr_{\vx,\vy}[\inner{\vy, \vq}\inner{\vy, \vq'} \leq 0 \wedge \inner{\vx, \vk}\inner{\vx, \vk'} \leq 0] &= \frac{\theta}{\pi}\frac{\phi}{\pi} \\
    \Pr_{\vx,\vy}\left[
    \inner{\vx, \vk}
    \inner{\vx, \vk'}
    \inner{\vy, \vq}
    \inner{\vy, \vq'}
    \geq 0
    \right] &= \left(1 - \frac{\theta}{\pi}\right)\left(1 - \frac{\phi}{\pi}\right) + \frac{\theta}{\pi}\frac{\phi}{\pi}
\end{align}
A bit of algebra now shows
\begin{align}
\Pr_{\vz \sim \bar\tau}\left[\rho(\vz, \omega) = \rho(\vz, \omega') \right]
&= \left(1 - \frac{\theta}{\pi}\right)\left(1 - \frac{\phi}{\pi}\right) + \frac{\theta}{\pi}\frac{\phi}{\pi}\\
&= \frac12 + \frac2{\pi^2}\left(\frac{\pi}2 - \theta\right)\left(\frac{\pi}2 - \phi\right)
\end{align}
By definition, $\theta = \arccos(\inner{\vq,\vq'})$ and $\phi = \arccos(\inner{\vk, \vk'})$. Using the identity $\arcsin(z) = \pi/2 - \arccos(z)$, we obtain
\begin{equation} \Pr_{\vz \sim \bar\tau}\left[\rho(\vz, \omega) = \rho(\vz, \omega') \right] = \frac12 + \frac2{\pi^2}\arcsin(\vq^\top \vq')\arcsin(\vk^\top \vk') \end{equation}
Finally,
\begin{align}
\E_{\vz \sim \bar\tau}\left[\rho(\vz, \omega)\rho(\vz, \omega') \right]
&= \Pr_{\vz \sim \bar\tau}\left[\rho(\vz, \omega) = \rho(\vz, \omega') \right] - \Pr_{\vz \sim \bar\tau}\left[\rho(\vz, \omega) \neq \rho(\vz, \omega') \right]\\
&= 2\Pr_{\vz \sim \bar\tau}\left[\rho(\vz, \omega) = \rho(\vz, \omega') \right] - 1\\
&= \frac4{\pi^2}\arcsin(\vq^\top \vq')\arcsin(\vk^\top \vk')
\end{align}
\end{proof}

The above lemma gives us a handy expression for $\Tcal^* \Tcal$ that allows to show the following:
\begin{lemma}\label{lem:eigenfunctions}
Let $Y, Y'$ be spherical harmonics of degrees $\ell$ and $\ell'$, respectively. Then $Y\otimes Y'$ is an eigenfunction of the operator $\Tcal^* \Tcal$:
\begin{equation} \Tcal^* \Tcal(Y\otimes Y') =  \frac4{\pi^2}\frac{\alpha_{\ell}\alpha_{\ell'}}{\sqrt{N(d,\ell)N(d,\ell')}} \cdot Y\otimes Y' \end{equation}
\end{lemma}
\begin{proof}
It is easily seen that
\begin{equation} (\Tcal^* f)(\cdot) = \int_{\Xcal} \rho(z,\cdot) f(z)d\bar\tau(z) \end{equation}
and thus, substituting and changing the order of integration
\begin{align}
[\Tcal^* \Tcal(Y\otimes Y')](\omega)
= \int_{\Omega} \E_{z \sim \bar\tau} [\rho(z, \omega)\rho(z, \omega')]\cdot (Y \otimes Y')(\omega')d\bar\tau(\omega')
\end{align}
Applying \cref{lem:dot_repr} and expanding $d\bar\tau(\omega)$ and $Y\otimes Y'$,
\begin{align}
&= \frac4{\pi^2} \int_{\Omega} \arcsin(\vq^\top \vq')\arcsin(\vk^\top \vk')\cdot (Y \otimes Y')(\omega')d\bar\tau(\omega')\\
&= \frac4{\pi^2} \int_{\sphere} \arcsin(\vq^\top \vq') Y(\vq')d\tau(\vq') \cdot \int_{\sphere} \arcsin(\vk^\top \vk') Y'(\vk')d\tau(\vk') 
\end{align}
Applying the Hecke-Funke formula (\cref{eq:hecke_funk}) to the first integral,
\begin{align}
\int_{\sphere} \arcsin(\vq^\top \vq') Y(\vq')d\tau(\vq')
&= Y(\vq) \inner{\arcsin, P_\ell}_{u_d}\\
&= Y(\vq) \inner{\arcsin, \frac{P_\ell}{\|P_\ell\|_{u_d}}}_{u_d} \cdot \|P_\ell\|_{u_d}\\
&= Y(\vq) \frac{\alpha_\ell}{\sqrt{N(d,\ell)}}
\end{align}
By the same logic, the second integral equals $Y'(\vk') \cdot \alpha_{\ell'} / \sqrt{N(d,\ell)}$.
Combining these proves the lemma.
\end{proof}

The previous lemma immediately implies that the functions $\Tcal(Y \otimes Y')$ form an orthogonal basis:
\begin{lemma}
\label{lem:ortho_basis}
Let $B$ be a set of orthonormal spherical harmonics.
Then the elements of $\{\Tcal(Y \otimes Y') \mid Y, Y' \in B\}$ are also orthogonal.
Furthermore, if $Y$ and $Y'$ have degrees $\ell$ and $\ell'$, then
\begin{equation} \|\Tcal(Y \otimes Y')\|_{\bar\tau}^2 = \frac4{\pi^2}\frac{\alpha_{\ell}\alpha_{\ell'}}{\sqrt{N(d,\ell)N(d,\ell')}} \end{equation}
\end{lemma}
\begin{proof}
Let $Y_i, Y_j, Y_{i'}, Y_{j'} \in B$.
Let $Y_i'$ have degree $\ell$ and $Y_j'$ have degree $\ell'$.
Then
\begin{align}
\inner{\Tcal(Y_i \otimes Y_j), \Tcal(Y_{i'} \otimes Y_{j'})}
&= \inner{Y_i \otimes Y_j, \Tcal^* \Tcal(Y_{i'} \otimes Y_{j'})}\\
&= \inner{Y_i \otimes Y_j, Y_{i'} \otimes Y_{j'}} \cdot \frac4{\pi^2}\frac{\alpha_{\ell}\alpha_{\ell'}}{\sqrt{N(d,\ell)N(d,\ell')}}
\end{align}
But $\inner{Y_i \otimes Y_j, Y_{i'} \otimes Y_{j'}}$ is one if $Y_i = Y_{i'}$ and $Y_j = Y_{j'}$, and zero otherwise.
\end{proof}

\subsection{Expansion of the Target Function}
\label{sec:target_expansion}
We define a surrogate target function that will turn out to be the relevant one for our analysis.
\begin{definition}
\label{def:surrogate_target}
The surrogate target function $\tilde f: \Xcal \to \R$ is
\begin{equation} \tilde f(z) := \sign(\vx^\top \vy) \end{equation}
\end{definition}

After a change of variables $(\vx, \vw) = (\vx_1 - \vx_2, \vx_1 + \vx_2)$, our original target function reduces simply to $\tilde f(z) \vx + \vw$.
We now wish to expand $\tilde f$ in the basis $\{\Tcal(Y \otimes Y')\}$.
We will first need the following lemma, which describes the correlation of a rank-1 head with the surrogate target function.

\begin{lemma}
\label{lem:target_T}
Fix $\omega = (\vq, \vk) \in \Omega$.
Then
\begin{equation}\inner{\tilde f, \rho(\cdot, \omega)}_{\bar\tau} = \sum_{\ell=0}^{\infty} c_\ell P_\ell(\vq^\top \vk) \end{equation}
where
\begin{equation} c_\ell = \frac2{\pi} \eta_\ell \alpha_\ell \end{equation}
\end{lemma}
\begin{proof}
By definition,
\begin{align}
\inner{\tilde f, \rho(\cdot, \omega)}_{\bar\tau}
= \E_{\vx, \vy \sim \tau}\left[\sign(\vx^\top \vy)\sign(\vx^\top \vk \vq^\top \vy) \right]
\end{align}
Let $\tau_+$ denote the uniform measure on the hemisphere $\{ \vx \in \sphere \mid \vx^\top\vk \geq 0\}$, and $\tau_-$ the uniform measure on the opposite hemisphere.
Then we can decompose the expectation as follows:
\begin{align}
\E_{\vx, \vy \sim \tau}[\sign(\vx^\top \vy)\sign(\vx^\top \vk \vq^\top \vy)]
&= \frac12 \E_{\substack{\vx \sim \tau_+\\\vy \sim \tau}}[\sign(\vx^\top \vy)\sign(\vq^\top \vy)]\\
&- \frac12 \E_{\substack{\vx \sim \tau_-\\\vy \sim \tau}}[\sign(\vx^\top \vy)\sign(\vq^\top \vy)]
\end{align}
Given any fixed unit vectors $\vx$, $\vq$ we have that 
\begin{equation}
    \Pr_{\vy} [\sign(\vx^\top \vy) = \sign(\vq^\top \vy)] = 1 - \frac{\arccos (\vx^\top \vq)}{\pi}
\end{equation}
Therefore,
\begin{align}
\E_{\vy}[\sign(\vx^\top \vy) \sign(\vq^\top \vy)]
&= \Pr_{\vy} [\sign(\vx^\top \vy) = \sign(\vq^\top \vy)]
- \Pr_{\vy} [\sign(\vx^\top \vy) \neq \sign(\vq^\top \vy)]\\
&= 2\Pr_{\vy} [\sign(\vx^\top \vy) = \sign(\vq^\top \vy)] - 1\\
&= 1 - \frac{2\arccos (\vx^\top \vq)}{\pi}
\end{align}
Plugging this into the expression above,
\begin{align}
&= \frac12 \E_{\vx \sim \tau_+}\left[1 - \frac{2\arccos (\vx^\top \vq)}{\pi}\right]
- \frac12 \E_{\vx \sim \tau_-}\left[1 - \frac{2\arccos (\vx^\top \vq)}{\pi}\right]\\
&= -\frac2{\pi}\left(\frac12 \E_{\vx \sim \tau_+}\left[\arccos (\vx^\top \vq)\right] - \frac12\E_{\vx \sim \tau_-}\left[\arccos (\vx^\top \vq)\right]\right)\\
&= -\frac2{\pi}\left(\frac12\E_{\vx \sim \tau_+}\left[\sign(\vx^\top \vk)\arccos (\vx^\top \vq)\right] + \frac12\E_{\vx \sim \tau_-}\left[\sign(\vx^\top \vk)\arccos (\vx^\top \vq)\right]\right)\\
&= -\frac2{\pi}\left(\E_{\vx \sim \tau}\left[\sign(\vx^\top \vk)\arccos (\vx^\top \vq)\right]\right)\\
\end{align}
Using the identity $\arccos(t) = \frac{\pi}2 - \arcsin(t)$ and the fact that $\E_{\vx}[\sign(\vx^\top \vk)] = 0$,
\begin{align}
&= \frac2{\pi} \E_{\vx \sim \tau}\left[\sign(\vx^\top \vk)\arcsin (\vx^\top \vq)\right]\\
&= \frac2{\pi} \Big\langle\sign(\inner{\cdot, \vk}), \arcsin(\inner{\cdot, \vq})\Big\rangle_\tau
\end{align}
We now expand $\sign(\inner{\cdot, \vk})$ and $\arcsin(\inner{\cdot, \vq})$ in a basis of spherical harmonics.
By Hecke-Funk,
\begin{align}
\Big\langle\sign(\inner{\cdot, \vk}), Y_\ell^j\Big\rangle_\tau &= Y_\ell^j(\vk) \inner{\sign, P_\ell}_{u_d} = Y_\ell^j(\vk) \eta_\ell \|P_\ell\|_{u_d} \\
\Big\langle\arcsin(\inner{\cdot, \vq}), Y_\ell^j\Big\rangle_\tau &= Y_\ell^j(\vq) \inner{\arcsin, P_\ell}_{u_d} = Y_\ell^j(\vq) \alpha_\ell \|P_\ell\|_{u_d} \\
\end{align}
Thus, writing the inner product in the basis of spherical harmnoics,
\begin{align}
\frac2{\pi} \Big\langle\sign(\inner{\cdot, \vk}), \arcsin(\inner{\cdot, \vq})\Big\rangle_\tau
&= \frac2{\pi} \sum_{\ell=0}^{\infty}\sum_{j=1}^{N(d, \ell)} \left(Y_\ell^j(\vk) \eta_\ell \|P_\ell\|_{u_d}\right)\left(Y_\ell^j(\vq) \alpha_\ell \|P_\ell\|_{u_d}\right)\\
&= \frac2{\pi} \sum_{\ell=0}^{\infty}\left( \eta_\ell \alpha_\ell  \|P_\ell\|_{u_d}^2 \sum_{j=1}^{N(d, \ell)} Y_\ell^j(\vk) Y_\ell^j(\vq)\right)
\end{align}
Applying the addition formula (\cref{eq:geg_norm}),
\begin{align}
&= \frac2{\pi} \sum_{\ell=0}^{\infty}\eta_\ell \alpha_\ell  \|P_\ell\|_{u_d}^2 N(d,\ell) P_\ell(\vk^\top \vq)\\
&= \sum_{\ell=0}^{\infty}\frac2{\pi} \eta_\ell \alpha_\ell  P_\ell(\vk^\top \vq)\\
\end{align}
\end{proof}

We now expand our surrogate target function $\tilde f$ in our basis $\{\Tcal(Y \otimes Y') \}$.
The following lemma shows that $\tilde f$ is orthogonal to any basis element for which $Y \neq Y'$, and that the coefficient of $\Tcal(Y \otimes Y')$ only depends only on the degree of $Y$.
That is, the energy of $\tilde f$ is evenly spread across all elements of $\{\Tcal(Y_\ell \otimes Y_\ell) \mid Y_\ell \in \Fcal_\ell\}$.
\begin{lemma}
\label{lem:target_expansion}
Let $Y, Y'$ be spherical harmonics of odd degree.
Let $\ell$ be the degree of $Y$.
Then
\begin{equation} \inner{\tilde f, \frac{\Tcal(Y \otimes Y')}{\|\Tcal(Y \otimes Y')\|_{\bar\tau}}}_{\bar\tau} = \frac{\eta_{\ell}}{\sqrt{N(d,\ell)}} \delta_{Y,Y'} \end{equation}
where $\delta_{Y,Y'} = \ones[Y = Y']$.
That is, if the basis element is built from two identical spherical harmonics of degree $\ell$, then its correlation with the target function depends only on $\ell$; otherwise it is zero.
\end{lemma}
\begin{proof}
Expanding, switching the order of the integrals, and applying \cref{lem:target_T},
\begin{align}
\inner{\tilde f, \Tcal(Y \otimes Y')}_{\bar \tau}
&= \int_{\Xcal} \int_{\Omega} \tilde f(\vz) \rho(\vz, \omega)(Y \otimes Y')(\omega) d\bar\tau(\omega) d\bar\tau(\vz)\\
&= \int_{\Omega} \inner{\tilde f, \rho(\cdot, \omega)}_{\bar\tau} (Y \otimes Y')(\omega) d\bar\tau(\omega)\\
&= \sum_{\ell'=0}^{\infty} c_{\ell'} \int_{\Omega} P_{\ell'}(\vq^\top \vk) (Y \otimes Y')(\omega) d\bar\tau(\omega)
\end{align}
Expanding the integral over $\Omega$ and applying Hecke-Funk (\cref{eq:hecke_funk}),
\begin{align}
&= \sum_{\ell'=0}^{\infty} c_{\ell'} \int_{\sphere}\int_{\sphere} P_{\ell'}(\vq^\top \vk) Y'(\vk) Y(\vq) d\tau(\vk) d\tau(\vq)\\
&= \sum_{\ell'=0}^{\infty} c_{\ell'} \int_{\sphere}\Big(Y'(\vq) \inner{P_{\ell'}, P_{\ell'}}_{u_d}\Big) Y(\vq) d\tau(\vq)\\
&= \sum_{\ell'=0}^{\infty} c_{\ell'} \|P_{\ell'}\|_{u_d}^2 \inner{Y, Y'}_{\tau}\\
&= \frac{c_{\ell}}{N(d,\ell)}
\end{align}
Finally, applying the formula for $c_\ell$ from \cref{lem:target_T} and the formula for $\|\Tcal(Y \otimes Y')\|_\tau$ from \cref{lem:ortho_basis},
\begin{align}
\inner{\tilde f, \frac{\Tcal(Y \otimes Y)}{\|\Tcal(Y \otimes Y)\|_{\bar\tau}}}_{\bar\tau}
= \frac{c_\ell}{N(d,\ell)} \cdot \frac{1}{\|\Tcal(Y \otimes Y)\|_{\bar\tau}}
= \frac{\frac2{\pi} \eta_\ell \alpha_\ell}{N(d,\ell)} \cdot \frac1{\sqrt{\frac4{\pi^2}\alpha_{\ell(i)}^2 /N(d,\ell)}}
= \frac{\eta_\ell}{\sqrt{N(d,\ell)}}
\end{align}
\end{proof}

Up to now, we have constructed a basis without showing that its span includes our target function.
\cref{lem:target_in_image} (in \cref{sec:random_features_proof}) verifies that, in fact, $\tilde f$ lies in this span.
This lemma is not needed for the proof of \cref{thm:invariant_lower}, but is used in the kernel approximation of \cref{sec:random_features_proof}.
It also shows that this step of the proof is tight.
We do not lose anything by lower bounding the error only on the part of $\tilde f$ that lies in the span of our basis functions.

\subsection{Expansion of the Head Functions}
In this section, we expand the low-rank attention head function in our basis $\{\Tcal(Y \otimes Y')\}$.
Unlike the target function, the energy of an attention head is not spread out, but concentrated on a few basis elements in each harmonic.
We first need the following lemma, which we will use to bound the number of these special basis elements.

\label{sec:head_expansion}
\begin{lemma}
\label{lem:num_marginal}
Let $\Acal_\ell$ be the span of the harmonics of degree $\ell$ on $\sphere$ that are zero after marginalizing onto the first $r$ coordinates.
Then
\begin{equation} \dim(\Fcal_\ell / \Acal_\ell) := M(r, \ell) \leq \binom{r+\ell}{\ell} \end{equation}
where $\Fcal_\ell / \Acal_\ell$ is the orthogonal complement of $\Acal_\ell$ in $\Fcal_\ell$.
Furthermore, $M(1,\ell) = 1$.
\end{lemma}
\begin{proof}
Let $\mathcal L : \mathcal F_\ell \to L^2(B_r)$ be the linear operator which marginalizes a degree $\ell$ spherical harmonic function on the first $r$ coordinates. (Here, $B_r$ is the unit $r$-ball.)
That is,
\begin{equation} (\mathcal Lf)(\vx) := \E_{\vy \sim \mathbb S^{d-r-1}} f\left(\begin{bmatrix}\vx \\ \vy \sqrt{1-\|\vx\|^2}\end{bmatrix}\right) \end{equation}
By definition, $\Acal_\ell$ is the null space of $\mathcal L$.
We will show below that the range of $\mathcal L$ contains only polynomials of the first $r$ coordinates of degree at most $\ell$.
The dimension of the space of polynomials in dimension $r$ of degree at most $\ell$ is $\binom{r+\ell}{\ell}$. Thus, by the rank-nullity theorem,
\begin{align}
\dim(\Fcal_\ell) \leq \dim(\Acal_\ell) + \binom{r+\ell}{\ell}
\end{align}
and therefore
\begin{align}
\dim(\Fcal_\ell / \Acal_\ell) = \dim(\Fcal_\ell) - \dim(\Acal_\ell) \leq \binom{r+\ell}{\ell}
\end{align}

We will now show that the range of $\mathcal L$ contains only polynomials in the first $r$ coordinates of degree at most $\ell$.
Each spherical harmonic is the restriction to $\sphere$ of a harmonic homogeneous polynomial on $\R^d$, so it suffices to show that $\mathcal L$ maps monomials of degree exactly $\ell$ in $\R^d$ to polynomials of degree at most $\ell$ in the first $r$ coordinates.
Let
\begin{equation}Y\left(\begin{bmatrix}\vx \\ \vy\end{bmatrix}\right) := x_1^{p_1}\cdots x_r^{p_r} y_{r+1}^{p_{r+1}}\cdots y_d^{p_d} = \left(\prod_{i=1}^r x_i^{p_i}\right)\left(\prod_{i=r+1}^d y_i^{p_i}\right) \end{equation}
be one such monomial.
If any of $p_{r+1}, \ldots, p_d$ is odd, then $L[Y] = 0$.
If all are even, then
\begin{align}
L[Y](\vx)
&= \left(\prod_{i=1}^r x_i^{p_i}\right)\left(\E_{\vy \sim \mathbb S^{d-r-1}}\prod_{i=r+1}^d \left(y_i \sqrt{1-\|\vx\|^2}\right)^{p_i}\right)\\
&= \left(\prod_{i=1}^r x_i^{p_i}\right)\left(\prod_{i=r+1}^d \left(1-\|\vx\|^2\right)^{p_i/2}\right)\left(\E_{\vy \sim \mathbb S^{d-r-1}}\prod_{i=r+1}^d y_i^{p_i}\right)
\end{align}
is a polynomial in $\vx$ whose highest degree term has degree $\left(\sum_{i=1}^r p_i\right) + \left(\sum_{i=r+1}^d p_i\right)$, which equals the degree of the original monomial.
    
For the special case of $r=1$, it suffices to show that $\mathcal L$ has rank one, or equivalently that its nullspace has dimension $N(d,\ell) - 1$.
Let $Y_1 = P_{\ell}(\inner{\hat \ve_1, \cdot})$, where $\hat{\ve}_1 \in \R^d$ is the first standard basis vector.
By Theorem 4.10 of \cite{frye2012spherical}, $Y_1$ is a spherical harmonic of degree $\ell$.
Complete an orthonormal basis $\{Y_1, \ldots Y_{N(d,\ell)}\}$ of $\Fcal_{\ell}$.
Our goal is to show that $\mathcal L Y_j = 0$ for all $j \in \{2, \ldots N(d,\ell)\}$ (with equality in the weak sense).

To do this, it suffices to show that $\inner{P_\ell, \mathcal L Y_j} = 0$ for all $\ell$:
\begin{align}
\inner{P_\ell, \mathcal L Y_j}
&= \E_{x \sim u_d}\left[P_\ell(x)(\mathcal L Y_j)(x)\right]\\
&= \E_{x \sim u_d}\left[P_\ell(x) \E_{\vy \in \mathbb S^{d-2}} Y_j\left(\begin{bmatrix}x \\ \vy \sqrt{1 - |x|^2}\end{bmatrix}\right)\right]\\
&= \E_{\vz \sim \tau} \left[P_\ell(x) Y_j\left(\vz\right)\right]
\end{align}
where $\vz := \begin{bmatrix}x \\ \vy \sqrt{1 - |x|^2}\end{bmatrix} \in \sphere$.
But by definition, $P_\ell(x) = Y_1\left(\begin{bmatrix}x \\ \vy \sqrt{1 - |x|^2}\end{bmatrix}\right)$ for all $\vy \in \mathbb S^{d-2}$. Continuing from above,
\begin{equation}
= \E_{\vz \sim \tau} \left[Y_1(\vz) Y_j(\vz)\right]
= \inner{Y_1, Y_j}_\tau = 0
\end{equation}
for all $j \neq 1$.
\end{proof}

\begin{lemma}
\label{lem:ortho_kills}
Let $\mX$ be a square matrix.
Let $\mathcal D$ be the uniform distribution over orthogonal matrices.
Then,
\begin{equation} \E_{\mQ \sim \mathcal D} [\mQ^\top \mX \mQ] = \mathrm{tr}(\mX) \cdot \mI \end{equation}
\end{lemma}
\begin{proof}
Let $q_{ki}$ denote the entry in the $k$th row and $i$th column of $\mQ$.
Then the $(i,j)$ entry of the expectation is
\begin{equation} \E_{\mQ \sim \mathcal D} [\mQ^\top \mX \mQ]_{ij} = \sum_k \sum_{\ell} x_{k\ell} \E_Q[q_{ki}q_{\ell j}] \end{equation}
So long as $(k,i) \neq (\ell,j)$, then conditional distribution of $q_{\ell j}$ given $q_{ki}$ is symmetric, since negating the $\ell$th row (or $j$th column) of $\mQ$ would produce another orthonormal matrix.
Thus, if $(k,i) \neq (\ell, j)$, then the expectation is zero.
The only non-zero terms are
\begin{equation} \E_{\mQ \sim \mathcal D} [\mQ^\top \mX \mQ]_{ii} = \sum_k x_{kk} \E_Q[q_{ki}^2] \end{equation}
Since the marginal distribution of each row (or column) is uniform on the unit sphere, the variance of each entry is 1.
\end{proof}

\begin{lemma}
\label{lem:head_ortho}
Define $M(r,\ell)$ as in \cref{lem:num_marginal}.
Assume the rank $r < d$ and consider the functions $g_h(z) = \vx^\top \mV_h \vx \cdot \tilde \phi_h(\mK_h^\top \vx, \vy)$ for $\tilde \phi_h : \R^r \times \sphere \to \R$ and $\mK_h \in \R^{d \times r}$ for $h \in [H]$.
Then there exists a subspace $\Acal_\ell \subseteq \Fcal_\ell$ of dimension at least $N(d,\ell) - H\cdot M(r,\ell)$ such that $\Tcal (Y_\ell \otimes Y_\ell)$ is orthogonal to $g_h$ for any $Y_\ell \in \Acal_\ell$ and any $h \in H$.
\end{lemma}
\begin{proof}
The first part of the proof gives a construction for $\Acal_\ell$.
Fix $\vy$, $\vq$ and $h$ and define
\begin{equation} h_{\mK}(\vk) := \E_{\vx \sim \tau} \left[\rho( \vz, \omega) g_h(\vx) \right] = \E_{\vx \sim \tau} \left[\sign(\vx^\top \vk \vq^\top \vy) \vx^\top \mV \vx \cdot \tilde \phi_h(\mK^\top \vx, \vy) \right] \end{equation}
Define $\overline \mK = \begin{bmatrix}\mK & \vk\end{bmatrix}$.
As a first step, we show that this function only depends on a particular projection of $\mV$, not on $\mV$ itself.
Choose a basis such that the column span of $\overline \mK$ is $\Span(\{\ve_1, \ldots, \ve_{r'}\}$, where $1 \leq r' \leq \min(r+1, d)$.
Then we can rewrite $\mV = \begin{bmatrix}\mA & \mB \\ \mC & \mD\end{bmatrix}$ where $\mA \in \R^{r' \times r'}$.
The distribution of $\vx$ is isotropic and independent of $\vy$.
Therefore, we can rotate it without affecting the expectation.
In fact, we can draw a random orthogonal matrix from any distribution, and $\E_{\vx, \mQ}[f(\mQ\vx)]$ will equal $\E_{\vx}[f(\vx)]$.
We draw random orthogonal matrices that fix the column span of $\overline \mK$, that is, matrices of the form $\mQ = \begin{bmatrix}\mI & \cdot \\ \cdot & \tilde \mQ \end{bmatrix}$, where $\tilde \mQ \in \R^{(d-r')\times(d-r')}$ is a uniformly distributed orthogonal matrix.
Then,
\begin{align}
h_{\mK}(\vk)
&= \E_{\vx, \mQ} \left[\sign(\vx^\top \mQ^\top \vk \vq^\top \vy) \vx^\top \mQ^\top \mV_h \mQ \vx \cdot \tilde \phi_h(\mK^\top \mQ \vx, \vy) \right]\\
&= \E_{\vx, \tilde \mQ} \left[\sign(\vx^\top \vk \vq^\top \vy) \vx^\top \begin{bmatrix}\mA & \mB\tilde \mQ \\ \tilde \mQ^\top \mC & \tilde \mQ^\top \mD \tilde \mQ\end{bmatrix} \vx \cdot \tilde \phi_h(\mK^\top \vx, \vy) \right]
\end{align}
Moving the expectation over $\tilde \mQ$ inside, the off-diagonal blocks are both 0.
Applying \cref{lem:ortho_kills}, the bottom right block becomes $\mathrm{tr}(\mD) \cdot \mI$.
Thus, letting $\mA' = \mA - \mathrm{tr}(\mD)\cdot \mI$,
\begin{equation} \E_{\tilde \mQ} \begin{bmatrix}\mA & \mB\tilde \mQ \\ \tilde \mQ^\top \mC & \tilde \mQ^\top \mD \tilde \mQ\end{bmatrix} = \mathrm{tr}(\mD) \cdot \mI + \mU \mA' \mU^\top \end{equation}
where $\mU = \begin{bmatrix}\mI \\ \cdot\end{bmatrix}$ is defined to be the column span of $\overline \mK$.
In all,
\begin{equation}h_{\mK}(\vk)  = \E_{\vx} \left[\sign(\vx^\top \vk \vq^\top \vy) \vx^\top \left(\mathrm{tr}(\mD) \cdot \mI + \mU \mA' \mU^\top\right) \vx \cdot \tilde \phi_h(\mK^\top \vx, \vy) \right] \end{equation}
Now that we have reduced $\mV$, we can more clearly see the implications of the rotational invariance of the distribution of $\vx$.
Let $\mO$ be an arbitrary orthonormal matrix.
Then
\begin{align}
h_{\mK}(\vk)
&= \E_{\vx \sim \tau} \left[\sign(\vx^\top \mO^\top \vk \vq^\top \vy) \vx^\top \mO^\top \left(\mathrm{tr}(\mD) \cdot \mI + \mU \mA' \mU^\top \right)\mO\vx\cdot \tilde \phi_h(\mK^\top \mO \vx, \vy) \right]\\
&= \E_{\vx \sim \tau} \left[\sign(\vx^\top \mO^\top \vk \vq^\top \vy) \vx^\top \left(\mathrm{tr}(\mD) \cdot \mI +  \mO^\top\mU \mA' \mU^\top\mO \right)\vx\cdot \tilde \phi_h(\mK^\top \mO \vx, \vy) \right]\\
&= h_{\mO^\top\mK}(\mO^\top \vk) 
\end{align}
where the last step follows because $\mO^\top \mU$ is precisely the column span of $\mO^\top \mK$.
Thus by Weyl's fundamental theorem of invariant functions,
there exists $\tilde h : \R^r \to \R$ such that 
\begin{equation} h_{\mK}(\vk) = \tilde h(\mK^\top \vk) \end{equation}

Let $\tau_{\mK}$ denote the marginal distribution of $\tau$ on the column space of $\mK$ and let $\tau_{\mK^\perp}$ denote its marginal distribution on the orthogonal complement of the column space of $\mK$.
Then the random vector $\vv + \vv^\perp\sqrt{1-\|\vv\|}$, where $\vv \sim \tau_{\mK}$ and $\vv^\perp \sim \tau_{\mK^\perp}$ is distributed uniformly on the sphere.
Let $Y$ be a spherical harmonic that is zero after marginalizing the onto the column space of $\mK$.
(For example, if $\mK^\top = \begin{bmatrix}\tilde \mK^\top & \vzero_{r \times d-r} \end{bmatrix}$, then marginalizing onto the column space means taking the average of the function over the final $d-r$ coordinates.)
Then
\begin{align}
\inner{h_{\mK}, Y}
&= \int_{\sphere} h_{\mK}(\vk) Y(\vk) d\tau(\vk)\\
&= \int\int h_{\mK}(\vv + \vv^\perp\sqrt{1-\|\vv\|}) Y(\vv + \vv^\perp\sqrt{1-\|\vv\|}) d\tau_{\mK^\perp}(\vv^\perp)d\tau_{\mK}(\vv)\\
&=\int\tilde h_{\mK}(\vv) \left(\int Y(\vv + \vv^\perp\sqrt{1-\|\vv\|}) d\tau_{\mK^\perp}(\vv^\perp)\right)d\tau_{\mK}(\vv)\\
&= 0
\end{align}
Let $\Acal_\ell^h \subset \Fcal_\ell$ be the space of spherical harmonics of degree $\ell$ that have this marginalization property with respect to $\mK_h$.
Let $\Acal_\ell = \cap_h \Acal_\ell^h$.
Recall that $N(d,\ell)$ is the dimension of $\Fcal_\ell$, and $M(r,\ell)$ is the dimension of the orthogonal complement of $\Acal_\ell^h$ in $\Fcal_\ell$, denoted $\Fcal_\ell / A_\ell^h$.
Thus,
\begin{equation} \dim(\Acal_\ell) = \dim(\Fcal_\ell) - \dim(\Fcal_\ell / \Acal_\ell) = N(d,l) - \dim(\oplus_h (\Fcal_\ell / \Acal_\ell^h)) \geq N(d,\ell) - H\cdot M(r,\ell) \end{equation}
It remains to show that for all $Y \in \Acal_\ell$, $\Tcal (Y_\ell \otimes Y_\ell)$ is orthogonal to $g_h$. 
\begin{align}
\inner{\Tcal(Y \otimes Y), g_h}_{\bar\tau}
&= \int_\Omega \E_{\vz}[\rho(\vz, \omega)g_h(\vz)] Y(\vk)Y(\vq)d\tau(\vk)d\tau(\vq)\\
&= \int_{\sphere}\E_{\vy}\left(\int_{\sphere}\E_{\vx}[\rho(\vx,\vy, \omega)g_h(z)] Y(\vk)d\tau(\vk)\right)Y(\vq)\tau(\vq)
\end{align}
But for any fixed $\vy$ and $\vq$,
\begin{align}
\int_{\sphere}\E_{\vx}[\rho(\vx,\vy, \omega)g_h(z)] Y(\vk)d\tau(\vk)
= \inner{h_{\mK_h}, Y} 
= 0
\end{align}
by the calculation above, where the final step follows because $Y \in \Acal_\ell \subset \Acal_\ell^h$.
\end{proof}

\subsection{Proof of \texorpdfstring{\cref{thm:invariant_lower}}{Theorem 2}}
\label{sec:main_proof_main_step}
\invariantLower*
\begin{proof}
We lower bound the error by projecting it onto the unit vector $(\vx_1 - \vx_2)/(\sqrt 2)$.
For convenience, we define a basis
\begin{equation} \vx = \frac{\vx_1 - \vx_2}{\sqrt 2} \qquad \vw = \frac{\vx_1 + \vx_2}{\sqrt 2} \end{equation}
The joint distribution of $\vx$ and $\vw$ is the same as that of $\vx_1$ and $\vx_2$.
They are each uniformly distributed on the sphere, and they are always orthogonal.
The projection of the target function onto $\vx$ yields the surrogate target function of \cref{def:surrogate_target}:
\begin{equation} \inner{\frac{\vx_1 - \vx_2}{\sqrt 2}, f(\mX; y)} =\frac1{\sqrt 2}\sign\left(\inner{\vx_1 - \vx_2, \vy}\right) =: \frac1{\sqrt 2}\tilde f(\vx, \vy) \end{equation}
Let the attention weights produced by a softmax head be $t_1$ and $t_2 = 1-t_1$.
Then the output of the head before multiplication with $\mV$ is
\begin{equation} t\vx_1 + (1-t)\vx_2 = \frac{t_1 - t_2}{\sqrt 2} \vx + \frac1{\sqrt 2}\vw \end{equation}
Letting $\tilde \phi(\mK^\top \vx, \vy) = (t_1 - t_2)/\sqrt{2}$, the inner product of the head with $\vx$ is
\begin{equation} \vx^\top \mV \vx \cdot \tilde \phi(\mK^\top \vx, \vy) + \vx^\top \mV \vw \end{equation}
Notice that, since the conditional distribution of $\vw$ given $\vx$ is symmetric, the correlation of the second term above with the surrogate target is zero:
\begin{equation} \E_{\vx_1, \vx_2 \sim \mathcal{D}_2(\sphere)}\left[\tilde f(\vx, \vy) \cdot \vx^\top \mV \vw\right] = 0 \end{equation}
Thus, we have the following lower bound:
\begin{align}
&\E_{\substack{\vx_1, \vx_2 \sim \mathcal{D}_2(\sphere)\\ \vy \sim \unif(\sphere)}} \left\|f(\mX; \vy) - \sum_{h=1}^H \mV_h \mX \phi_h\left(\mK_h^\top (\vx_1 - \vx_2), \vy\right)\right\|^2\\
&\geq\E_{\substack{\vx_1, \vx_2 \sim \mathcal{D}_2(\sphere)\\ \vy \sim \unif(\sphere)}} \inner{\vx, \quad f(\mX; \vy) - \sum_{h=1}^H \mV_h \mX \phi_h\left(\mK_h^\top (\vx_1 - \vx_2), \vy\right)}^2\\
&= \E_{\substack{\vx, \vw \sim \mathcal{D}_2(\sphere)\\ \vy \sim \unif(\sphere)}} \frac12\left(\tilde f(\vx, \vy) - \sum_{h=1}^H \vx^\top \mV_h \vx \cdot \tilde\phi_h(\mK_h^\top \vx, \vy) ) - \sum_{h=1}^H \vx^\top \mV_h \vw \right)^2\\
&\geq \E_{\vx, \vy \sim \unif(\sphere)} \frac12\left(\tilde f(\vx, \vy) - \sum_{h=1}^H \vx^\top \mV_h \vx \cdot \tilde\phi_h(\mK_h^\top \vx, \vy) )\right)^2\\
\
&= \frac12 \left\|\tilde f - \sum_{h=1}^H g_h\right\|_{\bar\tau}^2
\end{align}

where $g_h(z) = \vx^\top \mV_h \vx \cdot \tilde \phi_h(\mK_h^\top \vx, \vy)$.
Construct the space $\Acal_\ell \subseteq \Fcal_\ell$ according to \cref{lem:head_ortho}, and let $\{Y_\ell^i\}_{i=1}^{\dim \Acal_\ell}$ be an orthonormal basis of $\Acal_\ell$.
Then each element in the following set is orthogonal to each $g_h(z)$:
\begin{equation} \left\{\frac{\Tcal(Y_\ell^i \otimes Y_\ell^i)}{\|\Tcal(Y_\ell^i \otimes Y_\ell^i)\|_{\bar\tau}}\right\}_{i=1}^{\dim (\Acal_\ell)} \end{equation}
Furthermore, by \cref{lem:ortho_basis}, this set is orthonormal.
Thus
\begin{align}
\left\|\tilde f - \sum_{h=1}^H g_h\right\|_{\bar\tau}^2
&\geq \sum_{\ell \text{ odd}} \sum_{i=1}^{\dim(\Acal_\ell)} \inner{\tilde f - \sum_{h=1}^H g_h, \frac{\Tcal(Y_\ell^i \otimes Y_\ell^i)}{\|\Tcal(Y_\ell^i \otimes Y_\ell^i)\|_{\bar\tau}}}^2\\
&= \sum_{\ell \text{ odd}} \sum_{i=1}^{\dim(\Acal_\ell)} \inner{\tilde f, \frac{\Tcal(Y_\ell^i \otimes Y_\ell^i)}{\|\Tcal(Y_\ell^i \otimes Y_\ell^i)\|_{\bar\tau}}}^2\\
&= \sum_{\ell \text{ odd}} \dim(\Acal_\ell) \frac{\eta_\ell^2}{N(d,\ell)}
\end{align}
where the final step follows from \cref{lem:target_expansion}.
By the construction of $\Acal_\ell$ (\cref{lem:head_ortho}),
\begin{equation} \dim(\Acal_\ell) \geq N(d,\ell) - H \cdot M(r, \ell) \end{equation}
and thus
\begin{equation}\label{eq:last_step}
\geq \sum_{\ell \text{ odd}} \left(1 - H \cdot \frac{M(r,\ell)}{N(d,\ell)}\right) \eta_\ell^2 \end{equation}
Appealing either to \cref{lem:tiny_eps} or to \cref{lem:big_eps} finishes the proof.
\end{proof}

\subsection{Asymptotics}
\label{sec:asymptotics}
\begin{lemma}
\label{lem:basic_binomial_integral}
    Let $m>\ell$ and $\ell$ odd. Then 
    \begin{align}
        \int_0^1 \left( \frac{d}{dt}\right)^\ell (1-t^2)^m dt &= (-1)^{1+(\ell-1)/2} \binom{m}{\frac{\ell-1}{2}} (\ell-1)!~.
    \end{align}
\end{lemma}
\begin{proof}
    We have 
    \begin{align}
        \int_0^1 \left( \frac{d}{dt}\right)^\ell (1-t^2)^m dt & = - \left( \frac{d}{dt}\right)^{\ell-1} (1-t^2)^m \Big\rvert_{t=0} \\
        &= - \left( \frac{d}{dt}\right)^{\ell-1} \sum_{k=0}^m \binom{m}{k} (-1)^k t^{2k} \Big\rvert_{t=0} \\
        &= (-1)^{1+(\ell-1)/2} \binom{m}{\frac{\ell-1}{2}} (\ell-1)!~.
    \end{align}
\end{proof}

\begin{lemma}
\label{lem:eta_decay}
Define $\eta_\ell$ as in \cref{def:geg_coeffs}.
For odd $\ell$, $\eta_\ell^2 \sim \sqrt{\frac{d}{\ell^3(\ell+d)}}$.
\end{lemma}
\begin{proof}
From the definition, we have 
\begin{align}
    \eta_{l,d} &= 2\frac{\sqrt{N(d,l)}A_{d-2}}{A_{d-1}} \int_0^1  P_{l,d}(t) (1-t^2)^{(d-3)/2} dt~.
\end{align}
From the Rodrigues formula for $P_{l,d}$ \cite[Proposition 4.19]{frye2012spherical}, we have 
\begin{align}
    \eta_{l,d} &= 2\frac{\sqrt{N(d,l)}A_{d-2}}{A_{d-1}} \frac{(-1)^l}{2^l(l+(d-3)/2)_l} \int_0^1 \left( \frac{d}{dt}\right)^l (1-t^2)^{l+(d-3)/2} dt~.
\end{align}
Now, using Lemma \ref{lem:basic_binomial_integral}, we obtain
\begin{align}
    \eta_{l,d} &= 2\frac{\sqrt{N(d,l)}A_{d-2}}{A_{d-1}} \frac{(-1)^l}{2^l(l+(d-3)/2)_l} (-1)^{1+(l-1)/2} \binom{l+(d-3)/2}{\frac{l-1}{2}} (l-1)!~,
\end{align}
and thus, using $\frac{A_{d-2}}{A_{d-1}} \sim C' \sqrt{d}$, we have  
\begin{align}
    |\eta_{l,d} | & \sim C \sqrt{d} \sqrt{N(d,l)} 2^{-l}\frac{(l-1)!((d-3)/2)!}{(l+(d-3)/2)!} \binom{l+(d-3)/2}{\frac{l-1}{2}}~. \\
    &= C \frac{\sqrt{d}}{l} \sqrt{N(d,l)} 2^{-l}\frac{\binom{l+(d-3)/2}{\frac{l-1}{2}}}{\binom{l+(d-3)/2}{l}} \\
    &= C \frac{\sqrt{d}}{l} \sqrt{N(d,l)} 2^{-l} \frac{l! \left(\frac{d-3}{2}\right)!}{\left(\frac{l-1}{2}\right)! \left(\frac{d+l-2}{2}\right)!}~.
\end{align}

Using Stirling's approximation, we obtain

\begin{align}
    N(d,l) &\sim \frac{l+d}{l} \left(\frac{l+d}{l d} \right)^{1/2} \frac{(l+d)^{(l+d-3)}}{l^{(l-1)} d^{(d-2)}} \\ 
    &\sim (l+d)^{l+d-3/2} l^{-l-1/2} d^{-d+3/2}~,
\end{align}
as well as 
\begin{align}
    \frac{l! \left(\frac{d-3}{2}\right)!}{\left(\frac{l-1}{2}\right)! \left(\frac{d+l-2}{2}\right)!} &\sim \sqrt{\frac{l d}{l(d+l)}} l^{(l+1)/2} d^{(d-3)/2} (d+l)^{(-d-l+2)/2} 2^l~,\\
    &\sim (l+d)^{(-d-l+1)/2} l^{(l+1)/2} d^{(d-2)/2}2^l ~,
\end{align}
leading to 
\begin{align}
    |\eta_{l,d}| &\sim (l+d)^{(-d-l+1+l+d)/2 -3/4} l^{-1-l/2-1/4 +l/2 + 1/2} d^{1/2-d/2+3/4+d/2-1} \\
    &\sim (l+d)^{-1/4} l^{-3/4} d^{1/4}~,
\end{align}
as claimed. 
\end{proof}

\cref{lem:eta_decay} shows that $\eta_\ell^2$ decays slowly with $\ell$.
Using $\sqrt{\frac{d}{\ell^3 (\ell+d)}} \geq 1/\ell^2$ and including by a fudge factor $c$ that is slightly smaller than $1$, we get a form that is better suited to the proof of our lower bounds:
\begin{corollary}\label{lem:eta_nonasymptotic_decay}
    There exists a universal constant $c''$ such that $\eta_\ell^2 \geq c''/\ell^2$ for all sufficiently large $d$ and $\ell$ (say, for all $d,\ell > 4$).
\end{corollary}

\begin{lemma}[Decay of $\alpha_\ell$]
\label{lem:alpha_decay}
For $\ell$ odd, we have $\alpha_{\ell} = \frac14 \eta_\ell^2 / \sqrt{N(d,\ell)}$.
\end{lemma}
\begin{proof}

We start from $\arcsin = \pi/2 - \arccos$ and the kernel representation \cite[Section 3.1]{bach2017breaking}
\begin{align}
\frac{1}{2\pi}( \pi - \arccos( x \cdot y)) &= \mathbb{E}_{\theta \in \mathbb{S}^{d-1}}\left[ \mathbf{1}[x \cdot \theta >0] \mathbf{1}[y \cdot \theta >0] \right] ~.   
\end{align}
Now, from the Hecke-Funk formula, we have, up to zeroth-harmonic terms, the following correspondence between the Gegenbauer expansion of $\arcsin$ and that of of $\sign$, given precisely by $\eta_\ell$. Fix any $\vx \in \mathbb{S}^{d-1}$. Then
\begin{align}
    P_\ell(1) \langle \arcsin, P_\ell \rangle &= \int \arcsin( \vx \cdot \vy) P_\ell( \vx \cdot \vy) \tau(d\vy) \nonumber \\
  &= \frac14\int \int \sign( \vx \cdot \theta) \sign( \vy \cdot \theta) P_\ell( \vx \cdot \vy) \tau(d\vy) \tau(d\theta) \nonumber \\
  &=  \frac14\langle \sign, P_\ell \rangle \int \sign( \vx \cdot \theta) P_\ell(\vx \cdot \theta) \tau(d\theta) \nonumber \\
  &= \frac14 P_\ell(1) \langle \sign, P_\ell \rangle^2 ~.
\end{align}
Since $\inner{\arcsin, P_\ell} = \alpha_\ell \|P_\ell\|$
and $\inner{\sign, P_\ell} = \eta_\ell \|P_\ell\|$,
so $\alpha_\ell = \frac14 \eta_\ell^2 \|P_\ell\| = \frac14 \eta_\ell^2 / \sqrt{N(d,\ell)}$.

\end{proof}

\begin{lemma}\label{lem:tiny_eps}
There are universal constants $c$ and $C$ such that the following hold:
Assume $r \leq d-3$, $\epsilon \leq \frac{c}{d+1}$, and $H \leq C \cdot 2^{d - (r+1)\log_2\left(2d/r\right)}$.
Then
\begin{equation} \sum_{\ell \text{ odd}} \left(1 - H \cdot \frac{M(r,\ell)}{N(d,\ell)}\right) \eta_\ell^2 \geq \epsilon \end{equation}
\end{lemma}
\begin{proof}
\begin{equation} N(d,\ell)
= \frac{2\ell+d-2}{\ell}\binom{\ell+d-3}{\ell - 1} \end{equation}
Applying Stirling's approximation,
\begin{align}
N(d,\ell) &\gtrsim \frac{\ell+d-3}\ell \frac{(\ell+d-3)^{\ell+d-2.5}}{(\ell-1)^{\ell-0.5}(d-2)^{d-1.5}}\\
&\geq \frac{(\ell+d-3)^{\ell+d-1.5}}{\ell^{\ell+0.5}(d-2)^{d-1.5}}
\end{align}
Meanwhile, \cref{lem:num_marginal} and Stirling's approximation give
\begin{align}
M(r,\ell)
\leq \binom{r+\ell}{\ell}
\lesssim \frac{(r+\ell)^{r+\ell+0.5}}{r^{r+0.5}\ell^{\ell+0.5}}
\end{align}
By assumption, $r \leq d-3$, so
\begin{align}
\frac{M(r,\ell)}{N(d,\ell)}
&\lesssim \left(\frac{r+\ell}{\ell+d-3}\right)^{r+\ell+0.5}
\frac{(d-2)^{d-1.5}}{r^{r+0.5}(\ell+d-3)^{d-r-2}}\\
&\leq \frac{(d-2)^{d-1.5}}{r^{r+0.5}(\ell+d-3)^{d-r-2}}
\end{align}
The above expression is decreasing in $\ell$.
Thus for all $\ell \geq \mu d + 1$,
\begin{align}
\frac{M(r,\ell)}{N(d,\ell)}
&\lesssim \frac{(d-2)^{d-1.5}}{r^{r+0.5}((1+\mu)(d-2))^{d-r-2}}\\
&\leq \left(\frac{d}r\right)^{r+0.5} \frac1{(1+\mu)^{d-r-2}}\\
&= (1+\mu)^{-d+r+2+(r+0.5)\log_{1+\mu}(d/r)}
\end{align}
By assumption, $c/\epsilon \geq d + 1$, so the above holds with $\mu = 1$ for all $\ell \geq c/\epsilon$:
\begin{equation} \frac{M(r,\ell)}{N(d,\ell)} \lesssim 2^{-d + (r+1)\log_2\left(2d/r\right)} \end{equation}
Also by assumption, $H \leq C \cdot 2^{d - (r+1)\log_2\left(2d/r\right)}$.
Setting $C$ appropriately, $\left(1 - H \cdot \frac{M(r,\ell)}{N(d,\ell)}\right) \geq \frac12$ for all $\ell \geq c/\epsilon$
Finally, applying \cref{lem:eta_nonasymptotic_decay},
\begin{align}
\sum_{\ell \text{ odd}} \left(1 - H \cdot \frac{M(r,\ell)}{N(d,\ell)}\right) \eta_\ell^2
& \geq \sum_{\substack{\ell \geq c/\epsilon \\ \ell \text{ odd}}} \frac12 \cdot \frac{c''}{\ell^2} \\
& \geq \frac{c''}4 \sum_{\ell \geq c/\epsilon} \frac1{\ell^2} \\
& \geq \frac{c''}4 \cdot \frac{\epsilon}c
\end{align}
Setting $c = c''/4$ completes the proof.
\end{proof}

\begin{lemma}
\label{lem:big_eps}
There is a universal constant $c$ such that the following holds.
If $d \geq 5$, 
\begin{equation} \frac{2c}\epsilon < \frac{d}{2e^2} - r~, \end{equation}
and 
\begin{equation} H \leq \frac12 \left(\frac1{2e} \cdot \frac{d}{r+\frac{c}{\epsilon}}\right)^{\frac{c}{\epsilon}}~, \end{equation}
then
\begin{equation} \sum_{\ell \text{ odd}} \left(1 - H \cdot \frac{M(r,\ell)}{N(d,\ell)}\right) \eta_\ell^2 \geq \epsilon \end{equation}
\begin{equation}  \end{equation}
\end{lemma}
\begin{proof}
Recall the formula for $N(d,\ell)$ from \cref{eq:ndl}.
Lower bounding, for $\ell \geq 1$ and $d \geq 5$,
\begin{align}
N(d,\ell)
&= \frac{2\ell+d-2}{\ell}\binom{\ell+d-3}{\ell - 1}\\
&\geq \frac{\ell+d-3}\ell \left(\frac{\ell + d - 3}{\ell - 1}\right)^{\ell-1}
\geq \left(\frac{\ell + d - 3}{\ell}\right)^{\ell}\\
&\geq \left(\frac{d+\ell}{2\ell}\right)^\ell\\
\end{align}

Meanwhile, \cref{lem:num_marginal} gives
\begin{align}
M(r,\ell)
\leq \binom{r+\ell}{\ell}
\leq \left(\frac{e(r+\ell)}{\ell}\right)^\ell
\end{align}
Thus
\begin{equation} \frac{M(r,\ell)}{N(d,\ell)}
\leq \left(2e \cdot \frac{r+\ell}{d+\ell}\right)^\ell \leq \left(2e \cdot \frac{r+\ell}{d}\right)^\ell\end{equation}
The above is a decreasing function of $\ell$ for all $\ell < \frac{d}{2e^2} - r$.
Assume that $\frac{2c}{\epsilon} < \frac{d}{2e^2} - r$.
Then the following holds for all $\ell \in \left[\frac{c}{\epsilon}, \frac{2c}{\epsilon}\right]$:
\begin{equation} \frac{M(r,\ell)}{N(d,\ell)} \leq \left(2e \cdot \frac{r+\frac{c}{\epsilon}}{d}\right)^{\frac{c}{\epsilon}} \end{equation}
Assume $H \leq \frac12 \left(\frac1{2e} \cdot \frac{d}{r+\frac{c}{\epsilon}}\right)^{\frac{c}{\epsilon}}$.
Then for all $\ell \in \left[\frac{c}{\epsilon}, \frac{2c}{\epsilon}\right]$:
\begin{equation}
1 - H \cdot\frac{M(r,\ell)}{N(d,\ell)} \geq \frac12 \end{equation}
Finally, applying \cref{lem:eta_nonasymptotic_decay},
\begin{equation}
\sum_{\ell \text{ odd}} \left(1 - H \cdot \frac{M(r,\ell)}{N(d,\ell)}\right) \eta_\ell^2
\geq \frac12 \sum_{\ell \text{ odd}} \frac{c''}{\ell^2}
\geq \frac{c''}4\sum_{\ell=c/\epsilon}^{2c/\epsilon} \frac1{\ell^2} 
\geq \frac{c''}4 \cdot \frac{\epsilon}{2c}
\end{equation}
Setting $c = c''/8$ completes the proof.
\end{proof}

\subsection{Kernel Ridge Regression and Random Feature Approximation}
\label{sec:random_features_proof}
In this section, we analyze a simple approximation of the nearest neighbor function by standard rank-1 attention heads.
We show that $O(\epsilon^{-4} d^{2/\epsilon})$ heads suffice to achieve a squared approximation error of $\epsilon$, nearly matching the lower bound of \cref{thm:invariant_lower}.
First, we reduce this problem to approximating the surrogate target function $\tilde f$ by rank-1 hardmax heads.
Then we approximate $\tilde f$ in the RKHS generated by rank-1 hardmax attention heads (that is, generated by the feature map $\Tcal$).
Finally, we appeal to standard arguments to conclude that we can approximate $\tilde f$ by a finite linear combination of random rank-1 hardmax heads.

Recall that a standard rank-1 attention layer has the form $\sum_h \vo_h \vv_h^\top \mX \mathrm{sm}\left(\mX^\top \vk_h \vq_h^\top \vy\right)$ for $\vq_h, \vk_h, \vv_h, \vo_h \in \R^d$.
For simplicity, in this section we use rank-1 heads without a value/output transform, that is $\sum_h \alpha_h \mX \mathrm{sm}\left(\mX^\top \vk_h \vq_h^\top \vy\right)$ for $\alpha \in \R$.
Any such head can be constructed out of $d$ standard rank-1 heads by setting $\vv_h = \ve_i, \vv_o = \alpha \ve_i$ for $i \in [d]$, so this simplification does not meaningfully change our result.

\begin{lemma}
For any $u \in L^1(\Omega)$, there exists a rank-1 attention layer that approximates the nearest neighbor function $f$ up to expected squared error $\frac12\|\tilde f - \Tcal u\|_{\bar\tau}^2$, where $\Tcal$ is defined as in \cref{def:Tcal} and $\tilde f$ is the surrogate target function of \cref{def:surrogate_target}.
\end{lemma}
\begin{proof}
As in the proof of \cref{thm:invariant_lower}, define
\begin{equation} \vx = \frac{\vx_1 - \vx_2}{\sqrt 2}~, \qquad \vw = \frac{\vx_1 + \vx_2}{\sqrt 2} ~.\end{equation}
We can rewrite the target function in terms of the surrogate target function as follows:
\begin{equation}
f(\vx_1, \vx_2; \vy) = \frac{\vx}{\sqrt 2} \tilde f(\vx, \vy) + \frac{\vw}{\sqrt 2}~.
\end{equation}
Likewise, we can write a rank-1 hardmax attention head as
\[ \mX \ \mathrm{hm}\left(\mX^\top \vk \vq^\top \vy\right) = \frac{\vx}{\sqrt 2} \rho(\vx, \vy; \vq, \vk) + \frac{\vw}{\sqrt 2} ~,\]
where $\rho(\vx, \vy; \vq, \vk) := \sgn(\vx^\top \vk \vq^\top \vy)$ is defined as in \cref{eq:rank1_hm}.
An ``averaging head'' is an attention head that always returns the average of the target points, regardless of the source point.
It can be implemented by a rank-1 softmax head by setting $\vq = \vk = \vzero$:
\[ \mX \sm\left(\mX^\top \vzero \vy\right) = \frac{\vx_1 + \vx_2}2 = \frac{\vw}{\sqrt 2}~. \]
We construct our approximation to $f$ by taking a linear combination of hardmax heads with coefficients given by $u$ plus a single averaging head with coefficient $1 - \int_\Omega u(\vq, \vk) d\bar\tau(\vq,\vk)$:
\begin{equation}
(\mX, \vy) \mapsto \int_\Omega u(\vq, \vk) \ \mX \ \mathrm{hm}\left(\mX^\top \vk \vq^\top \vy\right) d\bar\tau(\vq, \vk) + \left(1 - \int_\Omega u(\vq, \vk) d\bar\tau(\vq,\vk)\right) \frac{\vx_1 + \vx_2}{2}~.
\end{equation}
To analyze its error, we use the Pythagorean theorem.
Due to the averaging head, the projection of the error onto $\vw$ is zero.
What remains is the projection of the error onto $\vx$:
\begin{align}
&\E_{\substack{\vx_1, \vx_2 \sim \mathcal{D}_2(\sphere)\\ \vy \sim \unif(\sphere)}}
\left\|f(\mX; \vy) - \left[\int_\Omega u(\vq, \vk) \ \mX \ \mathrm{hm}\left(\mX^\top \vk \vq^\top \vy\right) d\bar\tau(\vq, \vk) + \left(1 - \int_\Omega u(\vq, \vk) d\bar\tau(\vq,\vk)\right) \frac{\vw}{\sqrt 2}\right]\right\|^2\\
= &\E_{\vx, \vy \sim \unif(\sphere)}
\frac12\left(\tilde \vx^\top f(\mX; \vy) - \int_\Omega u(\vq, \vk) \ \vx^\top \mX \ \mathrm{hm}\left(\mX^\top \vk \vq^\top \vy\right) d\bar\tau(\vq, \vk)\right)^2 =: \frac12 \left\|\tilde f - \Tcal u\right\|_{\bar\tau}^2\\
= &\E_{\vx, \vy \sim \unif(\sphere)}
\frac12\left(\tilde f(\vx, \vy) - \int_\Omega \rho(\vx, \vy; \vq, \vk) u(\vq, \vk) d\bar\tau(\vq, \vk) \right)^2 =: \frac12 \left\|\tilde f - \Tcal u\right\|_{\bar\tau}^2~.
\end{align}
\end{proof}

By the above lemma, our task is to find a finitely supported signed measure $u$ for which $\tilde f \approx \Tcal u$.
We next show that it is possible to exactly represent $\tilde f$ using a measure that is not finitely supported.

\begin{lemma}
\label{lem:target_in_image}
The surrogate target function $\tilde f$ lies in the span of $\{\Tcal(Y_\ell^i \otimes Y_\ell^i)\}$. Furthermore, $\tilde f= \Tcal u$ where $u : \Omega \to \R$ is defined as follows:
\begin{equation}
\label{eq:u_t}
u(\omega) = \frac{\pi}2 \sum_{\ell \text{ odd}}\frac{\eta_\ell}{\alpha_\ell} N(d,\ell) \cdot P_\ell(\vq^\top \vk)~.
\end{equation}
\end{lemma}
\begin{proof}
For each odd $\ell$, let $\{Y_\ell^i\}_{i=1}^{N(d,\ell)}$ be an orthonormal basis for $\Fcal_\ell$.
Applying \cref{lem:target_expansion}, the norm of the projection of $\tilde f$ onto the span of $\{\Tcal(Y_\ell^i \otimes Y_\ell^i)\}$ is
\begin{equation}
\sum_{i=1}^{N(d,\ell)} \inner{\tilde f, \frac{\Tcal (Y_\ell^i \otimes Y_\ell^i)}{\|\Tcal (Y_\ell^i \otimes Y_\ell^i)\|_{\bar\tau}}}_{\bar\tau}^2
= \sum_{i=1}^{N(d,\ell)} \frac{\eta_\ell^2}{N(d,\ell)} = \eta_\ell^2~.
\end{equation}
Summing across all (odd) degrees, the energy equals that of $\tilde f$ itself.
\begin{align}
\sum_{\ell=0}^\infty \eta_{2\ell+1}^2
= \|\sign\|_{\bar\tau}^2
= 1
= \|\tilde f\|_{\bar\tau}^2~.
\end{align}
Thus, the projection of $\tilde f$ onto this basis equals $\tilde f$.
In addition, this implies that $\tilde f$ is in the range of $\Tcal$:
\begin{align}
\tilde f
&= \sum_{\ell \text{ odd}} \sum_{i=1}^{N(d,\ell)} \inner{\tilde f, \frac{\Tcal (Y_\ell^i \otimes Y_\ell^i)}{\|\Tcal (Y_\ell^i \otimes Y_\ell^i)\|_{\bar\tau}}}_{\bar\tau} \frac{\Tcal (Y_\ell^i \otimes Y_\ell^i)}{\|\Tcal (Y_\ell^i \otimes Y_\ell^i)\|_{\bar\tau}}\\
&= \sum_{\ell \text{ odd}} \sum_{i=1}^{N(d,\ell)} \frac{\eta_\ell}{\sqrt{N(d,\ell)}} \cdot \frac1{\frac2{\pi}\alpha_\ell \sqrt{N(d,\ell)}}
\Tcal (Y_\ell^i \otimes Y_\ell^i)\\
&= \Tcal\left(\frac{\pi}2 \sum_{\ell \text{ odd}} \frac{\eta_\ell}{\alpha_\ell} \sum_{i=1}^{N(d,\ell)} (Y_\ell^i \otimes Y_\ell^i) \right)\\
&= \Tcal(u)~,
\end{align}
where, by the addition formula,
\begin{equation}
u(\omega) = \frac{\pi}2 \sum_{\ell \text{ odd}}\frac{\eta_\ell}{\alpha_\ell} N(d,\ell) \cdot P_\ell(\vq^\top \vk)~.
\end{equation}
\end{proof}

Thus, it is possible to exactly represent the surrogate target with an infinite number of rank-1 heads, each weighted according to $u(\cdot) d\bar\tau(\cdot)$.
See \cref{fig:u_t} for an illustration of this function.
We can think of $u(\cdot) d\bar\tau(\cdot)$ as a signed measure over rank-1 heads that depends only on $\angle(\vq, \vk)$.
Notice that the hardmax head function $\rho$ is odd in each of its arguments $\vq$ and $\vk$.
Since $u(\cdot)$ is also an odd function, we get the same results by restricting this measure to $[-\frac{\pi}2, \frac{\pi}2]$.
\cref{fig:u_t} shows that for large $d$, the (restricted) measure $u(\cdot)$ approaches a Gaussian distribution centered at angle $0$.

\begin{figure}
\centering
\includegraphics[width=0.8\textwidth]{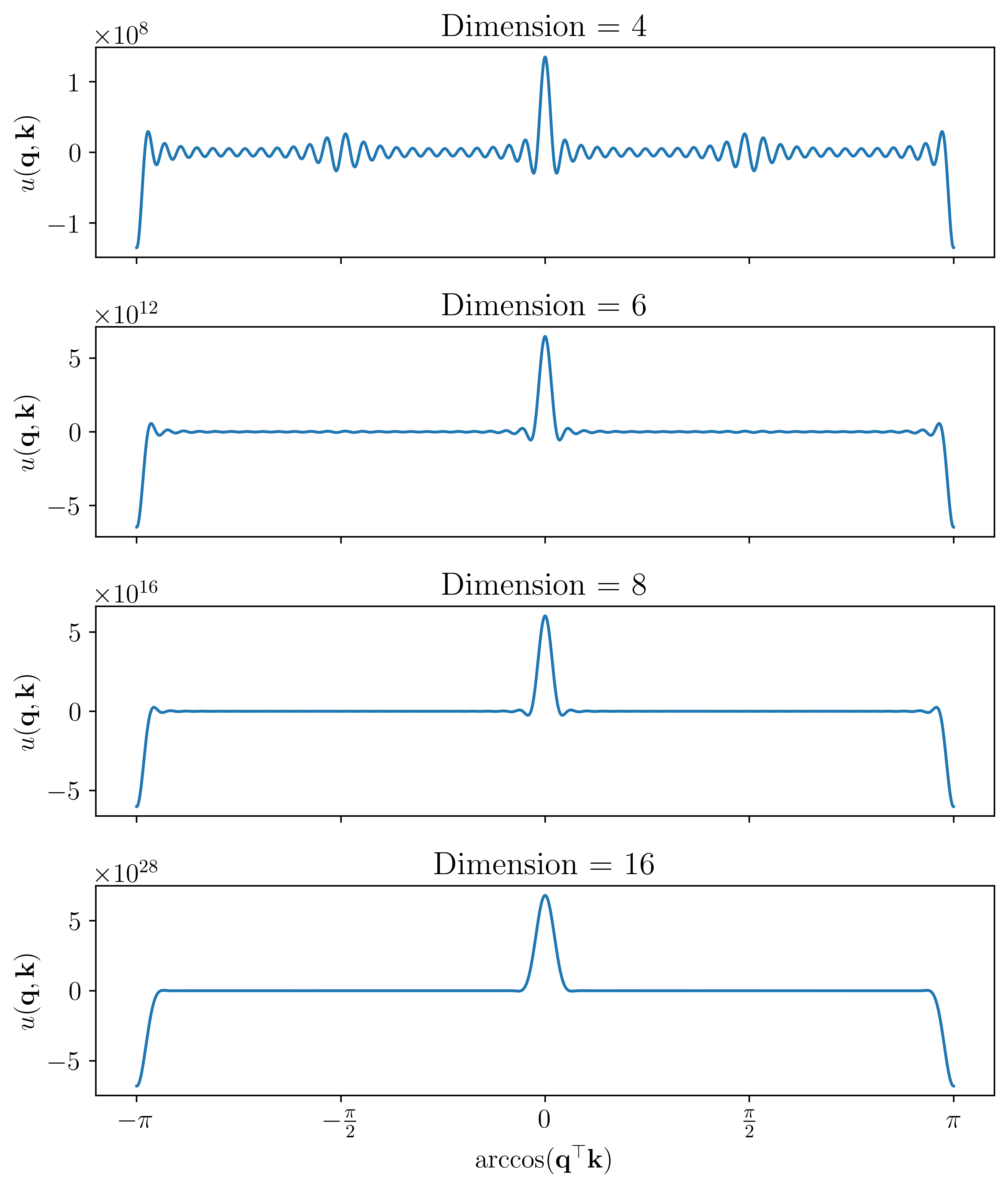}
\caption{Approximation to $u(\cdot)$ of \cref{eq:u_t} for several dimensions, using degree-50 ultraspherical expansion.
Heads with $\angle(\vq, \vk) = \theta$ are equivalent to those with angles $\theta \pm \pi$ up to a sign flip.
For large dimension, the distribution over $\angle(\vq, \vk)$ induced by $u$ approaches a Gaussian with mean $0$.
}
\label{fig:u_t}
\end{figure}

We have now shown how to represent $\tilde f$ using $\Tcal$.
This representation gives us a great deal of insight into the structure of $\tilde f$ for the following reason.
Implicit in the discussion above is the reproducing kernel Hilbert structure induced by the map $\Tcal$, as the following lemma shows:
\begin{lemma}
\label{lem:rkhs_target}
Let $\mathcal H \subseteq L^2(\Xcal)$ be the image of $\Tcal$.
Then $\mathcal H$ is a reproducing kernel Hilbert space with norm:
\begin{equation} \|f\|_{\mathcal H} = \inf \{\|u\|_{\bar\tau} : u \in \mathcal G, f = \Tcal u\}  \end{equation}
and kernel:
\begin{equation}(z, z') \mapsto \E_{\omega \sim \bar \tau}\left[\rho(z, \omega)\rho(z', \omega)\right]~.\end{equation}
\end{lemma}
The proof is given in \cite{bach2017breaking}, Appendix A.
Also note that kernel of this RKHS directly corresponds to the operator $\Tcal \Tcal^*$ by the following formula:
\begin{equation}
\left(\Tcal \Tcal^* f\right)(z) = \int_\Xcal \E_{\omega \sim \bar \tau}\left[\rho(z, \omega)\rho(z', \omega)\right] f(z') d\bar\tau(z')~.
\end{equation}

If our target function $\tilde f$ were an element of this Hilbert space, we would immediately be able to approximate it using random features.
Unfortunately, $\tilde f \not \in \mathcal H$ because
\begin{equation}
\|\tilde f\|_{\mathcal H} = \|u\|_{\bar\tau} = \sum_{\ell \text{ odd}} \left(\frac{\eta_\ell}{\alpha_\ell}\right)^2 N(d,\ell) = \infty~.
\end{equation}
However, we can approximate $f$ by an element of $\mathcal H$ obtained from solving a ridge regression problem.
For any $\lambda > 0$, let $\tilde f_\lambda$ be the solution to the following optimization problem:
\begin{equation}\label{eq:ridge}
    \min_{\hat{f} \in \mathcal{H}} \| \tilde f - \hat{f} \|_{\bar\tau}^2 + \lambda \| \hat{f} \|^2_{\mathcal{H}}~.
\end{equation}
By tuning $\lambda$, we can find an function that accurately approximates $\tilde f$ and that is smooth enough to be approximated using random features.
The following lemma constructs this $\tilde f_\lambda$.
Though we obtained this construction by solving \cref{eq:ridge}, for brevity we do not prove that it is the solution since it is not necessary for our construction.

\begin{lemma}
\label{lem:ridge_bounds}
For any regularization parameter $\lambda > 0$, there exists a function $\tilde f_{\lambda} \in \mathcal H$ for which
\begin{align}
\|\tilde f - \tilde f_\lambda\|_{\bar\tau}^2 &\leq \sum_{\ell \text{ odd}}\eta_\ell^2 \left(\frac{\lambda N(d,\ell)}{(\frac2{\pi}\alpha_\ell)^2 + \lambda N(d,\ell)} \right)^2~.
\end{align}
\end{lemma}
\begin{proof}
Define
\begin{align}
\tilde f_\lambda &:= \Tcal g_\lambda\\
g_\lambda &:= \sum_{\ell \text{ odd}} \sum_{i=1}^{N(d,\ell)} \gamma_\ell \cdot (Y_\ell^i \otimes Y_\ell^i)\\
\gamma_\ell &:= \frac{\frac2{\pi}\alpha_\ell \eta_\ell}{(\frac2{\pi}\alpha_\ell)^2 + \lambda N(d,\ell)}~.
\end{align}
Then by \cref{lem:rkhs_target}
\begin{align}
\|\tilde f_\lambda\|_{\mathcal H}^2
\leq \|g_\lambda\|_{\bar\tau}^2
&= \sum_{\ell \text{ odd}} N(d,\ell) \gamma_\ell^2\\
&\leq \sum_{\ell = 1}^{\ell_\lambda} N(d,\ell) \gamma_\ell^2
+
\sum_{\ell > \ell_\lambda} N(d,\ell) \eta_\ell^2 \left(\frac{\frac2{\pi}\alpha_\ell}{\lambda N(d,\ell)}\right)^2 \\
&\leq \sum_{\ell = 1}^{\ell_\lambda} N(d,\ell) \gamma_\ell^2 + \frac1{\lambda^2}\\
&< \infty~.
\end{align}
Thus $\tilde f \in \mathcal H$.
Furthermore, by the representation $\tilde f = \Tcal u$ of \cref{lem:target_in_image}
\begin{align}
\|\tilde f - \tilde f_\lambda\|_{\bar\tau}^2
= \|\Tcal\left(u - g_\lambda\right)\|_{\bar\tau}^2
= \left\|\sum_{\ell \text{ odd}}\sum_{i=1}^{N(d,\ell)}\left(\frac{\pi}2 \frac{\eta_\ell}{\alpha_\ell} - \gamma_\ell\right)\cdot\Tcal(Y_\ell^i \otimes Y_\ell^i)\right\|_{\bar\tau}^2~.
\end{align}
By \cref{lem:ortho_basis}, this is equal to 
\begin{align}
&= \sum_{\ell \text{ odd}}\sum_{i=1}^{N(d,\ell)}\left(\frac{\pi}2 \frac{\eta_\ell}{\alpha_\ell} - \gamma_\ell\right)^2\cdot\left\|\Tcal(Y_\ell^i \otimes Y_\ell^i)\right\|_{\bar\tau}^2\\
&=\sum_{\ell \text{ odd}} N(d,\ell) \left(\frac{\pi}2 \frac{\eta_\ell}{\alpha_\ell} - \gamma_\ell\right)^2 \frac4{\pi^2} \frac{\alpha_\ell^2}{N(d,\ell)} \\
&=\sum_{\ell \text{ odd}}\left(\eta_\ell - \frac2{\pi} \alpha_\ell \gamma_\ell \right)^2\\
&=\sum_{\ell \text{ odd}}\eta_\ell^2 \left(\frac{\lambda N(d,\ell)}{(\frac2{\pi}\alpha_\ell)^2 + \lambda N(d,\ell)} \right)^2~.
\end{align}
\end{proof}

We now derive an informal expression for the kernel ridge regression approximation using a tuned
regularization and describe its implications for random feature approximation in the high-dimensional regime.
From \Cref{lem:eta_decay} and \Cref{lem:alpha_decay}, we have $\eta_\ell^2 \lesssim \ell^{-3/2}$ and $\alpha_\ell^2 \sim \eta_\ell^4 / N(d, \ell)$.
By \Cref{lem:ridge_bounds}, for the kernel ridge regression approximation $\tilde f_\lambda$ to attain squared error $\epsilon$, we should set $\lambda$ so that $\lambda N(d, \ell^*) \simeq \alpha_{\ell^*}^2 $, where $\ell^* \sim 1/\epsilon^2$.
This roughly ensures that only degrees $\ell \gtrsim \ell^*$ are kept, while $\ell \lesssim \ell^*$ are shrunk, and hence
\begin{equation}
\|\tilde f - \tilde f_\lambda\|
\lesssim \sum_{\substack{\ell \gtrsim \ell^* \\ \ell \text{ odd}}} \eta_\ell^2
\lesssim \frac12\sum_{\ell \gtrsim \epsilon^{-2}} \ell^{-3/2}
\sim \epsilon~.
\end{equation} 
We thus obtain $\lambda \sim \alpha_{\ell^*}^2/ N(d, \ell^*) \sim \epsilon^6 N(d, \epsilon^{-2})^{-2}$.

Now that we have a sufficiently accurate kernel ridge regression approximation $\tilde f_\lambda \in \mathcal H$, we can approximate it using random features.
The key quantity controlling the number of random features needed is the \emph{degrees of freedom} of the kernel integral operator, defined as $D(\lambda) := \mathrm{tr}\left[\Tcal\Tcal^* (\Tcal\Tcal^* + \lambda \mI)^{-1}\right]$.
The eigenvalues of $\Tcal \Tcal^*$ are the same as those of $\Tcal^* \Tcal$.
By \Cref{lem:eigenfunctions}, these are 
$\left\{\frac4{\pi^2} \frac{\alpha_\ell \alpha_{\ell'}}{\sqrt{N(d,\ell)N(d,\ell')}} \mid \ell, \ell' \geq  0\right\}$, with the $(\ell, \ell')$-th eigenvalue having multiplicity $N(d,\ell)N(d,\ell') $.
Hence
\begin{equation}
D(\lambda)
= \sum_{\ell,\ell'} N(d,\ell)N(d,\ell') \cdot \frac{ \frac{\alpha_\ell \alpha_{\ell'}}{\sqrt{N(d,\ell)N(d,\ell') }}}{\frac{\alpha_\ell \alpha_{\ell'}}{\sqrt{N(d,\ell)N(d,\ell') }} + \lambda}
\leq \sum_{\ell,\ell'} N(d,\ell)N(d,\ell') \cdot \frac{ \frac{\alpha_\ell \alpha_{\ell'}}{\sqrt{N(d,\ell)N(d,\ell') }}}{\lambda}~.
\end{equation}
By \Cref{lem:alpha_decay}, $\frac{\alpha_\ell \alpha_{\ell'}}{\sqrt{N(d,\ell)N(d,\ell')}} \sim \frac{\eta_\ell^2 \eta_{\ell'}^2}{N(d,\ell)N(d, \ell')}$, so
\begin{align}
D(\lambda)
\sim \frac1{\lambda}\sum_{\ell,\ell'} \eta_\ell^2 \eta_{\ell'}^2
= \frac1{\lambda}\left(\sum_{\ell} \eta_\ell^2\right)^2
= \frac1{\lambda}
\sim \frac{N(d, \epsilon^{-2})^2}{\epsilon^6}
\lesssim \frac1{\epsilon^6} \cdot (ed\epsilon^2)^{2/\epsilon^2}
\end{align}
In the high-dimensional regime (where $\epsilon$ is fixed and $d$ goes to infinity), $D(\lambda) = \Theta\left(d^{2/\epsilon^2}\right)$.

By standard arguments about random feature expansions \cite{bach2017equivalence}, if the number of random features $H$ is of the order $H \gtrsim D(\lambda) \log (D(\lambda)) = \widetilde\Theta\left(d^{2/\epsilon^2}\right)$, then with high probability the random features achieve the same approximation accuracy $\epsilon$ as the associated kernel ridge regression solution $\tilde f_\lambda$.
It is likely that a better rate can be obtained by drawing the random features from a problem-specific distribution instead of uniformly at random.
Observe that the condition required by our lower bound in the rank-1 case has the same form, though a somewhat weaker dependence on $d$.
It is $H \leq \frac12 N(d, \frac1{4\epsilon})$ or $H = O\left(d^{\frac1{4\epsilon}}\right)$ for sufficiently large $d$.\footnote{
To see this from \cref{eq:last_step}, recall that $M(1, \ell) = 1$, follow the final steps of \cref{lem:tiny_eps}, and use the fact that we can replace $c''$ by 1 for large $d$.}

\section{Proofs from \texorpdfstring{\cref{sec:biased}}{Section 5}}
\label{appen:proofs from biased}

\subsection{Proof of \texorpdfstring{\cref{fact:biased_upper bound}}{Fact 3}}
The proof is similar to the proof of \cref{fact:invariant_upper}. The only difference is that here we consider the set $A_\delta:=\{(\vx_1,\dots,\vx_N,\vy)\in(\sd)^{N+1}: \forall i\neq j,~ |(\vx_i-\vx_j)^\top \vy +b_i - b_j| > \delta\}$~.

\subsection{Proof of \texorpdfstring{\cref{thm:biased_lower_bound}}{Theorem 4}}

In the following proofs 
when taking norms or inner products over functions, we always consider the expectation over $\Ncal\left(0,I\right)$, i.e. the distribution of $\by$. When we consider the distribution over $\vx_1$ and $\vx_2$ we explicitly take expectation.
To normalize the expectation over $\by$ we introduce the constant $c_d:= \left(\frac{1}{\sqrt{2\pi}}\right)^d$. 

We will first construct a periodic functions using a linear combination of thresholds. Let $a\in \mathbb{N}_{>2}$ and denote $H_a(x) = \mathds{1}(x + a \geq 0)$.  We define the following function:
\begin{equation}\label{eq: def of psi}
    \psi_a(x) = H_a(x) + \sum_{n=1}^{2a} H_{a-n}(x)\cdot (-1)^n - \frac{1}{2}~.
\end{equation}
This function have the following properties:

\begin{lemma}\label{lem:properties psi}
    The function $\psi_a(x)$ defined in \cref{eq: def of psi} satisfies that:
    \begin{enumerate}
        \item It is a periodic function in the interval $[-a,a]$, and odd if $a$ is an odd number.
        \item For every $\bw$ with $\norm{\bw}\geq d$, if $a > \norm{\bw}$ then $\norm{\psi_a(\inner{\bw,\cdot})^2}^2 \geq \frac{1}{40}$
    \end{enumerate}
\end{lemma}

\begin{proof}
    Let $x_0\in[-a,a-2]$. There is $n_0\in\{1,\dots,2a\}$ such that $\lceil x_0\rceil,\lceil x_0+2\rceil\in[a-n_0,a-n_0+2]$. For every $n<n_0$ or $n>n_0+2$ we have that $H_{a-n}(x_0) = H_{a-n}(x_0 + 2)$, since the bump in the threshold is either left of $x_0$ or right of $x_0+2$. We also have that $H_{a-n_0}(x_0) + H_{a-n_0 + 1}(x_0) = H_{a-n_0}(x_0 + 2) + H_{a-n_0 + 1}(x_0+ 2) = 0$. Hence $\psi_a(x_0) = \psi_a(x_0 + 2)$, which means it is a periodic function with a period of $2$. 

    If $a$ is an odd number, then for every $x_0\in [-1,0]$ we have $\psi_a(x_0) = -\frac{1}{2}$ and for every $x_0\in [0,1]$ we have $\psi_a(x_0) = \frac{1}{2}$. Since it is periodic with a period of $2$, it is odd in the interval $[-a,a]$.

    For the second item, since $\vx$ has a spherically symmetric distribution, we can assume w.l.o.g that $\bw = \norm{\bw}\be_1$. We now have that:
    \begin{align}
       \norm{\psi_a(\inner{\bw,\cdot})}^2 &= c_d\int_{\vx\in\reals^d}|\psi_a(\inner{\bw,\vx})|^2e^{-\frac{\norm{\vx}^2}{2}} d\vx \\
       &=c_d\int_{-\infty}^\infty |\psi_a(\norm{\bw}x_1)|^2e^{-\frac{x_1^2}{2}}dx_1\cdot \int_{-\infty}^\infty e^{-\frac{x_2^2}{2}}dx_2\cdots\int_{-\infty}^\infty e^{-\frac{x_d^2}{2}}dx_d \\
       & = \frac{1}{\norm{\bw}\sqrt{2\pi}}\int_{-\infty}^\infty |\psi_a(z)|^2 e^{-\frac{z^2}{2\norm{\bw}^2}} dz \\
       &\geq \frac{1}{\norm{\bw}e\sqrt{2\pi}}\int_{-\sqrt{2}\norm{\bw}}^{\sqrt{2}\norm{\bw}} |\psi_a(z)|^2 dz
    \end{align}
    where we used that if $z\leq \sqrt{2}\norm{\bw}$ then $e^{-\frac{z^2}{2\norm{\bw}^2}} \leq e^{-1}$. Since $a > \norm{\bw}$, then in the interval $\left[-\sqrt{2}\norm{\bw},\sqrt{2}\norm{\bw}\right]$ there are  at least $\lfloor\norm{\bw}\rfloor$ intervals of the form $[n,n+2]$ for $n\in\{-a,...,a-2\}$ where $\int_{n}^{n+2}|\psi_a(z)|^2 \geq \frac{1}{4}$. In total, we can bound the norm by:
    \begin{equation}
    \norm{\psi_a(\inner{\bw,\cdot})}^2 \geq \frac{1}{4e\sqrt{2\pi}} \geq \frac{1}{40}
    \end{equation}
\end{proof}

We now show that the correlation of this function with any other function that depends only on $w_1,\dots,w_r$ is small:

\begin{theorem}\label{thm:exp bound correlation}
     Let $g(w_1,\dots,w_r,\by)$ be some function that depends on the first $r$ coordinates of $\bw$ with $\sup_{\vx}|g(\vx)| \leq 1$, and take $a=2d^2 + 1$. Then, we have that:    
\begin{equation}\label{eq:inner prod psi g w1 wr}
        \E_{\bw\sim\Ucal(d\sd)}\left[ \E_{\by\sim\Ncal(0,I)}\left[\left|\psi_a(\inner{\bw,\by})\cdot g(w_1,\dots,w_r,\by)\right|\right]\right]\leq \exp(-c(d-r))
\end{equation}
for some universal constant $c>0$.
\end{theorem}

\begin{proof}

For a vector $\bv$ denote by $\bar{\bv}$ its last $d-r$ coordinates.
Using the law of total expectation, we can rewrite the expectation in the following way:

\begin{align}\label{eq:conditional prob wr yr}
     &\E_{\bw\sim\Ucal(d\sd)}\left[ \E_{\by\sim\Ncal(0,I)}\left[\left|\psi_a(\inner{\bw,\by})\cdot g(w_1,\dots,w_r,\by)\right|\right]\right] \nonumber\\
     &= \E_{w_1,\dots,w_r}\left[\E_{\bar{\bw}}\E_{y_1,\dots,y_r}\left[\E_{\bar{\by}}\left[\left|\psi_a\left(\sum_{i=1}^r w_iy_i + \inner{\bar{\bw},\bar{\by}}\right)\cdot g(w_1,\dots,w_r,\by)\right|  | y_1,\dots,y_r\right] \mid w_1,\dots,w_r\right]\right] \nonumber\\
    = & \E_{w_1,\dots,w_r}\E_{y_1,\dots,y_r}\E_{\bar{\bw}}\E_{\bar{\by}}\left[\left|\psi_a\left(\sum_{i=1}^r w_iy_i + \inner{\bar{\bw},\bar{\by}}\right)\cdot g(w_1,\dots,w_r,\by)\right| \mid y_1,\dots,y_r,w_1,\dots,w_r\right]~.
\end{align}
Namely, we consider the expectation conditioned on drawing the first $r$ coordinates of both $\bw$ and $\by$. Note that we could change the order of expectations since all the expectations are bounded and finite.

Let $\tilde{\psi}$ be a continuation of $\psi_a$ from $[-a,a]$ to $\reals$ such that it is periodic. Fix $w_1,\dots,w_r,y_1,\dots,y_r$ and denote by $s:= \sum_{i=1}^r w_iy_i$ and $\norm{\bar{\bw}} = 2\rho$.
Using Claim \ref{claim:from dist specific} we have that:
\begin{equation}\label{eq:inner prod periodic bound}
    \E_{\bar{\bw}\sim\Ucal(2\rho\mathbb{S}^{d-r-1})}\left[\left|\inner{g(\cdot),\tilde{\psi}(s + \inner{\bar{\bw},\cdot})}\right|\right] \leq c_1\cdot\left(\exp(-c_2(d-r)) + \sum_{n=1}^\infty \exp(-n\rho^2)\right)~.
\end{equation}
    
Note that in the above equation, $g$ is independent of $\bar{\bw}$ (although it does depend on $w_1,\dots,w_r$), and also that $\norm{g} \leq 1$ since $\sup_\vx|g(\vx)|\leq 1$ (recall that the norm is w.r.t a Gaussian measure).

We now have that:
\begin{align}\label{eq:psi tilde psi a}
    &\E_{\bar{\bw}\sim\Ucal(2\rho\mathbb{S}^{d-r-1})}\left[\left|\inner{g(\cdot),\psi_a(s + \inner{\bar{\bw},\cdot})}\right|\right] \nonumber\\
     \leq &\E_{\bar{\bw}\sim\Ucal(2\rho\mathbb{S}^{d-r-1})}\left[\left|\inner{g(\cdot),\tilde{\psi}(s + \inner{\bar{\bw},\cdot})}\right|\right] +     \E_{\bar{\bw}\sim\Ucal(2\rho\mathbb{S}^{d-r-1})}\left[\left|\inner{g(\cdot),\psi_a(s + \inner{\bar{\bw},\cdot}) - \tilde{\psi}(s + \inner{\bar{\bw},\cdot})}\right|\right]
\end{align}    
    The first term in \cref{eq:psi tilde psi a} can be bounded by $c_1\cdot\left(\exp(-c_2(d-r)) + \sum_{n=1}^\infty \exp(-n\rho^2)\right)$ by \cref{eq:inner prod periodic bound}. For the second term, by Cauchy-Schwartz we have that:
    \begin{align}
        & \E_{\bar{\bw}}\left[\left|\inner{g(\cdot),\psi_a(s + \inner{\bar{\bw},\cdot}) - \tilde{\psi}(s + \inner{\bar{\bw},\cdot})}\right|\right] \\
        \leq & \norm{g}\cdot \E_{\bar{\bw}}\left[\norm{\psi_a(s + \inner{\bar{\bw},\cdot}) - \tilde{\psi}(s + \inner{\bar{\bw},\cdot})}\right] \\
        \leq & \E_{\bar{\bw}}[\pr(|s + \inner{\bar{\bw},\bar{\by}}| 
 > a) ]
    \end{align}
where we used that $\norm{g}\leq 1$ and it is independent of $\bar{\bw}$, and that $\psi_a(z) = \tilde{\psi}(z)$ for every $|z| \leq a$. We have that $s + \inner{\bar{\bw},\bar{\by}} = \inner{\bw,\by}$, and $\inner{\bw,\by}\sim \Ncal(0,d^2)$ for every $\bw$ of norm $d$.
In particular, for $a \geq 2d^2$ there is some constant $c_3$ such that $\pr(|s + \inner{\bar{\bw},\bar{\by}}| > a)\leq \exp(-c_3 d)$. Combining the above we have that:

\begin{align}
    \E_{\bar{\bw}\sim\Ucal(2\rho\mathbb{S}^{d-r-1})}\left[\left|\inner{g(\cdot),\psi_a(s + \inner{\bar{\bw},\cdot})}\right|\right] \leq c_1\cdot\left(\exp(-c_2(d-r)) + \sum_{n=1}^\infty \exp(-n\rho^2)\right) + \exp(-c_3d)~.
\end{align}
We now go back to \cref{eq:conditional prob wr yr} and consider the conditional probability over $y_1,\dots,y_r$ and $w_1,\dots,w_r$. Note that when taking expectation over $y_1,\dots,y_r$ we either have that $|\inner{\bw,\by}|\leq a$ which happens w.p $>1-\exp(-c_3d)$ or $|\inner{\bw,\by}|\geq a$ in which case, since $,|g(z)|,|\psi_a(z)|\leq 1$ for every $z\in\reals$ also their product is bounded by $1$. 

Finally, we consider the expectation over $w_1,\dots,w_r$. We need to show that with high probability, $\rho = \frac{1}{2}\cdot\norm{\bar{\bw}}$ is large. Instead, we will consider the probability over $w_{r+1},\dots,w_d$ (note that since $\norm{\bw}=d$, if we lower bound $\norm{\bar{\bw}}$ it will also upper bound $\sqrt{\sum_{i=1}^rw_i^2}$). Since $\bw$ is sampled uniformly from $\Ucal(d\sd)$, we can instead consider sampling $z_i$ from $\Ncal(0,1)$ and setting $(\bw)_i = d\cdot\frac{z_i}{\norm{\bz}}$. By standard concentration bound on the norm of Gaussian random variables (see Section 3.1 in \cite{vershynin2018high}) there is some constant $c_4$ such that $\pr(\norm{\bar{\bw}}^2\notin[0.9(d-r),1.1(d-r)])\leq \exp(-c_4(d-r))$. Also, $\sum_{i=r+1}^d z_i^2$ has a $\chi^2$ distribution with $d-r$ degrees of freedom. 
From Lemma 1 in \cite{laurent2000adaptive} we have that $\pr\left(\sum_{i=r+1}^d w_i^2 \geq \frac{1}{2}\cdot(d-r)\right)\leq \exp(-c_5(d-r))$ for some constant $c_5$. Together, there is some constant $c_6$ such that $\pr(\norm{\bar{\bw}}^2 \geq \frac{1}{6}(d-r)) \leq \exp(-c_6(d-r))$.

Note that if $\rho > c'\sqrt{d-r}$ then $\sum_{i=1}^\infty \exp(-n\rho^2)\leq \exp(-c'(d-r))$.
Combining all the above and changing the constant terms appropriately, there is some universal constant $c>0$ such that:
\begin{equation}
        \E_{\bw\sim\Ucal(d\sd)}\left[ \E_{\by\sim\Ncal(0,I)}\left[\left|\psi_a(\inner{\bw,\by})\cdot g(w_1,\dots,w_r,\by)\right|\right]\right]\leq \exp(-c(d-r))
\end{equation}

\end{proof}

\begin{claim}\label{claim:from dist specific}
    For any $f\in L^2(\Ncal(0,I_d))$, odd periodic function $\psi:\reals\rightarrow \reals$ and $s\in\reals$, if $d > c'$ we have that:
    \begin{equation}
    \E_{\bw\sim\Ucal(2\alpha\sd)}\left[\left|\inner{f(\cdot),\psi(s + \inner{\bw,\cdot}}\right|\right] \leq c_1\norm{f}\cdot\left(\exp(-c_2d) + \sum_{n=1}^\infty \exp(-n\alpha^2)\right)~,
    \end{equation}
    here $c',c_1,c_2 > 0$ are some universal constants.
\end{claim}

\begin{proof}
    The proof is similar to the proof of Claim 1 from \cite{yehudai2019power} (which is directly derived from Lemma 5 in \cite{shamir2018distribution}), except for two changes:
    \begin{enumerate}
        \item Here we have an absolute value over the inner product, instead of a square as in Claim 1.
        \item We consider a translation of $\psi$, namely our periodic function is $\psi(s+\cdot)$ for a fixed $s$.
    \end{enumerate}
    
    For the first item, this is a direct application of Jensen's lemma:
    \begin{equation}
    \E_{\bw\sim\Ucal(2\alpha\sd)}\left[\sqrt{\left|\inner{f(\cdot),\psi(s + \inner{\bw,\cdot}}\right|^2}\right] \leq \sqrt{\E_{\bw\sim\Ucal(2\alpha\sd)}\left[\left|\inner{f(\cdot),\psi(s + \inner{\bw,\cdot}}\right|^2\right] }~,
    \end{equation}
    where now we can apply Claim 1 from \cite{yehudai2019power}. For the second item, note that $\psi(s+\cdot)$ is also a periodic function, and Lemma 5 from \cite{shamir2018distribution} applies to it in the same way as it does on $\psi(\cdot)$.
    
\end{proof}

\begin{theorem}\label{thm:lower bound indicator}

    There exists a bias term $b^*\in \reals$ such that for any choice of rank-$r$ heads $g_1,\dots,g_H$ each of the form $g_h(\vx_1,\vx_2,\by): = V_h X\phi_h\left(K_hX, \by\right)$ and any $V_1,\dots,V_H\in\reals^{d\times d}$, if $H\cdot\max_{h}\norm{V_{h}} \leq \frac{\exp(c_1 (d-r))}{d^2c_2}$ then:
    \begin{equation}
        \E_{\vx_1,\vx_2\sim\unif(d^2\sd),\vy\sim\Ncal(0,I)}\left[\norm{\mathds{1}(\inner{\vx_1-\vx_2,\by} + b^* > 0)\vx_1 - \sum_{i=h}^H V_hg_h(\vx_1,\vx_2,\by)}^2\right] > \frac{1}{20}~,
    \end{equation}
    where $c_1,c_2$ are some universal constants.
    
\end{theorem}

\begin{proof}
    Pick $a= 2d^2 + 1$, and recall the definition of $\psi_a$ from \cref{eq: def of psi}. In the proof, unless stated otherwise, the expectation is over $\vx_1,\vx_2\sim\unif(d\sd)$ and $\vy\sim\Ncal(0,I)$. In the last part of the proof we will multiply the norm of $\vx_1$ and $\vx_2$ by $d$.
    Assume towards contradiction that for every $b_j\in\left\{-a, -a+1,\dots, a\right\}$ we can find $V_1^j,\dots,V_H^j$ and rank-$r$ heads $g_1^j,\dots,g_H^j$ such that:
    \begin{equation}
    \E_{\vx_1,\vx_2,\by}\left[\norm{\sum_{h=1}^HV_h^jg_h^j(\vx_1,\vx_2,\by) - \mathds{1}(\inner{\vx_1 -\vx_2,\by} + b_j > 0)\vx_1}^2\right] \leq \epsilon~,
    \end{equation}
    and in addition there are $\mV_1^{a+1},\dots,\mV_H^{a+1}$ and rank-$r$ heads $g_1^{a+1},\dots,g_H^{a+1}$ with:
    \begin{equation}
    \E_{\vx_1,\vx_2,\by}\left[\norm{\sum_{h=1}^H\mV_h^{a+1}g_h^{a+1}(\vx_1,\vx_2,\by) + \frac{1}{2}\cdot\vx_1}^2\right] \leq \epsilon~,
    \end{equation}
    where $\epsilon$ will be chosen later on.
    Define $a_j = (-1)^j$, then we have that:
    
    \begin{align}\label{eq:upper bound psi g}
    & \E_{\vx_1,\vx_2,\by}\left[\norm{\sum_{j=-a}^{a+1}\sum_{h=1}^H \mV_h^jg_h^j(\vx_1,\vx_2,\by) - \psi_{a}(\inner{\vx_1 - \vx_2,\by})\vx_1 }^2\right] \nonumber\\
     = &\E_{\vx_1,\vx_2,\by}\left[\norm{\sum_{j=-a}^{a+1}\sum_{h=1}^H \mV_h^jg_h^j(\vx_1,\vx_2,\by) - \sum_{j=-a}^a a_j \mathds{1}(\inner{\vx_1 - \vx_2,\by} + b_j > 0)\vx_1 + \frac{1}{2}\vx_1 }^2\right] \nonumber\\
    \leq & \left(\sum_{j=-a}^a \E_{\vx_1,\vx_2,\by}\left[\norm{\sum_{h=1}^H\mV_h^jg_h^j(\vx_1,\vx_2,\by) - \mathds{1}(\inner{\vx_1 - \vx_2,\by} + b_j > 0)\vx_1}^2\right] \right)^2 + \nonumber\\ 
    + & \left(\E_{\vx_1,\vx_2,\by}\left[\norm{\sum_{h=1}^H\mV_h^{a+1}g_h^{a+1}(\vx_1,\vx_2,\by) + \frac{1}{2}\cdot\vx_1}^2\right] \right)^2\nonumber\\
     \leq & \epsilon^2 \cdot (2a+1)^2 \leq 5\epsilon^2 a^2~.
    \end{align}

    On the other hand, we have that :
    \begin{align}\label{eq:lower bound psi sum g}
        & \E_{\vx_1,\vx_2,\by}\left[\norm{\sum_{j=-a}^{a+1}\sum_{h=1}^H \mV_h^jg_h^j(\vx_1,\vx_2,\by) - \psi_{a}(\inner{\vx_1 - \vx_2,\by})\vx_1 }^2\right] \nonumber\\
        \geq & \E_{\vx_1,\vx_2,\by}[\norm{\vx_1}^2\cdot |\psi_{a}(\inner{\vx_1 - \vx_2,\by})|^2] - 2\E_{\vx_1,\vx_2,\by}\left[\inner{\sum_{j=-a}^{a+1}\sum_{h=1}^H \mV_h^jg_h^j(\vx_1,\vx_2,\by), \psi_{a}(\inner{\vx_1 -\vx_2,\by})\vx_1}\right] \nonumber\\
        \geq  & d^2\E_{\vx_1,\vx_2,\by}[ |\psi_{a}(\inner{\vx_1 - \vx_2,\by})|^2] - 2\sum_{j=-a}^{a+1}\sum_{h=1}^H\E_{\vx_1,\vx_2,\by}\left[\left|\inner{ \mV_h^jg_h^j(\vx_1,\vx_2,\by), \psi_{a}(\inner{\vx_1 - \vx_2,\by})\vx_1}\right|\right] 
    \end{align}
    We will now bound each term of the form $\E_{\vx_1,\vx_2,\by}\left[\left|\inner{ \mV_h^jg_h^j(\vx_1,\vx_2,\by), \psi_{a}(\inner{\vx_1 - \vx_2,\by})\vx_1}\right|\right]$.
    Each rank-$r$ head can be written as (omitting the $h$ and $j$ indices for brevity):
    \begin{align}
        g(\vx_1,\vx_2,\by) &= \mV\mX\phi(\mK\mX,\by)) \\
        & = \mV\mX \begin{pmatrix}
            g_1(\mK\mX,\by) \\ g_2(\mK\mX,\by)
            \end{pmatrix}
            \\ &= \mV (\vx_1g_1(\mK\mX,\by) + \vx_2 g_2(\mK\mX,\by))
    \end{align}
    where $\mK,\mQ\in\reals^{d\times r}$ and $g_1,g_2$ are some function with output bounded by $1$. We can bound:
    \begin{align}
        &\E_{\vx_1,\vx_2,\by}\left[\left|\inner{ \mV g(\vx_1,\vx_2,\by), \psi_{a}(\inner{\vx_1 - \vx_2,\by})\vx_1}\right|\right] \\
        = &  \E_{\vx_1,\vx_2,\by}\left[\left|\inner{ \mV (\vx_1g_1(\mK\mX,\by) + \vx_2 g_2(\mK\mX,\by)), \psi_{a}(\inner{\vx_1 - \vx_2,\by})\vx_1}\right|\right]\\ 
        \leq  & \norm{\mV}\cdot d^2 \E_{\bx_1,\bx_2,\by}\left[\left|g_1(\mK\mX,\by)\psi_{a}(\inner{\bx_1 - \bx_2,\by})\right| + \left|g_2(\mK\mX,\by)\psi_{a}(\inner{\bx_1 - \bx_2,\by})\right|\right] \\
        \leq & d^2\max_{h,j}\norm{V_h^j} \left(\E_{\bx_1,\bx_2,\by}\left[\left|g_1(\mK\mX,\by)\psi_{a}(\inner{\bx_1 - \bx_2,\by})\right|\right] + \E_{\bx_1,\bx_2,\by}\left[\left|g_2(\mK\mX,\by)\psi_{a}(\inner{\bx_1 - \bx_2,\by})\right|\right]\right) 
    \end{align}
    Since $\bx_1$ and $\bx_2$ have a symmetric distribution and $\mK$ has rank-$r$, we can assume w.l.o.g that the image of $K$ lies in $\text{span}\{\be_1,\dots,\be_r\}$. Denote $\bw:= \bx_1-\bx_2$, and note that $g_1$ and $g_2$ can now be written as a function of $w_1,\dots,w_r,\by$.  Also, by the assumption on the distribution we have $\bx_1\perp\bx_2$, hence $\bw\sim\Ucal(\sqrt{2}d\sd)$.
    Hence, we can use  \cref{thm:exp bound correlation} to get a constant $c_1 > 0$ such that:
    \begin{align}
    &\E_{\bw\sim\unif(\sqrt{2}d\sd),\by\sim\Ncal\left(0,I\right)}\left[\left| g_1(w_1,\dots,w_r,\by)\cdot  \psi_{a}(\inner{\bw,\by})\right|\right] 
    \leq \exp(-c_1(d-r))~.    
    \end{align}

    Note that this is true for $g_1,g_2$ and any rank-$r$ head. Hence, applying this and \cref{lem:properties psi} (2) to \cref{eq:lower bound psi sum g} we have: 
    \begin{equation}
    \E_{\bx_1,\bx_2,\by}\left[\norm{\sum_{j=-a}^{a+1}\sum_{h=1}^H \mV_h^jg_h^j(\bx_1,\bx_2,\by) - \psi_{a}(\inner{\bx_1 - \bx_2,\by})\bx_1 }^2\right] \geq \frac{d^2}{40} - 6H\max_{h,j}\norm{\mV_{h}^j}d^4\exp(-c_1(d-r))~.
    \end{equation}
    Combining this with \cref{eq:upper bound psi g} we have:
    \begin{equation}
    \frac{d^2}{40} - 6H\max_{h,j}\norm{\mV_{h}^j}d^4\exp(-c_1(d-r)) \leq 5\epsilon^2 d^4~.
    \end{equation}
    
    Combining all the above results, we get that there exists a bias term $b^*$, such that for all choice of heads $g_h$ and matrices $\mV_h$, if $H\cdot \max_h\norm{\mV_h} \leq \frac{\exp(c_1(d-r))}{6d^2}\cdot \left(\frac{1}{40} - 5\epsilon^2d^2\right)$, then:
    \begin{equation}
        \E_{\bx_1,\bx_2,\by}\left[\norm{\sum_{h=1}^H\mV_hg_h(\bx_1,\bx_2,\by) - \mathds{1}(\inner{\bx_1 - \bx_2,\by} + b^* > 0)\bx_1}^2\right] > \epsilon~.
    \end{equation}
    To finish the proof we need to make sure that $\frac{1}{40} - 5\epsilon^2d^2 > 0$, to achieve this we will scale the problem by a factor of $d$. We multiply the above displayed equation by $d$, and set $\epsilon = \frac{1}{20d}$ to get that there is a constant $c_2 > 0$ such that if $H\cdot \max_h\norm{\mV_h} \leq \frac{c_2\exp(c_1(d-r))}{d^2}$ then:
    \begin{equation}
    \E_{\bx_1,\bx_2,\by}\left[\norm{\sum_{h=1}^H d \mV_hg_h(\bx_1,\bx_2,\by) - d\cdot\mathds{1}(\inner{\bx_1 - \bx_2,\by} + b^* > 0)\bx_1}^2\right] > \frac{1}{20}~.
    \end{equation}
    Finally, we replace the distribution  of $\bx_1,\bx_2$ by $\unif(d^2\sd)$, namely, we multiply the norm by $d$. We also multiply $b^*$ by $d$, hence the threshold function remains unchanged. Since the above is true for any function $g_h$ and matrices $\mV_h$, we can also scale them by a factor of $d$ to achieve the result.
    
\end{proof}

We are ready to prove the main theorem:

\begin{proof}[Proof of \cref{thm:biased_lower_bound}]
    By \cref{thm:lower bound indicator} there is $b^*$ such that 
        \begin{equation}
        \E_{\bx_1,\bx_2,\by\sim \Dcal}\left[\norm{\mathds{1}(\inner{\bx_1-\bx_2,\by} + b^* > 0)\bx_1 - \sum_{h=1}^H \mV_hg_h(\bx_1,\bx_2,\by)}^2\right] > \frac{1}{20}~,
    \end{equation}
    Pick $\bb^* = \begin{pmatrix}
        b^* \\ 0
    \end{pmatrix}$, and write:
    \begin{equation}
    f(\bx_1,\bx_2,\by) = \arg\max_{\bx_i}\inner{\bx_i,\by} + b_i = \mathds{1}(\inner{\bx_1 - \bx_2,\by} + b^* > 0)\bx_1 + \mathds{1}(\inner{\bx_1 - \bx_2,\by} + b^* < 0)\bx_2~.
    \end{equation}
    Denote $f_1(\bx_1,\bx_2,\by):= \mathds{1}(\inner{\bx_1 - \bx_2,\by} + b^* > 0)\bx_1$ and $f_2(\bx_1,\bx_2,\by):=\mathds{1}(\inner{\bx_1 - \bx_2,\by} + b^* < 0)\bx_2$ and $g(\bx_1,\bx_2,\by)= \sum_{i=h}^H \mV_hg_h(\bx_1,\bx_2,\by)$. With these notations, we want to lower bound:
    \begin{align}
    &\E_{\bx_1,\bx_2}\left[\norm{f_1(\bx_1,\bx_2,\cdot) + f_2(\bx_1,\bx_2,\cdot) - g(\bx_1,\bx_2,\cdot)}^2\right] \\
    \geq & \E_{\bx_1,\bx_2}\left[\frac{1}{\norm{f(\bx_1,\bx_2,\cdot)}^2}\cdot\left|\inner{f_1(\bx_1,\bx_2,\cdot) + f_2(\bx_1,\bx_2,\cdot) - g(\bx_1,\bx_2,\cdot),f_1(\bx_1,\bx_2,\cdot)}\right|^2\right] \\
    = & \E_{\bx_1,\bx_2}\left[\frac{1}{\norm{f(\bx_1,\bx_2,\cdot)}^2}\cdot\left|\inner{f_1(\bx_1,\bx_2,\cdot)  - g(\bx_1,\bx_2,\cdot),f_1(\bx_1,\bx_2,\cdot)}\right|^2\right]
    \end{align}
    where the norm is w.r.t the Gaussian measure (i.e. w.r.t $\by$). We will now lower bound the terms inside the expectation.
    
    Note that if $\pr_{\bx_1,\bx_2,\by}(\mathds{1}(\inner{\bx_1 - \bx_2,\by} + b^* > 0) = 1) \leq \frac{1}{20}$, then approximating $\mathds{1}(\inner{\bx_1-\bx_2,\by} + b^* > 0)$ with the zero function would achieve an approximation error better than $\frac{1}{20}$, in contradiction to \cref{thm:lower bound indicator}. Hence $\pr_{\bx_1,\bx_2,\by}(\mathds{1}(\inner{\bx_1-\bx_2,\by} + b^* > 0) = 1) \geq \frac{1}{20}$. Also, note that $\norm{f_1(\bx_1,\bx_2,\cdot)}^2 = \E_{\by}[\inner{f_1(\bx_1,\bx_2,\by),f_1(\bx_1,\bx_2,\by)}] = \E_{\by}\left[\norm{\bx_1}^2\mathds{1}(\inner{\bx_1 - \bx_2,\by} + b^* > 0)\right]$ is independent of the choice of $\bx_1$ and $\bx_2$, since $\by$ has a spherically symmetric distribution, and the norm of $\bx_1$ is constant. Hence:
    
    \begin{align}
    &\E_{\bx_1,\bx_2}\left[\frac{1}{\norm{f(\bx_1,\bx_2,\cdot)}^2}\cdot\left|\inner{f_1(\bx_1,\bx_2,\cdot)  - g(\bx_1,\bx_2,\cdot),f_1(\bx_1,\bx_2,\cdot)}\right|^2\right] \\
    \geq & \frac{1}{d^2} \E_{\bx_1,\bx_2}\left[\left|\inner{f_1(\bx_1,\bx_2,\cdot)  - g(\bx_1,\bx_2,\cdot),f_1(\bx_1,\bx_2,\cdot)}\right|^2\right]~.    
    \end{align}

    We will bound the inner product inside the expectation. Let $A:=\{(\bx_1,\bx_2,\by)\in\reals^{d\times 3}:\mathds{1}(\inner{\bx_1 - \bx_2,\by} + b^* > 0) > 0\}$. Note that:
    \begin{equation}
    \E_{\bx_1,\bx_2,\by}\left[\norm{f_1(\bx_1,\bx_2,\by)-g(\bx_1,\bx_2,\by)}^2\cdot\mathds{1}((\bx_1,\bx_2,\by)\in A)\right] \geq \frac{1}{20}~,
    \end{equation}
    otherwise, taking $g(\bx_1,\bx_2,\by)$ to be the zero function would approximate $f_1(\bx_1,\bx_2,\by)$ with error less than $\frac{1}{20}$. Hence, we have that:
    \begin{align}
        & \frac{1}{d^2} \E_{\bx_1,\bx_2,\by}\left[\left|\inner{f_1(\bx_1,\bx_2,\cdot)  - g(\bx_1,\bx_2,\cdot),f_1(\bx_1,\bx_2,\cdot)}\right|^2\right] \\
        \geq & \frac{1}{d^2}\E_{\bx_1,\bx_2,\by}\left[\left|\inner{f_1(\bx_1,\bx_2,\cdot)  - g(\bx_1,\bx_2,\cdot),f_1(\bx_1,\bx_2,\cdot)}\right|^2\cdot \mathds{1}((\bx_1,\bx_2,\by)\in A)\right] \\
        = & \frac{1}{d^2}\E_{\bx_1,\bx_2,\by}\left[\norm{f_1(\bx_1,\bx_2,\cdot)  - g(\bx_1,\bx_2,\cdot)}^2\cdot \mathds{1}((\bx_1,\bx_2,\by)\in A)\norm{\bx_1}^2\right] \geq \frac{1}{20}
    \end{align}
\end{proof}

\section{Proofs from \texorpdfstring{\cref{sec:majority}}{Section 6} and an Additional Construction}
\label{sec:majority proofs}

In \cref{sec:majority}, we present a construction (\cref{thm:majority_positional}) that uses concatenated positional encodings to facilitate the majority voting strategy.
This construction has the strange property that it breaks the permutation invariance of standard attention layers in order to approximate a function that is permutation invariant.
It also increases the dimension of the transformer.
This begs the question of whether these properties are necessary to allow low-rank attention to represent the target.
Below, we presenting an alternative construction that does not have these properties.
Instead, it modifies the attention mechanism by concatenating the outputs of the heads together rather than summing them.
It then passes the concatenated outputs to an MLP layer that computes the mode.

\begin{restatable}[Majority Voting Approximation Upper Bound]{theorem}{thm:majority_wide_hidden}
\label{thm:majority_wide_hidden}
There exist universal constants $c_1, c_2,c_3,c_4 > 0$ such that for all $d > c_1$, $\epsilon \in \left(0, \frac{1}{2}\right)$, and $H \geq c_2\cdot\frac{d^3}{\epsilon^2}$, there exist vectors $\vq_1, \ldots, \vq_H$ and a $4$-layer feedforward network $g: \R^{dH} \to \R^d$ of width $c_3d^2H$ such that
\begin{equation}
\E_{\vx_1,\vx_2,\vy\sim 
\unif(\sphere)}\left\|f(\vx_1,\vx_2;\vy) - 
g\left(\begin{bmatrix}\mX \sm(\mX^\top \vq_1\vq_1^\top \vy) \\ 
\vdots \\ \mX \sm(\mX^\top\vq_H\vq_H^\top \vy) 
\end{bmatrix}\right) \right\|_2^2 \leq \epsilon + \exp(-c_4 d)~.
\end{equation}
\end{restatable}

This construction shows that using a constant-depth MLP to combine the heads can overcome the weakness of low rank attention.
The full proof can be found in \cref{sec:majority_wide_hidden_proof}.
The idea behind the construction of the MLP $g(\cdot)$ is to perform an inner product between the outputs of the heads, allowing us to compare which one of the outputs $\vx_1$ or $\vx_2$ received more votes.
The inner products can be approximated by a ReLU network, as long as the input vectors are not too close to each other, which happens with exponentially large probability. This is the cause of the extra exponentially small term in the loss.

\subsection{Lemmas}
To prove \cref{thm:majority_positional,thm:majority_wide_hidden} we will need several lemmas.

The first shows that for a fixed set of inputs, drawing a rank-$1$ head randomly will have the same output as the target $f$ with probability slightly larger than $\frac{1}{2}$.
This lemma justifies our majority voting strategy.

\begin{lemma}\label{lem:single head N=2 edge}
    Fix $\bx_1,\bx_2,\by \in\sd$ with $|\inner{\bx_1-\bx_2,\by}| \geq a$ for some $a > 0$.
    Then for $d > c_1$ we have that:
    \begin{equation}
    \Pr_{\bq\sim\Ucal(\sd)}\left( \arg\max_i \inner{\bx_i,\bq}\cdot \inner{\by,\bq} = \arg\max_i\inner{\bx_i,\by}\right)\geq \frac{1}{2} + c_2\cdot \frac{a}{\sqrt{d}}
    \end{equation}
    for some universal constants $c_1,c_2>0$.
\end{lemma}

\begin{proof}
    In the proof, all probabilities are for $\bq\sim\Ucal(\sd)$, thus we omit this notation. Denote $\bw:= \bx_1 - \bx_2$, and assume w.l.o.g that $\inner{\bw,\by} > 0$, the other direction is similar. We can write:
    \begin{align}
        &\pr\left( \arg\max_i \inner{\bx_i,\bq}\cdot \inner{\by,\bq} = \arg\max_i\inner{\bx_i,\by}\right) \nonumber\\
        = & \pr\left(\sgn(\inner{\bw,\by}) = \sgn(\inner{\bw,\bq}\cdot\inner{\by,\bq})\right) \nonumber \\
        = &\pr\left( \inner{\bw,\bq}\cdot\inner{\by,\bq} > 0\right) ~.\nonumber
    \end{align}
    Since the above probability is rotation invariant w.r.t $\bq$, we can assume w.l.o.g that $\bw = \be_1$. Hence we can write $\by = \begin{pmatrix}
        \tilde{a} \\ \bar{\by}
    \end{pmatrix}$, where $\bar{\by}\in \reals^{d-1}$ and $\tilde{a} = \inner{\bw,\by}$. Thus, the above probability is equal to:
    \begin{align}
    & \pr\left(q_1(\tilde{a}q_1 
 + \inner{\bar{\bq},\bar{\by}}) > 0\right) \\
 = & \frac{1}{2}\pr\left(q_1(\tilde{a}q_1 
 + \inner{\bar{\bq},\bar{\by}}) > 0| q_1 > 0\right) + \frac{1}{2}\pr\left(q_1(\tilde{a}q_1 
 + \inner{\bar{\bq},\bar{\by}}) > 0| q_1 < 0\right) \\
  =& \frac{1}{2}\pr\left(\tilde{a}q_1 
 + \inner{\bar{\bq},\bar{\by}} > 0| q_1 > 0\right) + \frac{1}{2}\pr\left(\tilde{a}q_1 
 + \inner{\bar{\bq},\bar{\by}} < 0| q_1 < 0\right) \\
 = & \pr\left(\tilde{a}q_1 
 + \inner{\bar{\bq},\bar{\by}} > 0| q_1 > 0\right)
    \end{align}
    where the last equality is by the symmetry of the distribution of $\bq$. Note that if $\inner{\bar{\bq},\bar{\by}} > 0$ which happens w.p $\frac{1}{2}$, then the term inside the above probability is positive. Hence, we can write:
    \begin{align}\label{eq:q y probabilites bound}
 & \pr\left(\tilde{a}q_1 
 + \inner{\bar{\bq},\bar{\by}} > 0| q_1 > 0\right) \nonumber\\
 =& \frac{1}{2} + \frac{1}{2}\cdot\pr\left(\tilde{a}q_1 
 + \inner{\bar{\bq},\bar{\by}} > 0| q_1 > 0, \inner{\bar{\bq},\bar{\by}} < 0 \right) \nonumber\\
 \geq & \frac{1}{2} + \frac{1}{2}\cdot\pr\left(\tilde{a} q_1 \geq \frac{2\tilde{a}}{\sqrt{d}} | q_1 > 0\right) \cdot \pr\left(|\inner{\bar{\bq},\bar{\by}}| \leq \frac{\tilde{a}}{\sqrt{d}} | \inner{\bar{\bq},\bar{\by}}  < 0\right)
    \end{align}
    We will now lower bound each probability separately. First, note that if we sample $\bu\sim\Ncal\left(0,\frac{1}{d}I\right)$, then $\frac{u_1}{\norm{\bu}}$ has the same distribution as $q_1$. By the concentration of the norm of Gaussian random variables (see \cite{vershynin2018high} Section 3.1), there is a constant $c_1>0$ such that w.p $>1 - \exp(-c_1d)$ we have $\norm{\bu}\in[0.9,1.1]$. There is also a constant $c_2 \in \left(0,\frac{1}{2}\right)$ such that $\pr\left(u_1 > \frac{3}{\sqrt{d}}\right) > c_2$. This bounds the first probability term in \cref{eq:q y probabilites bound}. For the second term, note that $\norm{\bar{\by}} \leq \norm{\by} =  1$. By the same reasoning as above we can write:
    \begin{align}
        &\pr\left(|\inner{\bar{\bq},\bar{\by}}| \leq \frac{\tilde{a}}{\sqrt{d}} | \inner{\bar{\bq},\bar{\by}}  < 0\right) \geq  \pr\left(|\inner{\bar{\bq},\bar{\by}}| \leq \frac{a}{\sqrt{d}} | \inner{\bar{\bq},\bar{\by}}  < 0\right)\\
         = &  \pr_{\bu\sim\Ncal\left(0,\frac{1}{d}I\right)}\left(\left|\frac{u_2}{\norm{\bu}}\right| \leq \frac{a}{\sqrt{d}} \right) 
          \geq   (1-\exp(-c_1d))\cdot \pr_{u_2\sim\Ncal\left(0,\frac{1}{d}\right)}\left(\left|u_2\right| \leq \frac{a\cdot 0.9}{\sqrt{d}} \right)  
    \end{align}
    The above probability is bounded by $\text{erf}\left(\frac{a\cdot 0.9}{\sqrt{d}}\right) \geq \frac{a\cdot 0.9}{\sqrt{d}}$, where this inequality is since $\text{erf}(z) > z$ for $z\in \left[0,\frac{1}{2}\right]$. In total, we can bound this probability by 
    \begin{equation}
    \pr\left(|\inner{\bar{\bq},\bar{\by}}| \leq \frac{a}{\sqrt{d}} | \inner{\bar{\bq},\bar{\by}}  < 0\right) \geq  (1-\exp(-c_1d))\cdot \frac{a\cdot 0.9}{\sqrt{d}}~.
    \end{equation} We take $d >\tilde{c}$ so that $\exp(-c_1d) \leq \frac{1}{2}$, Combining the two bounds, and changing the universal constant finishes the proof.
\end{proof}

The following lemma shows that a random draw of inputs will satisfy a certain condition which allows the use of the previous lemma.

\begin{lemma}\label{lem:inner prod small x_1 x_2 y}
    Let $\epsilon > 0$, then:  
    \begin{equation}
    \pr_{\vx_1,\vx_2,\vy\sim\unif(\sd)}\left(|\inner{\vx_1-\vx_2,\vy}| \leq \epsilon \right) \leq (1-\exp(-c_1d))\cdot 2\epsilon\sqrt{d}~,
    \end{equation}
    where $c_1 > 0 $ is some universal constant.
\end{lemma}

\begin{proof}
    By the symmetry of the distribution, we can assume w.l.o.g that $\vy= \ve_1$. Also, note that for $\vu,\vv\sim\Ncal\left(0,\frac{1}{d}I\right)$, we can view the distribution of $(\vx_1)_1$ and $(\vx_2)_1$ as $\frac{u_1}{\norm{\vu}}$ and $\frac{v_1}{\norm{\vv}}$. Combining the above, we get that:
    \begin{align}
        \pr_{\vx_1,\vx_2,\vy\sim\unif(\sd)}\left(|\inner{\vx_1-\vx_2,\vy}| \leq \epsilon \right) 
        = \pr_{\vu,\vv\sim\Ncal\left(0,\frac{1}{d}I\right)}\left(\left|\frac{u_1}{\norm{\vu}} - \frac{v_1}{\norm{\vv}}\right| \leq \epsilon \right)~.
    \end{align}
    By the concentration of the norm of normal random vectors (see \cite{vershynin2018high} section 3.1) we have w.p $>1-\exp(-c_1d)$ that $\norm{\vu},\norm{\vv} \leq 1.1$ for some universal constant $c_1 > 0$. Also $z:=u_1 - v_1\sim\Ncal\left(0,\frac{2}{d}\right)$. Hence, the above probability can be upper bounded by $\Pr_{z\sim\Ncal\left(0,\frac{2}{d}\right)}\left(|z|< 1.1\epsilon\right)\leq  \text{erf}\left(\epsilon\sqrt{d}\right)$. Note that $\text{erf}(x) \leq 2x$ for every $x > 0$, hence the above probability can be bounded by $(1-\exp(-c_1d))\cdot 2\epsilon\sqrt{d}$
\end{proof}

The following lemma shows a construction of the majority function over $H$ input vectors. This construction uses an approximation of the inner product of two inputs using a ReLU network.

\begin{lemma}
\label{lem:mode_MLP}
Let $\vv_1, \ldots \vv_H \in \{\vx_+, \vx_-\} \subset \R^d$, where $\inner{\vx_-,\vx_+} \leq 0.1$.
Let $\vv^*$ be the mode of $\vv_1, \ldots \vv_H$.
Then there exists a $4$-layer feedforward network $g: \R^{d(H+2)} \to \R^d$ with width $c\cdot d^2 H$ for some universal constant $c > 0$ and weights bounded by $2$ such that
\begin{equation} 
g\left(\begin{bmatrix}\vv_1 \\ \vdots \\ \vv_H \\ \vx_+ \\ \vx_- \end{bmatrix}\right) = 
g\left(\begin{bmatrix}\vv_1 \\ \vdots \\ \vv_H \\ \vx_- \\ \vx_+ \end{bmatrix}\right) = \vv^* \end{equation}
\end{lemma}
\begin{proof}
    Let $\vx$ be the finally $d$ coordinate of $\vv:= \begin{bmatrix}\vv_1 & \cdots & \vv_H & \vx_- & \vx_+ \end{bmatrix}^\top\in\reals^{d(H+2)}$, and let $\hat{\vx}$ be the second to last block of $d$ coordinates of $\vv$. Note that either $\vx= \vx_+$ and $\hat{\vx}=\vx_-$ or the other way around. We construct a network that calculates the inner product between $\vx$ and each $\vv_i$ up to accuracy of $\frac{1}{10H}$. By \cref{lem:inner product with NN} there is such a $2$-layer network $M_1:\reals^{d(H+2)}\rightarrow\reals^{2d+1}$ with width $cd^2H$ for some universal constant $c > 0$ and weights bounded by $2$. We add $2d$ more neurons which act as two identity matrices to keep the last $2d$ coordinates of $\vv$. We add an additional output layer to $M_1$ which sums all the outputs of the inner products.

    We now construct another network $M_2:\reals^{2d+1}\rightarrow\reals^d$ which either output $\vx$ if the sums of the inner product is larger than $0.2\cdot H$ or $\hat{\vx}$  otherwise. Note that by our assumption that $\inner{\bx_1,\bx_2} \leq 0.1$, $M_2$ will output the mode of the $\vv_i$'s. This is because $M_1$ calculates inner products up to an error of $\frac{1}{10H}$, summing over $H$ such inner products returns the exact sum plus an error which is bounded by $\frac{1}{10}$. Composing $M_1$ and $M_2$ provides an MLP which will output either $\bx_+$ or $\bx_-$ depending on who is the mode.
    
    The total width of the network is $c_3d^2H$, since we calculate inner products up to an error of $\frac{1}{10H}$, and the depth of the network is $4$.
\end{proof}

We next show that shallow neural networks can approximately compute the inner product of two vectors.

\begin{lemma}\label{lem:inner product with NN}
    Let $\epsilon > 0$. There exists a $2$-layer network $N:\left(\sd\right)^2\rightarrow\reals$ with width $\frac{cd^2}{\epsilon}$ and weights bounded by $2$ that calculates $\inner{\bx,\bx'}$ up to accuracy $\epsilon$. Here $c > 0$ is some universal constant.
\end{lemma}

\begin{proof}
    By Lemma 6 in \cite{daniely2017depth} there exists a depth $2$ network $N_{\text{square}}:\reals\rightarrow\reals$ that calculates $\frac{x^2}{2}$ in $[-2,2]$ with an error of $\frac{\epsilon}{d}$, width of at most $\frac{32d}{\epsilon}$ and weights bounded by $2$. For each coordinate $i\in[d]$ we compose the linear function $(\bx)_i + (\bx')_i$ with $N_{\text{square}}$ to get a depth $2$ network that calculates $\frac{((\bx)_i + (\bx')_i)^2}{2}$ up to an error of $\frac{\epsilon}{d}$. Summing over these networks for every index $i$ and subtracting $1$ results in a network that calculates $\inner{\bx,\bx'}$ with an error of $\epsilon$ and width $\frac{32d^2}{\epsilon}$
\end{proof}

Finally, the following lemma shows that if we draw random rank-$1$ attention heads, taking their ``majority vote'' will approximate the target function $f$. The rate of approximation depends on the number of sampled heads and on the input dimension.

\begin{lemma}\label{lem:hoeffding for majority}
Let $M:\left(\reals^d\right)^H\rightarrow\reals^d$ be the majority function over $H$ vectors in $\reals^d$. Namely, given a set of $H$ vectors, $M$ outputs the vector which appears the most times in the set, and breaks ties randomly.
For a vector $\vq_h$ define $g_h(\vx_1,\vx_2;\vy) = \arg\max_{\vx_i}\inner{\vx_i,\vq_h}\cdot \inner{\vy,\vq_h}$.
There exist universal constants $c_1, c_2 > 0$ such that if $H > \frac{c_1d^3}{\epsilon^2}$, then with probability at least $1-\exp(c_2d)$ over samples $\vq_1,\dots,\vq_H\sim\unif(\sd)$, we have that:
\begin{equation} \E_{\vx_1,\vx_2;\vy\sim\unif(\sd)}\left[\norm{f(\vx_1,\vx_2;\vy) - M\left(\{g(\vx_1,\vx_2;\vy\}_{h=1}^H\right)}^2\right] \leq \epsilon~,
\end{equation}
Here, $f$ is defined as in \cref{eq:target}.
\end{lemma}

\begin{proof}
    
    Fix $\bx_1,\bx_2,\by$ with $|\inner{\bx_1 - \bx_2,\by}| \geq \epsilon$. Denote by $A_h$ the event over sampling $\vq\sim\unif(\sd)$ which output $1$ if $ \arg\max_i \inner{\bx_i,\bq_h}\cdot \inner{\by,\bq_h} = \arg\max_i\inner{\bx_i,\by}$ and $0$ otherwise. By \cref{lem:single head N=2 edge} we have that $\pr(A_h = 1) \geq \frac{1}{2} + c_2\cdot\frac{\epsilon}{\sqrt{d}}$ if $d > c_1$ for some universal constants $c_1,c_2 > 0$. Note that the events $\{A_h\}_{h=1}^H$ are independent when $\bx_1,\bx_2,\by$ are fixed. Hence, we can use Hoeffding's inequality:
    \begin{equation}
    \pr_{q_1,\dots,q_H}\left(\left|\frac{1}{H}\sum_{h=1}^H A_h - \left(\frac{1}{2} + c_2\cdot \frac{\epsilon}{\sqrt{d}} \right)\right|\geq t\right) \leq 2\exp(-2Ht^2)~.
    \end{equation}
    By setting $t = \frac{c2\epsilon}{\sqrt{d}}$ and $H \geq \frac{d^{2}}{\epsilon^2}$ we get that:
    \begin{equation}
    \pr\left(\frac{1}{H}\sum_{h=1}^H A_h < \frac{1}{2}\right) \leq 2\exp(-2c_2d)~.
    \end{equation}

    From now on, we condition on the event that $\vq_1,\dots,\vq_H$ are sampled such that $\frac{1}{H}\sum_{h=1}^H A_h \geq  \frac{1}{2}$, which happens w.p $ > 1 - 2\exp(-2c_2 d)$. Note that if this event happens, then the majority of the functions $g_h(\bx_1,\vx_2,\vy)$ will output the same vector as $f(\vx_1,\vx_2,\vy)$.
    
    By \cref{lem:inner prod small x_1 x_2 y} we have that $\Pr\left(|\inner{\bx_1-\bx_2,\by}| \leq \epsilon \right) \leq (1-\exp(-c_3d))\cdot 2\epsilon\sqrt{d}$ for some universal constant $c_3 > 0$. Hence, we get that:
    \begin{align}
         & \E_{\vx_1,\vx_2,\vy\sim\unif(\sd)}\left[\norm{f(\vx_1,\vx_2,\vy) - M\left(\{g(\vx_1,\vx_2,\vy\}_{h=1}^H\right)}^2\right] \\
         = & \Pr\left(|\inner{\bx_1-\bx_2,\by}| \leq \epsilon \right) \cdot \E\left[\norm{f(\vx_1,\vx_2,\vy) - M\left(\{g(\vx_1,\vx_2,\vy\}_{h=1}^H\right)}^2
         \Big\vert|\inner{\bx_1-\bx_2,\by}| \leq \epsilon  \right] + \\
         + &      \pr\left(|\inner{\bx_1-\bx_2,\by}| \geq \epsilon \right) \cdot 
         \E\left[\norm{f(\vx_1,\vx_2,\vy) - M\left(\{g(\vx_1,\vx_2,\vy\}_{h=1}^H\right)}^2\Big\vert|\inner{\bx_1-\bx_2,\by}| \geq \epsilon  \right] \\
         \leq & (1-\exp(-c_3d))\cdot 2\epsilon\sqrt{d}\cdot 1 + 1 \cdot \exp(-2c_2d) \leq c\cdot \epsilon\sqrt{d}
    \end{align}
    where we choose $d$ large enough such that $1-\exp(-c_3d)\geq \frac{1}{2}$ and $\exp(-2c_2d)\leq \frac{1}{2}$ and changed the constant $c>0$ accordingly. Replacing $\epsilon$ with $\tilde{\epsilon} =\frac{\epsilon}{c\sqrt{d}}$ finishes the proof.

\end{proof}

\subsection{Proof of \texorpdfstring{\cref{thm:majority_wide_hidden}}{Theorem 32}}
\label{sec:majority_wide_hidden_proof}
\begin{proof}
    By \cref{lem:hoeffding for majority} there exist $\vq_1,\dots,\vq_{H-2}\in\sd$ such that if $H \geq \frac{c_2d^3}{\epsilon^2}$:
    \begin{equation}
    \E_{\vx_1,\vx_2,\vy\sim\unif(\sd)}\left[\norm{f(\vx_1,\vx_2,\vy) - M\left(\{g(\vx_1,\vx_2,\vy\}_{h=1}^H\right)}^2\right] \leq \epsilon
    \end{equation}
    where $g_h(\vx_1,\vx_2,\vy) = \arg\max_{\vx_i}\inner{\vx_i,\vq_h}\cdot \inner{\vx_,\vq_h}$ and $M$ is the majority function. We can take $H-2$ instead of $H$ by increasing the constant by a factor of at most $2$.

    We define $\mM_i = \alpha\vq_i \vq_i^\top$ for $i=1,\dots,H-2$ for some $\alpha > 0$ which will be defined later. We also pick some $\vq_0\in \sd$ and define $\mM_{H-1} = \alpha\vq_0 \vq_0^\top$ and $\mM_{H} = -\alpha\vq_0 \vq_0^\top$. Note that if $\vq_0\notin\{\vx_1,\vx_2,\vy\}$ and  $\argmax_{\vx_i} \vx_i \mM_{H-1} \vy = \vx_1$ then $\argmax_{\vx_i} \vx_i \mM_{H} \vy = \vx_2$ and vice versa.

    Let $g:\reals^{dH}\rightarrow\reals^d$ be the $4$-layer network with width $c_1d^2H$ as defined in \cref{lem:mode_MLP} which simulates the majority. Denote by $\vv:= \begin{bmatrix}\mX \sm(\mX^\top \mM_1 \vy) \\ \vdots \\ \mX \sm(\mX^\top \mM_H \vy) \end{bmatrix}$ and by $\vv_{\max} = \begin{bmatrix}\arg\max_{\vx_i}(\vx_i^\top \mM_1 \vy) \\ \vdots \\ \arg\max_{\vx_i}(\vx_i^\top \mM_H \vy) \end{bmatrix}$. We have that:
    \begin{align} \label{eq:f-g(v) two terms}
        & \E_{\vx_1,\vx_2,\vy\sim \unif(\sphere)}\left[\left\|f(\vx_1,\vx_2;\vy) - g\left(\vv\right) \right\|^2\right]  \nonumber\\
        \leq & \E_{\vx_1,\vx_2,\vy\sim \unif(\sphere)}\left[\left\|f(\vx_1,\vx_2;\vy) - g\left(\vv_{\max}\right) \right\|^2\right]  + \E_{\vx_1,\vx_2,\vy\sim \unif(\sphere)}\left[\left\|g\left(\vv_{\max}\right) - g\left(\vv\right) \right\|^2 \right]~.
    \end{align}
    We will bound each term separately. For the first term in \cref{eq:f-g(v) two terms} we can write:
    \begin{align}
        &\E_{\vx_1,\vx_2,\vy\sim \unif(\sphere)}\left[\left\|f(\vx_1,\vx_2;\vy) - g\left(\vv_{\max}\right) \right\|^2\right] \\
        = &\E\left[\left\|f(\vx_1,\vx_2;\vy) - g\left(\vv_{\max}\right) \right\|^2|\inner{\vx_1,\vx_2} \leq 0.1\right]\cdot \Pr(\inner{\vx_1,\vx_2} \leq 0.1) + \\
        + &\E\left[\left\|f(\vx_1,\vx_2;\vy) - g\left(\vv_{\max}\right) \right\|^2|\inner{\vx_1,\vx_2} > 0.1\right]\cdot \Pr(\inner{\vx_1,\vx_2} > 0.1)~.
    \end{align}
    By \cref{lem:hoeffding for majority} the first term is bounded by $\epsilon$. For the second term, note that $\left\|f(\vx_1,\vx_2;\vy) - g\left(\vv_{\max}\right) \right\|^2 \leq 2$ since the output of each function is a unit vector. Also, by standard concentration of random vectors on the unit sphere (see Section 3 in \cite{vershynin2018high}), there is a universal constant $c_3 > 0$ such that $\Pr(\inner{\vx_1,\vx_2} > 0.1) \leq \exp(-c_3 d)$. Hence, we can bound $\E\left[\left\|f(\vx_1,\vx_2;\vy) - g\left(\vv_{\max}\right) \right\|^2\right] \leq \epsilon + 2\exp(-c_3 d)$.

    We will bound the second term in \cref{eq:f-g(v) two terms} uniformly for any $\vx_1,\vx_2,\vy$. Note that $g$ is a ReLU neural network with $4$ layers, width $c_1d^2 H$ and weights bounded by $2$. Hence, we can bound its Lipschitz constant by the multiplication of the Frobenius norm of its weights matrices, which is bounded by
    $\left(4(c_1d^2 H))^4\right)$. Hence:
    \begin{align}
            \left\|g\left(\vv_{\max}\right) - g\left(\vv\right) \right\|^2 & \leq \left(4(c_1d^2 H))^4\right)\cdot \norm{\vv_{\max} - \vv}^2 \\
            & \leq \left(4(c_1d^2 H))^4\right)H\cdot \max_h\norm{\mX\sm(\mX^\top\mM_h\vy) - \arg\max_{\vx_i} (\vx_i^\top\mM_h\vy)}^2~.
    \end{align}
    There is $\delta > 0$ which depends on $\epsilon$ such that for the set:
    \begin{equation}
    A_\delta:=\{\vx_1,\vx_2,\vy\in\sd: \forall \vq_h,~ (\vx_1-\vx_2)^\top \vq_h \vq_h^\top \vy > \delta\}~,
    \end{equation}
    we have that $\Pr((\vx_1,\vx_2,\vy) \notin A_\delta) \leq \frac{\epsilon}{\left(4(c_1d^2 H))^4\right)2H}$. 
    Note that $\mX\sm(\alpha\mX^\top \vq_h \vq_h^\top \vy)\underset{\alpha\rightarrow\infty}{\longrightarrow}\arg\max_{\vx_i} (\vx_i^\top \vq_h \vq_h^\top\vy)$ uniformly on $A_\delta$ for every $\vq_h$. Hence,  we can find $\alpha > 0$ large enough such that:
    \begin{equation}
    \sup_{\vx_1,\vx_2,\vy\in\sd} \max_h\norm{\mX\sm(\mX^\top\mM_h\vy) - \arg\max_{\vx_i} (\vx_i^\top\mM_h\vy)}^2 \leq \frac{\epsilon}{ \left(4(c_1d^2 H))^4\right)H}~.
    \end{equation}
    This bounds $\E_{\vx_1,\vx_2,\vy\sim \unif(\sphere)}\left[\left\|g\left(\vv_{\max}\right) - g\left(\vv\right) \right\|^2 \right] \leq \epsilon$.

    Combining both bounds from \cref{eq:f-g(v) two terms} we have:
    \begin{equation}
    \E_{\vx_1,\vx_2,\vy\sim \unif(\sphere)}\left[\left\|f(\vx_1,\vx_2;\vy) - g\left(\vv\right) \right\|^2\right]  \leq \epsilon + \exp(-c_3d)
    \end{equation}
    where we changed the constant $c_3$ accordingly.
\end{proof}

\subsection{Proof of \texorpdfstring{\cref{thm:majority_positional}}{Theorem 7}}
\label{sec:majority_positional_proof}

\begin{proof}
We first define the construction. Let $\vq_1,\dots,\vq_H$ be such that the conclusions of \cref{lem:hoeffding for majority} are satisfied (e.g. by drawing them uniformly from the unit sphere). 
Let $\mE = \begin{bmatrix}1 & -1 & 0 \\ 0 & 0 & 0\end{bmatrix}$.
We call the second dimension of the positional encodings the ``scratch space''.
We construct the heads of the first layer as follows:
For each $h$, let 
\begin{equation}
\mM_h^{(1)} = \alpha\begin{bmatrix}\vq_h \\ 0 \\ 0\end{bmatrix} \begin{bmatrix}\vq_h^\top & 0 & 0\end{bmatrix}
\qquad
\mV_h^{(1)} = \begin{bmatrix}\vzero \\ 0 \\ 1\end{bmatrix} \begin{bmatrix}\vzero^\top & 1 & 0\end{bmatrix}
\end{equation}
The number of heads in the first layer is $H$. The weights of the second layer of the transformer are defined as:

\begin{equation} 
\mM_i^{(2)} = \begin{bmatrix}\vzero \\ 1 \\ 0\end{bmatrix} \begin{bmatrix}\vzero^\top & 0 & 1\end{bmatrix}
\qquad
\mV_i^{(2)} = \beta\begin{bmatrix}\ve_i \\ 0 \\ 0\end{bmatrix}\begin{bmatrix}\ve_i^\top & 0 & 0\end{bmatrix} 
\end{equation}
for the standard basis vectors $\ve_i$, and $\beta > 0$ will be defined later. The number of heads in the second layer is $d$. Finally, we set the output layer as $\mA = \frac1a\begin{bmatrix}\mI_d & 0 & 0 \end{bmatrix}$.

We will now prove the correctness of the construction. For the following argument, assume that each head uses hardmax instead of softmax. Note that by a similar argument used in the proof of \cref{thm:majority_wide_hidden}, this incurs an extra loss of $\epsilon$ for any $\epsilon > 0$ at the cost of increasing $\alpha$.

When the first layer is applied to the input $\vy$, the scratch space of the output of each head is $1$ if $\vx_1^\top \vq_h \vq_h^\top \vy > \vx_2^\top \vq_h \vq_h^\top \vy$ and $-1$ otherwise.
Let $s_y,s_{\vx_1},s_{\vx_1}$ be the sum of the scratch spaces of all the $H$ heads (we will in fact only use $s_\vy$). Note that $s_\vy > 0$ if the majority of the heads outputted $\vx_1$ and $s_\vy < 0$ if the majority outputted for $\vx_2$.
All other dimensions of the output are 0.
Thus, after the skip connection, the output of the first layer is
\begin{equation} T^{(1)}\left(\begin{bmatrix}\vx_1 & \vx_2 & \vy \\ 1 & -1 & 0 \\ 0 & 0 & 0\end{bmatrix}\right) = \begin{bmatrix}\vx_1 & \vx_2 & \vy \\ 1 & -1 & 0 \\ s_{\vx_1} & s_{\vx_2} & s_{\vy} \end{bmatrix} ~.
\end{equation}
For the second layer of attention, note that each head attends to $\vx_1$ if $s_{\vy} > 0$ and to $\vx_2$ otherwise.
By summing $d$ such heads, where each head corresponds to some standard basis vector, 
the output of the second layer is
\begin{equation} T^{(2)}\left(T^{(1)}\left(\begin{bmatrix}\vx_1 & \vx_2 & \vy \\ 1 & -1 & 0 \\ 0 & 0 & 0\end{bmatrix}\right)\right) = \begin{bmatrix}\vy \\ 0 \\ s_{\vy}\end{bmatrix} + \beta\begin{bmatrix}\vx_1 \\ 1 \\ s_{\vx_1} \end{bmatrix} \end{equation}
if $s_{\vy} > 0$, and the same with $\vx_2$ if $s_{\vy} > 0$.
Finally, the after the output layer, the output of the entire transformer is $\frac{1}{\beta}\vy + \vx_1$ if $s_{\vy} > 0$, or $\frac{1}{\beta}\vy + \vx_2$ otherwise.

By taking first take $\beta > \frac{1}{\epsilon} $, we get that the output of the transformer is the same as the output of the majority of the rank-$1$ attention heads of the first layer of the transformer, up to an extra error of $\epsilon$. By \cref{lem:hoeffding for majority} and taking the number of heads $H$ to be large enough, we get that the majority of the heads in the first layer approximates the target up to an error of $\epsilon$. Scaling $\epsilon$ appropriately finishes the proof.

\end{proof}

\end{document}